\documentclass{article}

\usepackage[english]{babel}
 
\usepackage[utf8x]{inputenc}
\usepackage[T1]{fontenc}
\usepackage{natbib}
\usepackage{subcaption}
\usepackage{booktabs}
\usepackage{multirow}
\usepackage{url}
\usepackage{soul}
\usepackage{standalone}
\usepackage{tikz,pgfplots}
\usepackage{tikz,xcolor}
\usetikzlibrary{decorations.pathmorphing}
\usetikzlibrary{calc}

\usetikzlibrary{arrows}
\usetikzlibrary{calc}
\usetikzlibrary{decorations.pathmorphing}
\usetikzlibrary{positioning}
\usetikzlibrary{shapes}
\usetikzlibrary{arrows.meta,backgrounds,automata}
\usetikzlibrary{positioning,shapes,matrix,calc}
\tikzstyle{vertex}=[circle, draw, fill=gray!80!white,thick,scale=1.2]
\tikzstyle{edge}=[draw=black, thick,-]
\usetikzlibrary{hobby,arrows,decorations.pathmorphing,backgrounds,positioning,fit,petri,calc}
\pgfdeclarelayer{background}
\pgfsetlayers{background,main}

\definecolor{lgreen}{HTML}{4DA84D}
\definecolor{lred}{HTML}{007B82}
\definecolor{lrose}{HTML}{FFC300}
\definecolor{lblue}{HTML}{4169E1}
\definecolor{lorange}{HTML}{FF6F61}
\definecolor{lviolet}{HTML}{7A66CC}

\usepackage{todonotes}

\usepackage{paralist}
\usepackage[stable]{footmisc}
\usepackage{adjustbox}

\usepackage{ifthen}
\newcommand{\CC}[1][]{$\text{C\hspace{-.25ex}}^{_{_{_{++}}}}
\ifthenelse{\equal{#1}{}}{}{\text{\hspace{-.625ex}#1}}$}

\usepackage{bm}
\usepackage{bbm}
\usepackage{amsmath}
\usepackage[all,warning]{onlyamsmath}
\usepackage{amssymb}
\usepackage{amsthm}
\usepackage{amsfonts}
\usepackage{thmtools}	
\usepackage[euro]{isonums}

\usepackage[mathic=true]{mathtools}
\usepackage{fixmath}
\usepackage{siunitx}
\usepackage{dirtytalk}

\usepackage{mleftright}
\usepackage{stmaryrd}
\usepackage{xfrac}

\usepackage{anyfontsize}
\usepackage[scaled=0.86]{helvet}

\usepackage{thm-restate}

\let\originalleft\left
\let\originalright\right
\renewcommand{\left}{\mathopen{}\mathclose\bgroup\originalleft}
\renewcommand{\right}{\aftergroup\egroup\originalright}

\usepackage{pifont}

\newcommand{\xxmark}{
  \tikz[scale=0.3, line width=0.6pt, line cap=round] {
    \draw[thick] (0.2, 0.1) to[bend left= 20] (0.8, 0.9);
    \draw[thick] (0.3, 0.9) -- (0.55, 0.2);
  }%
}

\usepackage{enumitem}
\setlist[enumerate]{itemsep=0.2ex, topsep=0.5\topsep}
\setlist[description]{itemsep=0.2ex, topsep=0.5\topsep}
\setlist[itemize]{itemsep=0.2ex, topsep=0.5\topsep}

\makeatletter
\def\thmt@refnamewithcomma #1#2#3,#4,#5\@nil{%
\@xa\def\csname\thmt@envname #1utorefname\endcsname{#3}%
\ifcsname #2refname\endcsname
\csname #2refname\expandafter\endcsname\expandafter{\thmt@envname}{#3}{#4}%
\fi
}
\makeatother

\usepackage[pagebackref,
pdfa,
hidelinks,
pdftex, 
pdfdisplaydoctitle,
pdfpagelabels,
pdfauthor={Christopher Morris},
pdftitle={},
pdfsubject={},
pdfkeywords={MPNNs, generalization, margin, expressivity},
pdfproducer={Latex with the hyperref package},
pdfcreator={pdflatex}
]{hyperref}

\usepackage[capitalise,noabbrev]{cleveref}   

\newtheorem{theorem}{Theorem}

\newtheorem{lemma}[theorem]{Lemma}

\theoremstyle{definition}

\theoremstyle{remark}

\usepackage[activate={true,nocompatibility},final,kerning=true]{microtype}
\usepackage{ellipsis}
\usepackage{todonotes}

\usepackage[accepted]{preprint}

\newcommand{\mG}{\mathbold{G}}

\newcommand{\cO}{\mathcal{O}}

\newcommand{\Rb}{\mathbb{R}}

\newcommand{\hb}{\mathbold{h}}

\newcommand{\UPD}{\mathsf{UPD}}
\newcommand{\AGG}{\mathsf{AGG}}

\newcommand{\new}[1]{\emph{#1}}

\newcommand{\trans}{^\intercal}
\renewcommand{\vec}[1]{\mathbold{#1}}
\newcommand{\oms}{\{\!\!\{}
\newcommand{\cms}{\}\!\!\}}
\newcommand{\tup}[1]{{(#1)}}

\icmltitlerunning{Towards graph neural networks for provably solving convex optimization problems}

\begin{document}

\twocolumn[
\icmltitle{Towards graph neural networks for provably solving convex optimization problems}

\icmlsetsymbol{equal}{*}

\begin{icmlauthorlist}
\icmlauthor{Chendi Qian}{yyy}
\icmlauthor{Christopher Morris}{yyy}
\end{icmlauthorlist}

\icmlaffiliation{yyy}{Department of Computer Science, RWTH Aachen University, Aachen, Germany}

\icmlcorrespondingauthor{Chendi Qian}{chendi.qian@log.rwth-aachen.de}

\icmlkeywords{graph neural networks, GNNs, MPNN, optimization, quadratic}

\vskip 0.3in
]

\printAffiliationsAndNotice{}

\begin{abstract}
Recently, message-passing graph neural networks (MPNNs) have shown potential for solving combinatorial and continuous optimization problems due to their ability to capture variable-constraint interactions. While existing approaches leverage MPNNs to approximate solutions or warm-start traditional solvers, they often lack guarantees for feasibility, particularly in convex optimization settings. Here, we propose an iterative MPNN framework to solve convex optimization problems with provable feasibility guarantees. First, we demonstrate that MPNNs can provably simulate standard interior-point methods for solving quadratic problems with linear constraints, covering relevant problems such as SVMs. Secondly, to ensure feasibility, we introduce a variant that starts from a feasible point and iteratively restricts the search within the feasible region. Experimental results show that our approach outperforms existing neural baselines in solution quality and feasibility, generalizes well to unseen problem sizes, and, in some cases, achieves faster solution times than state-of-the-art solvers such as Gurobi. 
\end{abstract}

\section{Introduction}

\new{Message-passing graph neural networks} (MPNNs) have recently been widely applied to optimization problems, including continuous and combinatorial domains \citep{bengio2021machine, cappart2023combinatorial, scavuzzo2024machine}. Due to their inherent ability to capture structured data, MPNNs are well-suited as proxies for representing and solving such problems, e.g., in satisfiability problems, literals and clauses can be modeled as different node types within a bipartite graph \citep{selsam2018learning}. At the same time, in \new{linear programming} (LP), the variable-constraint interaction naturally forms a bipartite graph structure~\citep{Gas+2019}. Consequently, MPNNs are used as a lightweight proxy to solve such optimization problems in a data-driven fashion.

Most \new{combinatorial optimization} (CO) problems are \textsf{NP}-{hard}~\citep{Aus+1999}, making exact solutions computationally intractable for larger instances. To address this, it is often more practical to relax the integrality constraints and solve the corresponding continuous optimization problems instead. This approach involves formulating a \new{continuous relaxation} as a proxy for the original problem. After solving the relaxed problem, the resulting continuous solutions guide the search in the discrete solution space \citep{Schrijver86}. For example, solving the underlying \new{linear programming relaxation} in \new{mixed-integer linear programming} (MILP) is a crucial step for scoring candidates in strong branching \citep{achterberg2005branching}, also using MPNNs~\citep{Gas+2019,pmlr-v238-qian24a}. Recent studies have extended MPNN-based modeling to other continuous optimization problems, especially \new{quadratic programming} (QP)~\citep{li2024pdhg,gao2024ipm,xiong2024neuralqp}. 
Some existing methods integrate neural networks with traditional solvers \citep{pmlr-v202-fan23d,jung2022learning,li2022learning2reform,li2024pdhg,liu2024learningpivot}. These approaches either replace components of the solver with neural networks or use neural networks to warm-start the solver. In the former case, the methods remain limited by the solver's framework. In contrast, in the latter, the solver is still required to solve a related problem to produce feasible and optimal solutions.

However, the above works mainly aim to predict a near-optimal solution without ensuring the feasibility of such a solution. For example, \citet{fioretto2020predicting,pmlr-v238-qian24a} propose penalizing constraint violations by adding an extra loss term during training without offering strict feasibility guarantees. Existing strategies typically fall into two categories: relying on solvers to produce feasible solutions or projecting neural network outputs into the feasible region. The first approach, where the neural network serves primarily as a warm-start for the solver, still relies on the solver for final solutions. The second approach, which involves projecting outputs into the feasible region, often requires solving an additional optimization problem and can be efficient only for specific cases. Moreover, this projection may degrade solution quality \citep{li2024onsmallgnn}. More recently, approaches leveraging Lagrangian duality theory have emerged, aiming to design neural networks capable of producing dual-feasible solutions \citep{fioretto2021lagrangian,klamkin2024dual,park2023ss_pdl}. While promising, these methods often require many iterations and still lack guarantees for strict feasibility.

\paragraph{Present work} We propose an MPNN architecture that directly outputs high-quality, feasible solutions to convex optimization problems, closely approximating the optimal ones. Building on \citet{pmlr-v238-qian24a}, which pioneered using MPNNs for simulating polynomial-time \new{interior-point methods} (IPM) for LPs, we extend this approach to \new{linearly constrained quadratic programming} (LCQP), covering relevant problems such SVMs. Unlike \citet{pmlr-v238-qian24a}, where each MPNN layer corresponds to an IPM iteration, our fixed-layer MPNN predicts the next interior point from the current one, decoupling the number of MPNN layers from IPM iterations. We further prove that such MPNNs can simulate IPMs for solving LCQPs. In addition, we incorporate a computationally lightweight projection step that restricts the search direction to the feasible region to ensure feasibility, leveraging the constraints' linear algebraic structure. Experiments show our method outperforms neural network baselines, generalizes well to larger, unseen problem instances, and, in some cases, achieves faster solution times than state-of-the-art exact solvers such as Gurobi \citep{gurobi}.

In summary, our contributions are as follows.
\begin{enumerate}
    \item We propose an MPNN-based approach for predicting solutions of LCQP instances.
    \item We theoretically show that an MPNN with $\mathcal{O}(1)$ distinct layers, where each layer has unique weights, and $\mathcal{O}(m+n)$ total message-passing steps and each step executes a layer that may be reused across steps, can simulate an IPM for LCQP.
    \item We introduce an efficient and feasibility-guaranteed variant that incorporates a projection step to ensure the predicted solutions strictly satisfy the constraints of the LCQP.
    \item Empirically, we show that both methods achieve high-quality predictions and outperform traditional QP solvers in solving time for certain problems. Furthermore, our approach can generalize to larger, unseen problem instances in specific cases.
\end{enumerate}

\subsection{Related work}

\begin{figure*}
\begin{center}
\includegraphics[scale=1.5]{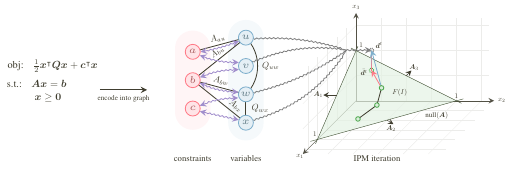}
\end{center}
\vspace{-20pt}
\caption{Overview of our MPNN architectures for solving LCQPs.}
\end{figure*}

Here, we discuss relevant related work.

\paragraph{MPNNs} MPNNs~\citep{Gil+2017,Sca+2009} have been extensively studied in recent years. Notable architectures can be categorized into spatial models \citep{Duv+2015,Ham+2017,bresson2017residual,Vel+2018,xu2018how} and spectral MPNNs \citep{Bru+2014,defferrard2016convolutional,Kip+2017,Lev+2019,monti2018motifnet,geisler2024spatio}. The former conforms to the message-passing framework of~\citet{Gil+2017}, while the latter leverage the spectral property of the graph. 

\paragraph{Machine learning for convex optimization}
In this work, we focus on convex optimization and direct readers interested in combinatorial optimization problems to the surveys \citet{bengio2021machine,cappart2023combinatorial,peng2021graph}; see a more detailed discussion on related work in \cref{sec:more_literature}.

A few attempts have been made to apply machine learning to LPs. \citet{li2022learning2reform} learned to reformulate LP instances, and \citet{pmlr-v202-fan23d} learned the initial basis for the simplex method, both aimed at accelerating the solver. \citet{liu2024learningpivot} imitated simplex pivoting, and \citet{pmlr-v238-qian24a} proposed using MPNNs to simulate IPMs for LPs \citep{nocedal2006numerical}. \citet{li2024pdhg} introduced PDHG-Net to approximate and warm-start the \new{primal-dual hybrid gradient algorithm} (PDHG)  \citep{applegate2021practical,lu2024first}. \citet{li2024onsmallgnn} bounded the depth and width of MPNNs while simulating a specific LP algorithm.
Quadratic programming (QP) has seen limited standalone exploration. Notable works include \citet{bonami2018learning}, who analyzed solver behavior to classify linearization needs, and \citet{pmlr-v157-getzelman21a}, who used reinforcement learning for solver selection. Others accelerated solvers by learning step sizes \citep{ichnowski2021accelerating,jung2022learning} or warm-started them via end-to-end learning \citep{sambharya2023end}. Graph-based representations have been applied to quadratically constrained quadratic programming (QCQP) \citep{wu2024representingqcqp,xiong2024neuralqp}, while \citet{gao2024ipm} extended \citet{pmlr-v238-qian24a} to general nonlinear programs.
On the theoretical side, \citet{chen2022representing,chen2024qp,wu2024representingqcqp} examined the expressivity of MPNNs in approximating LP and QP solutions, offering insights into their capabilities and limitations.

See \Cref{sec:more_literature} for a discussion on machine learning for constrained optimization.

\subsection{Background}

Here, we introduce notation, MPNNs, and convex optimization problems.

\paragraph{Notation} Let $\mathbb{N} \coloneqq \{0, 1, 2, \dots \}$. For $n \geq 1$, let $[n] \coloneqq \{ 1, \dotsc, n \} \subset \mathbb{N}$. We use $\{\!\!\{ \dots\}\!\!\}$ to denote multisets, i.e., the generalization of sets allowing for multiple instances for each of its elements. A \new{graph} $G$ is a pair $(V(G),E(G))$ with \emph{finite} sets of
\new{vertices} or \new{nodes} $V(G)$ and \new{edges} $E(G) \subseteq \{ \{u,v\} \subseteq V(G) \mid u \neq v \}$. For ease of notation, we denote the edge $\{u,v\}$ in $E(G)$ by $(u,v)$ or $(v,u)$. Throughout the paper, we use standard notation, e.g., we denote the \new{neighborhood} of a node $v$ by $N(v)$; see~\cref{notation_app} for details. By default, a vector $\vec{x} \in \mathbb{R}^{d}$ is a column vector.

\paragraph{MPNNs} Intuitively, MPNNs learn a vectorial representation, i.e., a $d$-dimensional real-valued vector, representing each vertex in a graph by aggregating information from neighboring vertices. Let $G$ be an $n$-order attributed graph with node feature matrix $\vec{X} \in \mathbb{R}^{n \times d}$, for $d > 0$, following, \citet{Gil+2017} and \citet{Sca+2009}, in each layer, $t > 0$,  for vertex $v \in V(G)$, we compute a vertex feature $\hb_{v}^\tup{t} \coloneqq$
\begin{equation*}
	\UPD^\tup{t}\Bigl(\hb_{v}^\tup{t-1},\AGG^\tup{t} \bigl(\oms \hb_{u}^\tup{t-1}
	\mid u\in N(v) \cms \bigr)\Bigr) \in \Rb^{d},
\end{equation*}
where  $\UPD^\tup{t}$ and $\AGG^\tup{t}$ may be parameterized functions, e.g., neural networks, and $\hb_{v}^\tup{0} \coloneq \vec{X}_v$. 

\paragraph{Convex optimization problems} In this work, we focus on convex optimization problems, namely LCQPs, of the following form,
\begin{equation}
\begin{aligned}
\setlength{\abovedisplayskip}{6pt}
\setlength{\belowdisplayskip}{6pt}
\label{eq:standard_qp}
\min_{\vec{x} \in \mathbb{Q}^n_{\geq 0}}  \frac{1}{2} \vec{x}\trans \vec{Q} \vec{x} + \vec{c}\trans \vec{x} \text{ such that } \vec{A} \vec{x} = \vec{b}, \vec{x} \geq \vec{0}.
\end{aligned}
\end{equation}
Here, an LCQP instance $I$ is a tuple $(\vec{Q}, \vec{A}, \vec{b}, \vec{c})$, where $\vec{Q} \in \mathbb{Q}^{n \times n}$ and $\vec{c} \in \mathbb{Q}^n$ are the quadratic and linear coefficients of the objective, $\vec{A} \in \mathbb{Q}^{m \times n}$ and $\vec{b} \in \mathbb{Q}^m$ form the constraints. We assume that the quadratic matrix $\vec{Q}$ is positive semi-definite (PSD), i.e., $\vec{v}\trans \vec{Q} \vec{v} \geq 0$, for all $\vec{v} \in \mathbb{R}^n$; otherwise, the problem is non-convex. We assume $m \leq n$ and that $\vec{A}$ has full rank $m$; otherwise, either the problem is infeasible, or some linear constraints can be eliminated using resolving techniques~\citep{Andersen1995PresolvingIL}. Furthermore, if constraints are inequalities, we can always transform them into equalities by adding slack variables \citep{boyd2004convex}. We denote the feasible region of the instance $I$ as $F(I) \coloneq \{ \vec{x} \in \mathbb{Q}^n \mid \vec{A}_j \vec{x} = \vec{b}_j \text{ for } j \in [m] \text{ and } \vec{x}_i \geq 0 \text{ for } i \in [n] \}$. The optimal solution $\vec{x}^* \in F(I)$ is defined such as $\frac{1}{2} \vec{x}\trans \vec{Q} \vec{x} + \vec{c}\trans \vec{x} \geq \frac{1}{2} {\vec{x}^*}\trans \vec{Q} \vec{x}^* + \vec{c}\trans \vec{x}^*, $\text{ for }$ x \in F(I)$, which we assume to be unique. 

\section{Towards MPNNs for solving convex optimization problems}

Here, we provide a detailed overview of our MPNN architectures. We first describe how to represent an LCQP instance as a graph. Next, we present an MPNN architecture that provably simulates IPMs for solving LCQPs. Finally, we introduce a null space projection method to ensure feasibility throughout the search. %

\paragraph{Encoding LCQP instances as graphs} \label{sec:graph_construct}
Previous works have demonstrated that LP instances can be effectively encoded as a bipartite or tripartite graph~\citep{Gas+2019,khalil2022mip,pmlr-v238-qian24a}. In the case of LCQPs, following~\citet{chen2024qp}, we encode a given LCQP instance $I$ into a graph $G(I)$ with constraint node set $C(I)$ and variable node set $V(I)$. We define the constraint-variable node connections via the non-zero entries of the $\vec{A}$ matrix and define the edge features via $\vec{e}_{cv} \coloneq \vec{A}_{cv}$, for $v \in V(I), c \in C(I)$. In addition, the constraint vector $\vec{b}$ acts as features for the constraint nodes, i.e., we design the constraint node feature matrix as $\vec{C} \coloneq \text{reshape}\left(\vec{b}\right) \in \mathbb{Q}^{m \times 1}$. The feature matrix for the variable nodes is set to the objective vector $\vec{V} \coloneq \text{reshape} \left(\vec{c}\right) \in \mathbb{Q}^{n \times 1}$. To encode the $\vec{Q}$ matrix, we follow \citet{chen2024qp} and encode the non-zero entries $\vec{Q}_{vu}$ as edges between variable nodes $v,u$, and use the value $\vec{Q}_{vu}$ as the edge attribute, $\vec{e}_{vu} \coloneq \vec{Q}_{vu}, v,u \in V(I)$. Moreover, we add a global node $\{g(I)\}$ to $G(I)$, similar to \citet{pmlr-v238-qian24a} for LPs, and connect it to all the variable and constraint nodes with uniform edge features $\vec{e}_{cg} \coloneq 1, c \in C(I)$, and $\vec{e}_{vg} \coloneq 1, v \in V(I)$. 

\paragraph{MPNNs for simulating IPMs} \label{sec:ipm} Here, we derive an MPNN architecture to provably simulate IPMs for solving LCQPs. For more details on IPM for LCQP, we refer to \cref{sec:ipm_derive}. Similar to \citet{gao2024ipm}, we decouple the number of layers $L$ of the MPNN and the number of iterations $T$. Additionally, the learnable parameters of the $L$-layer MPNN are shared across the iterations. At each iteration $t \in [T]$, the $L$-layer MPNN takes the graph $G(I)$ together with the current solution $\vec{x}^{(t-1)}$ as input. The MPNN outputs a scalar per variable node as the prediction for the next interior point $\vec{x}^{(t)}$, with an arbitrary initial solution $\vec{x}^{(0)} \geq \bm{0}$. We denote the node embedding at layer $l \in [L]$ and iteration $t \in [T]$ as $\vec{h}_{\circ}^{(l,t)}, \circ \in V(I) \cup C(I) \cup \{g(I)\}$. At the beginning of each iteration, we initialize the node embeddings as
\begin{equation}
\label{eq:tri_MPNN_init}
\begin{aligned}
\setlength{\abovedisplayskip}{6pt}
\setlength{\belowdisplayskip}{6pt}
    \vec{h}_c^{(0,t)} \coloneqq & \vec{C}_c \in \mathbb{Q}, c \in C(I),\\
    \vec{h}_v^{(0,t)} \coloneqq & \textsf{CONCAT}\left(\vec{V}_v,  \vec{x}^{(t-1)}_v \right) \in \mathbb{Q}^{2}, v \in V(I), \\
    \vec{h}_g^{(0,t)} \coloneqq & \bm{0}.
\end{aligned}
\end{equation}

Each message passing layer consists of three sequential steps, similar to \citet{pmlr-v238-qian24a}. First, the embeddings of the constraint nodes are updated, using the embeddings of the variable nodes and the global node $\vec{h}_c^{(l,t)} \coloneqq$
\begin{equation}
\label{eq:tri_MPNN_update1}
\begin{aligned}
    &\textsf{UPD}^{(l)}_{\text{c}}\Bigl[  \vec{h}_c^{(l-1,t)}, \\
    &\textsf{MSG}^{(l)}_{\text{v} \rightarrow \text{c}}\left(\{\!\!\{ (\vec{h}_v^{(l-1,t)}, \vec{e}_{cv}) \mid v \in {N}\left(c \right) \cap V(I) \}\!\!\}  \right), \\
    &\textsf{MSG}^{(l)}_{\text{g} \rightarrow \text{c}}\left( \vec{h}_g^{(l-1,t)}, \vec{e}_{cg} \right) \Bigr] \in \mathbb{Q}^d.
\end{aligned}
\end{equation}

\begin{algorithm}[t]
\caption{Training phase of our MPNN architecture simulating IPM steps. 
}
\label{alg:ipm_qp_MPNN_train}
\begin{algorithmic}[1]
    \REQUIRE An LCQP instance $I = (\vec{Q}, \vec{A}, \vec{b}, \vec{c})$, an arbitrary initial solution $\vec{x}^{(0)}$, an $L$-layer $\textsf{MPNN}$ architecture following \cref{eq:tri_MPNN_init,eq:tri_MPNN_update1,eq:tri_MPNN_update2,eq:tri_MPNN_update3,eq:tri_MPNN_pred}, number of iterations $T$, the ground-truth interior points $\{\vec{x}^{*(t)} \mid t \in [T] \}$. \\
    \ENSURE Supervised loss $\mathcal{L}$.
    \vspace{3pt}

	\FOR{$t \in [T]$}
        \vspace{3pt}
	    \STATE $\vec{x}^{(t)} \leftarrow \textsf{MPNN}\left(I, \vec{x}^{(t-1)} \right) $
            \STATE Calculate $\mathcal{L}^{(t)}(\vec{x}^{*(t)}, \vec{x}^{(t)})  $ as in \cref{eq:ipm_loss}
	\ENDFOR
    \RETURN The loss $\mathcal{L} \leftarrow \frac{1}{T} \sum_{t=1}^{T} \mathcal{L}^{(t)}(\vec{x}^{*(t)}, \vec{x}^{(t)})$
\end{algorithmic}
\end{algorithm}

\begin{algorithm}[t]
\caption{Inference phase of our MPNN architecture simulating IPM steps. 
}
\label{alg:ipm_qp_MPNN_val}
\begin{algorithmic}[1]
    \REQUIRE An LCQP instance $I = (\vec{Q}, \vec{A}, \vec{b}, \vec{c})$, an arbitrary initial solution $\vec{x}^{(0)}$, an $L$-layer $\textsf{MPNN}$ architecture following \cref{eq:tri_MPNN_init,eq:tri_MPNN_update1,eq:tri_MPNN_update2,eq:tri_MPNN_update3,eq:tri_MPNN_pred}, number of iterations $T$, threshold  $\delta$ for solution selection. \\
    \ENSURE Best found solution $\vec{x}$.
    \vspace{3pt}
        \STATE Define $\text{obj}(\vec{x}) \coloneq \frac{1}{2} \vec{x}\trans \vec{Q} \vec{x} + \vec{c}\trans \vec{x}$
        \STATE Define $\text{cons}(\vec{x}) \coloneq \max \{ \left|\vec{A}_i \vec{x} - \vec{b}_i\right| \mid i \in [m]\}$
        \IF{$\text{cons}(\vec{x}^{(0)}) < \delta$}
            \STATE $\vec{x} \leftarrow \vec{x}^{(0)}$
        \ENDIF
	\FOR{$t \in [T]$}
        \vspace{3pt}
	    \STATE $\vec{x}^{(t)} \leftarrow \textsf{MPNN}\left(I, \vec{x}^{(t-1)} \right)$
            \IF{$\text{cons}(\vec{x}^{(t)}) < \delta$ and $\text{obj}\left(\vec{x}^{(t)}\right) < \text{obj}\left(\vec{x}\right)$}
                \STATE $\vec{x} \leftarrow \vec{x}^{(t)}$
            \ENDIF
	\ENDFOR
    \RETURN $\vec{x}$.
\end{algorithmic}
\end{algorithm}

Next, the global node embeddings are updated based on the embeddings of the variable nodes and the most recently updated constraint node embeddings $\vec{h}_g^{(l,t)} \coloneqq$
\begin{equation}
\label{eq:tri_MPNN_update2}
\begin{aligned}
     &\textsf{UPD}^{(l)}_{\text{g}}\Bigl[  \vec{h}_g^{(l-1,t)}, \\
    &\textsf{MSG}^{(l)}_{\text{v} \rightarrow \text{g}}\left(\{\!\!\{ (\vec{h}_v^{(l-1,t)}, \vec{e}_{vg}) \mid v \in V(I) \}\!\!\}  \right), \\
    &\textsf{MSG}^{(l)}_{\text{c} \rightarrow \text{g}}\left(\{\!\!\{ (\vec{h}_c^{(l,t)}, \vec{e}_{cg}) \mid c \in C(I) \}\!\!\}  \right) \Bigr] \in \mathbb{Q}^d.
\end{aligned}
\end{equation}
Finally, the embeddings of the variable nodes are updated by aggregating information from their neighboring variable nodes and the updated constraint node and global node embeddings $ \vec{h}_v^{(l,t)} \coloneqq$
\begin{equation}
\label{eq:tri_MPNN_update3}
\begin{aligned}
    &\textsf{UPD}^{(l)}_{\text{v}}\Bigl[  \vec{h}_v^{(l-1,t)}, \\
    &\textsf{MSG}^{(l)}_{\text{v} \rightarrow \text{v}}\left(\{\!\!\{ (\vec{h}_u^{(l-1,t)}, \vec{e}_{uv}) \mid u \in {N}\left(v \right) \cap V(I) \}\!\!\}  \right),\\
    &\textsf{MSG}^{(l)}_{\text{c} \rightarrow \text{v}}\left(\{\!\!\{ (\vec{h}_c^{(l,t)}, \vec{e}_{cv}) \mid c \in {N}\left(v \right) \cap C(I) \}\!\!\}  \right), \\
    &\textsf{MSG}^{(l)}_{\text{g} \rightarrow \text{v}}\left(\vec{h}_g^{(l,t)}, \vec{e}_{vg}\right)\Bigr] \in \mathbb{Q}^d.
\end{aligned}
\end{equation}
Finally, we use a multi-layer perceptron for predicting the current variable assignments,
\begin{equation}
\label{eq:tri_MPNN_pred}
    \vec{x}_v^{(t)} \coloneqq \textsf{MLP}\left( \vec{h}_v^{(L,t)}\right) \in \mathbb{Q},
\end{equation}
whose output vector $\vec{x}^{(t)} \in \mathbb{Q}^n$ serves as the prediction of the next interior point. 

For training, we use the mean-squared error between our intermediate predictions $\vec{x}^{(t)}, t \in [T]$, and the ground truth interior point given by the IPM $\vec{x}^{*(t)}, t \in [T]$,
\begin{equation}
\label{eq:ipm_loss}
\mathcal{L}^{(t)} \mleft(\vec{x}^{*(t)}, \vec{x}^{(t)} \mright) \coloneq \lVert \vec{x}^{*(t)} - \vec{x}^{(t)} \rVert_2^2.
\end{equation}
During training, we pre-set the iterations and supervise all predicted interior points. During inference, however, our framework allows for an arbitrary number of iterations and picks the best solution simultaneously. The training and inference processes are summarized in \cref{alg:ipm_qp_MPNN_train} and \cref{alg:ipm_qp_MPNN_val}.

Now, the following result shows that there exists an MPNN architecture, $f_{\textsf{MPNN,IPM}}$, of $\cO(1)$ layers and $\cO(m)$ message passing steps in the form of \cref{eq:tri_MPNN_init,eq:tri_MPNN_update1,eq:tri_MPNN_update2,eq:tri_MPNN_update3}, that is capable of simulating the IPM algorithm for LCQPs. For the detailed proof, please see \cref{sec:proof_appendix}.

\begin{theorem}
\label{thm:mpnn_can_ipm}
There exists an MPNN $f_{\textsf{MPNN,IPM}}$ composed of $\cO(1)$ layers and $\cO(m+n)$ successive message-passing steps that reproduces each iteration of the IPM algorithm for LCQPs, in the sense that for any LCQP instance $I = \left(\vec{Q}, \vec{A},\vec{b},\vec{c} \right)$ and any primal-dual point $\left(\vec{x}^{(t)}, \boldsymbol{\lambda}^{(t)}, \vec{s}^{(t)}\right)$ with $t > 0$, $f_{\textsf{MPNN,IPM}}$ maps the graph $G(I)$ carrying $[\vec{x}^{(t-1)}, \vec{s}^{(t-1)}]$ on the variable nodes, $[\boldsymbol{\lambda}^{(t-1)}]$ on the constraint nodes, and $[\mu, \sigma]$ on the global node to the same graph $G(I)$ carrying the output $[\vec{x}^{(t)}, \vec{s}^{(t)}]$ and $[\boldsymbol{\lambda}^{(t)}]$ of \cref{alg:ipm-practice} on the variable and constraint nodes, respectively.
\end{theorem}

\paragraph{Ensuring feasible solutions} \label{sec:feas}
In supervised learning, ensuring feasibility is challenging due to training and validation errors in \cref{alg:ipm_qp_MPNN_train} and \cref{alg:ipm_qp_MPNN_val}. Correcting infeasible solutions typically involves additional optimization, which adds computational overhead and may degrade solution quality. To address this, we propose an iterative method that maintains strict feasibility throughout optimization. Starting from a feasible point $\vec{x}^{(0)}$, the search is constrained to the feasible region $\vec{x}^{(t)} \in F(I)$, for all $t \in [T]$. We also discard intermediate steps from the expert solver, focusing only on the optimal solution to enhance flexibility.
This section details our approach, answering two key questions: (1) How is an initial feasible solution constructed? (2) How can feasibility be preserved during updates? We also describe MPNN modifications and the search algorithm.

We can obtain a feasible initial solution by solving a trivial LP with the given constraints and a simple objective, such as $\vec{c} = \bm{0}$. This incurs moderate overhead, as solving an LP is computationally cheaper than a quadratic problem. Existing methods, like the big-M and two-phase simplex methods \citep{i2lp}, efficiently compute feasible basic solutions, while MOSEK's IPM \citet{Andersen2000mosekipm} achieves feasibility within a few updates.

We follow the generic framework of iterative optimization methods \citep{nocedal2006numerical}, where, at each iteration, a search direction $\vec{d}^{(t)} \in \mathbb{Q}^n$ is determined and corresponding step length $\alpha^{(t)} > 0$ is computed. However, our approach differs in that we train an MPNN in a supervised way to predict the \new{displacement} $\vec{d}^{*(t)} \coloneq \vec{x}^* - \vec{x}^{(t-1)}$ from the current solution $\vec{x}^{(t-1)}$ to the optimal solution $\vec{x}^*$.\footnote{We distinguish the terms \new{direction} and \new{displacement}; the former focuses on the orientation only, with normalized magnitude by default, while the latter emphasizes both the magnitude and direction of the movement from one point to another.} 

To correct the predicted displacement $\vec{d}^{(t)}$, we compute a feasibility-preserving displacement $\tilde{\vec{d}}^{(t)}$ such that $\vec{A} \left(\vec{x}^{(t-1)} + \tilde{\vec{d}}^{(t)}\right) = \vec{b}.$ 
If $\vec{A} \vec{x}^{(t-1)} = \vec{b}$, this reduces to $\vec{A} \tilde{\vec{d}}^{(t)} = \bm{0}$, meaning $\tilde{\vec{d}}^{(t)}$ lies in the null space $\{\vec{v} \in \mathbb{Q}^n \mid \vec{A} \vec{v} = \bm{0}\}$. 
For full-rank $\vec{A} \in \mathbb{Q}^{m \times n}$, the null space has dimension $n-m$, represented by $\{\vec{v}_1, \ldots, \vec{v}_{n-m}\}$. We project $\vec{d}^{(t)}$ onto this space as $\tilde{\vec{d}}^{(t)} = \sum_{i=1}^{n-m} a_i \vec{v}_i$ where $a_i \coloneq \frac{\vec{v}_i\trans \vec{d}^{(t)}}{\vec{v}_i\trans \vec{v}_i}$. Thus, $\tilde{\vec{d}}^{(t)} \coloneq \sum_{i=1}^{n-m} \frac{\vec{v}_i\trans \vec{d}^{(t)}}{\vec{v}_i\trans \vec{v}_i} \vec{v}_i$, ensuring $\vec{A} \tilde{\vec{d}}^{(t)} = \bm{0}$.
Updating $\vec{x}^{(t)} = \vec{x}^{(t-1)} + \tilde{\vec{d}}^{(t)}$ preserves feasibility, assuming $\vec{x}^{(0)}$ is feasible. In practice, \new{singular value decomposition} (SVD) \citep{strang2000linear} provides the orthonormal null space of $\vec{A}$. The projection operator is expressed as  
$\boldsymbol{\Pi}_{\vec{A}} = \sum_{i=1}^{n-m} \frac{\vec{v}_i \vec{v}_i\trans}{\vec{v}_i\trans \vec{v}_i}$, satisfying $\boldsymbol{\Pi}_{\vec{A}}^2 = \boldsymbol{\Pi}_{\vec{A}}$ and has no effect on vectors already in the null space.

If the prediction of the MPNN is exact, i.e., $\tilde{\vec{d}}^{(t)} = \vec{d}^{(t)}$, and we take a step along $\tilde{\vec{d}}^{(t)}$ with step length $\alpha^{(t)}=1$ and we end up with the optimal solution. Therefore, by default, we set the step length $\alpha^{(t)}=1$. However, if a full step using $\alpha^{(t)}=1$ leads to the violation of the positivity constraint, we take a maximal possible $\alpha^{(t)} < 1$ so that all variables would still lie in the positive orthant, specifically
\begin{equation}
\label{eq:step_size}
\alpha^{(t)} \coloneqq\!\min \mleft\{1, \sup \mleft\{ \alpha \mid \vec{x}^{(t)}_i + \alpha \tilde{\vec{d}}^{(t)}_i \geq 0, i \in [n] \mright\} \!\mright\}\!.
\end{equation}
Formally, the update step is 
\begin{equation}
\label{eq:update_sol}
\vec{x}^{(t)} \coloneq \vec{x}^{(t-1)} + \alpha^{(t)} \boldsymbol{\Pi}_{\vec{A}} \vec{d}^{(t)}.
\end{equation}
The MPNN architecture is similar to the IPM-guided approach introduced above. We drop the global node to enhance computational efficiency and align with the theorem outlined below. At each iteration $t \in [T]$, the $L$-layer MPNN takes the graph together with the current solution $\vec{x}^{(t-1)}$ as input but, unlike above, predicts the displacement $\vec{d}^{(t)}$ instead of subsequent point. In each iteration, we initialize the node embeddings as
\begin{equation}
\label{eq:MPNN_init}
\begin{aligned}
    \vec{h}_c^{(0,t)} \coloneqq & \vec{C}_c \in \mathbb{Q},\\
    \vec{h}_v^{(0,t)} \coloneqq & \textsf{CONCAT}\left(\vec{V}_v,  \vec{x}^{(t-1)} \right) \in \mathbb{Q}^{2}. 
\end{aligned}
\end{equation}
The message passing on the heterogeneous graph is defined as 
\begin{equation}
\label{eq:MPNN_update}
\begin{aligned}
    &\vec{h}_c^{(l,t)} \coloneqq \textsf{UPD}^{(l)}_{\text{c}}\Bigl[ \vec{h}_c^{(l-1,t)}, \\ &\textsf{MSG}^{(l)}_{\text{v} \rightarrow \text{c}}\mleft(\{\!\!\{ (\vec{h}_v^{(l-1,t)}, \vec{e}_{cv}) \mid v \in {N}\mleft(c \mright) \cap V(I) \}\!\!\}  \mright) \Bigr].  \\
    &\vec{h}_v^{(l,t)} \coloneqq \textsf{UPD}^{(l)}_{\text{v}}\Bigl[ \vec{h}_v^{(l-1,t)},\\
    & \textsf{MSG}^{(l)}_{\text{v} \rightarrow \text{v}}\mleft(\{\!\!\{ (\vec{h}_u^{(l-1,t)}, \vec{e}_{uv}) \mid u \in {N}\mleft(v \mright) \cap V(I) \}\!\!\}  \mright), \\
    &\textsf{MSG}^{(l)}_{\text{c} \rightarrow \text{v}}\,\mleft(\{\!\!\{ (\vec{h}_c^{(l,t)}, \vec{e}_{cv}) \mid u \in {N}\mleft(v \mright) \cap C(I) \}\!\!\}  \mright) \Bigr],
\end{aligned}
\end{equation}
where we first update the constraint node embeddings and then the variable ones. In addition, we use a multi-layer perceptron to predict the displacement from the current solution $\vec{x}^{(t-1)}$  to the optimal solution $\vec{x}^*$,
\begin{equation}
\label{eq:MPNN_pred}
    \vec{d}_v^{(t)} \coloneqq \textsf{MLP}\left( \vec{h}_v^{(L,t)}\right) \in \mathbb{Q}.
\end{equation}

We denote the exact displacement pointing from the current solution at iteration $t$ to the optimal solution as the oracle displacement, $\vec{d}^{*(t)} = \vec{x}^* - \vec{x}^{(t-1)}$, and define the supervised loss
\begin{equation}
\label{eq:loss}
\mathcal{L}^{(t)} \mleft( \vec{d}^{*(t)}, \vec{d}^{(t)} \mright) \coloneq \lVert \vec{d}^{*(t)} - \vec{d}^{(t)} \rVert_2^2.
\end{equation}
In practice, when the current solution $\vec{x}^{(t)}$ approaches the boundary of the positive orthant $\mathbb{Q}^n_{\geq 0}$ and the prediction $\vec{d}^{(t)}$ is inaccurate, the step length $\alpha^{(t)}$ can be too small to ensure the non-negativity constraints. Due to the continuous nature of neural networks, the prediction $\vec{d}^{(t+1)}$ will hardly change since $\vec{x}^{(t+1)} \approx \vec{x}^{(t)}$ and a small $\alpha^{(t+1)}$ will be picked. This makes the search process stagnant. To address this vanishing step length issue, we introduce a correction term $\frac{\tau}{\vec{x}^{(t)} + \epsilon}$, where the bias $\epsilon$ ensures numerical stability and $\tau$ adjusts the correction's magnitude, encouraging the solution $\vec{x}^{(t)}$ to move away from the orthant boundary. Details of this design are provided in \cref{sec:log_barrier_bias}. \cref{alg:lpMPNN_train,alg:lpMPNN_val} show the training and inference using the MPNNs. 

Now, the following result shows that our proposed MPNN architecture is expressive enough to predict the displacement arbitrarily close. 
\begin{theorem}
Given a LCQP instance $I$, assume $I$ is feasible with solution $\vec{x}^*$, $\vec{x}^{(0)} \geq \bm{0}$ is an initial feasible point, for any $\epsilon, \delta > 0$, there exists an MPNN architecture $f_{\textsf{MPNN,2}}$ such that
\begin{equation*}
    P\left[ \Big\lVert f_{\textsf{MPNN,2}}\left(I, \vec{x}^{(0)}\right) - \left(\vec{x}^* - \vec{x}^{(0)} \right) \Big\rVert_2 > \delta \right] < \epsilon.
\end{equation*}
\end{theorem}

\subsection{Complexity analysis} 
It is straightforward to show that our IPM-guided approach lies within the framework of MPNN. Thus, the runtime is linear to the number of nodes and edges, i.e., $\cO(n+m+|E(G)|)$, where $m,n$ are the numbers of constraints and variables. In the worst case, where the $\vec{A}$ and $\vec{Q}$ matrices are dense, the complexity amounts to $\cO(mn+n^2)$. However, we need further investigation for the feasibility approach. The two-phase method, for the feasible initial solution, requires solving an extra LP and finding the initial feasible point, which is in  $\cO(nm)$ \citep{i2lp}. The null space calculation is based on SVD or QR decomposition and is of complexity $\cO(nm^2)$ \citep{strang2000linear}. Fortunately, finding a feasible solution and calculating the null space and corresponding projection matrix only need to be done once in the pre-processing phase. The projection of the predicted direction on the null space of $\vec{A}$ is $\cO(n(n-m))$. The message passing scheme between the two node types has the same complexity $\cO(|E(G)|)$, which depends on the number of edges, i.e., the nonzero entries of the $\vec{A}$ and $\vec{Q}$ matrices, also $\cO(mn+n^2)$ in the worst case. The remaining part of the algorithm, such as the line search and variable update, is in $\cO(n)$. 

\begin{algorithm}[ht]
\caption{Training phase of our MPNN architecture ensuring feasibility. %
}
\label{alg:lpMPNN_train}
\begin{algorithmic}[1]
    \REQUIRE An LCQP instance $I = (\vec{Q}, \vec{A}, \vec{b}, \vec{c})$, a feasible initial solution $\vec{x}^{(0)}$, projection matrix of $\vec{A}$ as $\boldsymbol{\Pi}_{\vec{A}}$, an $L$-layer $\textsf{MPNN}$ initialized as \cref{eq:MPNN_init,eq:MPNN_update,eq:MPNN_pred}, number of iterations $T$, scaling factor $\tau$, small bias $\epsilon$. \\
    \ENSURE Supervised loss $\mathcal{L}$.
    \vspace{3pt}

	\FOR{$t \in [T]$}
        \vspace{3pt}
	    \STATE $\vec{d}^{(t)} \leftarrow \textsf{MPNN}\left(I, \vec{x}^{(t-1)} \right) + \dfrac{\tau}{\vec{x}^{(t-1)} + \epsilon} $
            \STATE $\tau \leftarrow \tau / 2$
            \STATE Calculate loss $\mathcal{L}^{(t)}( \vec{d}^{*(t)}, \vec{d}^{(t)})$ given by \cref{eq:loss}
            \STATE $\tilde{\vec{d}}^{(t)} \leftarrow \boldsymbol{\Pi}_{\vec{A}} \vec{d}^{(t)}$
            \STATE Get step length $\alpha^{(t)}$ as given by \cref{eq:step_size}
            \STATE Update $\vec{x}^{(t)}$ with \cref{eq:update_sol}
	\ENDFOR
    \RETURN The loss $\mathcal{L} \leftarrow \frac{1}{T} \sum_{t=1}^T \mathcal{L}^{(t)}( \vec{d}^{*(t)}, \vec{d}^{(t)}) $.
\end{algorithmic}
\end{algorithm}

\begin{algorithm}[ht!]
\caption{Inference phase of our MPNN architecture ensuring feasibility. 
}
\label{alg:lpMPNN_val}
\begin{algorithmic}[1]
    \REQUIRE An LCQP instance $I = (\vec{A}, \vec{b}, \vec{c})$ or QP instance $I = (\vec{Q}, \vec{A}, \vec{b}, \vec{c})$, a feasible initial solution $\vec{x}^{(0)}$,projection matrix of $\vec{A}$ as $\boldsymbol{\Pi}_{\vec{A}}$, an $L$-layer $\textsf{MPNN}$ initialized as \cref{eq:MPNN_init,eq:MPNN_update,eq:MPNN_pred}, number of iterations $T$, scaling factor $\tau$, small bias $\epsilon$. \\
    \ENSURE Best found solution $\vec{x}$.
    \vspace{3pt}
        \STATE $\vec{x} \leftarrow \vec{x}^{(0)}$
        \STATE Define $\text{obj}(\vec{x}) \coloneq \frac{1}{2} \vec{x}\trans \vec{Q} \vec{x} + \vec{c}\trans \vec{x}$
	\FOR{$t \in [T]$}
        \vspace{3pt}
	    \STATE $\vec{d}^{(t)} \leftarrow \textsf{MPNN}\left(I, \vec{x}^{(t-1)} \right) + \dfrac{\tau}{\vec{x}^{(t-1)} + \epsilon} $
            \STATE $\tau \leftarrow \tau / 2$
            \STATE $\tilde{\vec{d}}^{(t)} \leftarrow \boldsymbol{\Pi}_{\vec{A}} \vec{d}^{(t)}$
            \STATE Compute step length $\alpha^{(t)}$ as given by \cref{eq:step_size}
            \STATE Update $\vec{x}^{(t)}$ with \cref{eq:update_sol}
            \IF{$\text{obj}\left(\vec{x}^{(t)}\right) < \text{obj}(\vec{x})$}
                \STATE $\vec{x} \leftarrow \vec{x}^{(t)}$
            \ENDIF
	\ENDFOR
    \RETURN Best found solution $\vec{x}$.
\end{algorithmic}
\end{algorithm}

\section{Experimental study}
In the following, we investigate to what extent our theoretical results translate into practice. The implementation is open source at \url{https://github.com/chendiqian/FeasMPNN}. 
\begin{description}
    \item[Q1] How good is the solution quality of our MPNN architectures regarding objective value and constraint violation?
    \item[Q2] How well can our (pre-)trained MPNN architectures generalize to larger unseen instances?
    \item[Q3] How fast are our MPNN architectures compared with baselines and traditional solvers?
\end{description}

We evaluate three types of synthetic LCQP problems: generic, soft-margin support vector machine (SVM), and Markowitz portfolio optimization problems, following~\citet{jung2022learning}. Details of dataset generation are provided in \cref{sec:ds}. For each problem type, we generate $\num{1000}$ instances, split into training, validation, and test sets with an 8:1:1 ratio. 
Hyperparameters for all neural networks are tuned on our feasibility variant and shared across baselines. Specifically, we train an MPNN of 8 layers with hidden dimension 128 using the Adam optimizer \citep{Kin+2015} for up to $\num{1000}$ epochs with early stopping after 300 epochs. During training, our IPM- and feasibility-based architectures use 8 iterations, which are increased to 32 during inference. For our feasibility approach, the instances are preprocessed using SciPy \citep{2020SciPy-NMeth} for null space computation and an IPM solver \citep{frenk2013high} to obtain feasible initial solutions from the linear constraints. Training is conducted on four Nvidia L40s GPUs. Timing evaluations for neural network methods in \cref{tab:runtime} are performed on a single Nvidia L40s. In contrast, solver and preprocessing evaluations are carried out on a MacBook Air with an Apple M2 chip.\footnote{We found that conducting the CPU-based experiments on the M2 chip was faster than executing them on our compute server.}

\paragraph{Solution quality (Q1)} We benchmark all approaches regarding the relative objective gap and constraint violation metrics. Given the ground truth optimal solution $\vec{x}^*$, MPNN-predicted solution $\vec{x}$ of an LCQP instance, and the objective function $\text{obj}(\vec{x}) \coloneq \frac{1}{2} \vec{x}\trans \vec{Q} \vec{x} + \vec{c}\trans \vec{x}$ the relative objective gap is calculated as
\begin{equation*}
\label{eq:rel_obj_err}
\left| \dfrac{\text{obj}(\vec{x}) - \text{obj}(\vec{x}^*)}{\text{obj}(\vec{x}^*)} \right| \times 100\%,
\end{equation*}
and the constraint violation metric is computed as
\begin{equation*}
\dfrac{1}{m} \sum_{i=1}^m \dfrac{\left| \vec{A}_i \vec{x} - \vec{b}_i \right|}{\max \{ \left| \vec{b}_i \right|, \max_j \left|\vec{A}_{ij} \right|\}},
\end{equation*}
i.e., for each constraint $i$, we calculate the absolute violation number $\left| \vec{A}_i \vec{x} - \vec{b}_i \right|$, normalized by the scale of $\vec{A}_i, \vec{b}_i$, and we calculate the mean value across all the constraints. To show the model-agnostic property of our approaches, we use the GIN \citep{xu2018how} and GCN \citep{Kip+2017} as our MPNN layers. The selected baselines are: (1) a naive MPNN approach that directly predicts the LCQP solution, following \citet{chen2024qp}, and (2) a variant of IPM-MPNN for LCQPs \citep{pmlr-v238-qian24a}, where the output of each layer represents an intermediate step of the search process.

Our IPM-guided search theoretically requires a global node in the graph, but our feasibility method does not. To ensure fairness, we conduct experiments with and without the global node; see~\cref{tab:tripart_qp_main} and \cref{tab:bipart_qp_main}. Our feasibility method generally achieves lower relative objective gaps, indicating better approximations of optimal solutions, particularly for generic LCQP and portfolio problems. However, the advantage is less pronounced on SVM tasks, where the quadratic matrix has a diagonal structure. While both IPM approaches exhibit constraint violations, similar to naive MPNN predictions, our feasibility method ensures strict feasibility up to numerical precision ($\num{e-7}$). Differences between global-node and non-global-node settings are typically minor, except for the \citet{chen2024qp} approach on portfolio optimization, where problem-specific factors may account for an order-of-magnitude difference.

\begin{table}[t!]
\caption{Comparisons of our IPM- and feasibility-oriented approaches with the MPNN-based prediction from \citet{chen2024qp} and \citet{pmlr-v238-qian24a}, all \textbf{with the global node}. The table reports the relative objective gap (in percentage) and constraint violation (in normalized absolute value). Training was repeated three times with different random seeds, and we report the mean and standard deviation across runs. For our feasibility method, the standard deviation of constraint violation is omitted as it is not meaningful.}
\label{tab:tripart_qp_main}
\centering
\resizebox{\columnwidth}{!}{
\begin{tabular}{cccccc}
\toprule
    & \textbf{Method}  & MPNN  & Generic & SVM  & Portfolio \\

    \midrule
    {\multirow{8}{*}{Obj. gap [\%]}}   
    
    & {\multirow{2}{*}{Naive \citeyear{chen2024qp}}} & GCN   & 2.281\scriptsize$\pm$0.150 & 0.352\scriptsize$\pm$0.032  & 2.446\scriptsize$\pm$0.756  \\
	 &   & GIN & 2.782\scriptsize$\pm$0.056 & 0.148\scriptsize$\pm$0.086  & 0.295\scriptsize$\pm$0.024  \\
    \cmidrule{2-6}

    & {\multirow{2}{*}{IPM \citeyear{pmlr-v238-qian24a}}} & GCN   & 1.829\scriptsize$\pm$0.079 & 0.162\scriptsize$\pm$0.009  & 0.717\scriptsize$\pm$0.353  \\
	 &   & GIN & 1.581\scriptsize$\pm$0.044 & 0.141\scriptsize$\pm$0.034 & 0.416\scriptsize$\pm$0.052 \\
    \cmidrule{2-6}

    & {\multirow{2}{*}{IPM(ours)}} & GCN   & 1.661\scriptsize$\pm$0.199 &  0.485\scriptsize$\pm$0.176  & 1.838\scriptsize$\pm$0.720  \\
	 &   & GIN & 2.746\scriptsize$\pm$0.152 & 0.109\scriptsize$\pm$0.017  & 0.486\scriptsize$\pm$0.041  \\
    \cmidrule{2-6}

    & {\multirow{2}{*}{Feas.(ours)}}
     & GCN   & 0.070\scriptsize$\pm$0.004 & 0.190\scriptsize$\pm$0.008  & 1.005\scriptsize$\pm$0.159  \\
	 &   & GIN & 0.084\scriptsize$\pm$0.005 & 0.192\scriptsize$\pm$0.017 & 1.078\scriptsize$\pm$0.121  \\
    \midrule
    {\multirow{8}{*}{Cons. vio.}}   

    & {\multirow{2}{*}{Naive \citeyear{chen2024qp}}} & GCN   & 0.018\scriptsize$\pm$0.001 & 0.007\scriptsize$\pm$0.001  & 0.017\scriptsize$\pm$0.006   \\
	 &   & GIN & 0.019\scriptsize$\pm$0.001 & 0.002\scriptsize$\pm$0.001  & 0.002\scriptsize$\pm$0.001   \\
    \cmidrule{2-6}

    & {\multirow{2}{*}{IPM \citeyear{pmlr-v238-qian24a}}} & GCN   & 0.016\scriptsize$\pm$0.0003 & 0.003\scriptsize$\pm$0.0001  & 0.006\scriptsize$\pm$0.003  \\
	 &   & GIN & 0.011\scriptsize$\pm$0.001 & 0.002\scriptsize$\pm$0.001  & 0.003\scriptsize$\pm$0.001 \\
    \cmidrule{2-6}

    & {\multirow{2}{*}{IPM(ours)}} & GCN   & 0.029\scriptsize$\pm$0.009 & 0.012\scriptsize$\pm$0.005  & 0.015\scriptsize$\pm$0.005 \\
	 &   & GIN & 0.029\scriptsize$\pm$0.003 & 0.003\scriptsize$\pm$0.001 & 0.004\scriptsize$\pm$0.001 \\
    \cmidrule{2-6}

    & {\multirow{2}{*}{Feas.(ours)}}
     & GCN   & 1.142 \(\times\) \(10^{-7}\) & 3.489 \(\times\) \(10^{-7}\)  & 3.427 \(\times\) \(10^{-7}\) \\
	 &   & GIN & 1.167 \(\times\) \(10^{-7}\) & 3.418 \(\times\) \(10^{-7}\) & 3.427 \(\times\) \(10^{-7}\)  \\

    \bottomrule
\end{tabular}
}
\end{table}

\begin{table}[t!]
\caption{Comparisons of our IPM- and feasibility-oriented approaches with the MPNN-based prediction from \citet{chen2024qp} and \citet{pmlr-v238-qian24a}, all \textbf{without the global node}. The table reports the relative objective gap (in percentage) and constraint violation (in normalized absolute value). Training was repeated three times with different random seeds, and we report the mean and standard deviation across runs. For our feasibility method, the standard deviation of constraint violation is omitted as it is not meaningful.}
\label{tab:bipart_qp_main}
\centering
\resizebox{\columnwidth}{!}{
\begin{tabular}{cccccc}
\toprule
    & \textbf{Method}  & MPNN  & Generic & SVM  & Portfolio \\

    \midrule
    {\multirow{8}{*}{Obj. gap [\%]}}   
    
    & {\multirow{2}{*}{Naive \citeyear{chen2024qp}}} & GCN   & 2.547\scriptsize$\pm$0.126 & 0.446\scriptsize$\pm$0.102  & 11.075\scriptsize$\pm$0.363  \\
	 &   & GIN & 2.330\scriptsize$\pm$0.151 & 0.191\scriptsize$\pm$0.094  & 10.801\scriptsize$\pm$0.184  \\
    \cmidrule{2-6}

    & {\multirow{2}{*}{IPM \citeyear{pmlr-v238-qian24a}}} & GCN   & 1.894\scriptsize$\pm$0.166 & 0.175\scriptsize$\pm$0.022  & 1.497\scriptsize$\pm$0.686  \\
	 &   & GIN & 1.298\scriptsize$\pm$0.148  & 0.098\scriptsize$\pm$0.006  & 0.438\scriptsize$\pm$0.080  \\
    \cmidrule{2-6}

    & {\multirow{2}{*}{IPM(ours)}} & GCN   & 2.236\scriptsize$\pm$0.249 & 0.696\scriptsize$\pm$0.236  & 1.168\scriptsize$\pm$0.048 \\
	 &   & GIN & 2.746\scriptsize$\pm$0.152 & 0.478\scriptsize$\pm$0.244 & 0.657\scriptsize$\pm$0.121 \\
    \cmidrule{2-6}

    & {\multirow{2}{*}{Feas.(ours)}}
     & GCN   & 0.049\scriptsize$\pm$0.015 & 0.126\scriptsize$\pm$0.034  & 0.935\scriptsize$\pm$0.062 \\
	 &   & GIN & 0.142\scriptsize$\pm$0.034 & 0.091\scriptsize$\pm$0.016 & 1.105\scriptsize$\pm$0.094  \\
    \midrule
    {\multirow{8}{*}{Cons. vio.}}   

    & {\multirow{2}{*}{Naive \citeyear{chen2024qp}}} & GCN   & 0.012\scriptsize$\pm$0.004 & 0.013\scriptsize$\pm$0.009  & 0.083\scriptsize$\pm$0.001  \\
	 &   & GIN & 0.010\scriptsize$\pm$0.001 & 0.004\scriptsize$\pm$0.001  & 0.146\scriptsize$\pm$0.018  \\
    \cmidrule{2-6}

    & {\multirow{2}{*}{IPM \citeyear{pmlr-v238-qian24a}}} & GCN   & 0.013\scriptsize$\pm$0.001 & 0.004\scriptsize$\pm$0.001 & 0.011\scriptsize$\pm$0.005 \\
	 &   & GIN & 0.009\scriptsize$\pm$0.002 & 0.002\scriptsize$\pm$0.001  & 0.003\scriptsize$\pm$0.001 \\
    \cmidrule{2-6}

    & {\multirow{2}{*}{IPM(ours)}} & GCN   & 0.031\scriptsize$\pm$0.008 & 0.020\scriptsize$\pm$0.010 & 0.008\scriptsize$\pm$0.001 \\
	 &   & GIN & 0.029\scriptsize$\pm$0.003 & 0.011\scriptsize$\pm$0.003  & 0.005\scriptsize$\pm$0.001 \\
    \cmidrule{2-6}

    & {\multirow{2}{*}{Feas.(ours)}}
     & GCN   & 1.216 \(\times\) \(10^{-7}\) & 3.433 \(\times\) \(10^{-7}\)  & 2.179 \(\times\) \(10^{-7}\) \\
	 &   & GIN & 1.221 \(\times\) \(10^{-7}\) & 3.470 \(\times\) \(10^{-7}\) & 2.135 \(\times\) \(10^{-7}\)  \\

    \bottomrule
\end{tabular}
}
\end{table}

Furthermore, to compare against feasibility-related work, we generate another synthetic dataset of $\num{1000}$ LCQP instances with 200 variables, 50 equality constraints, and 200 trivial inequality constraints ($\vec{x}_i \geq 0$). The dataset is split into training, validation, and test sets with the ratio 8:1:1. The constraint matrix $\vec{A}$, quadratic matrix $\vec{Q}$, and objective vector $\vec{c}$ are shared across all instances. In contrast, only the right-hand side (RHS) $\vec{b}$ is randomized, following the setup described in DC3 \citep{donti2021dc3} and IPM-LSTM \citep{gao2024ipm}. Using their default hyperparameters, we evaluate DC3, IPM-LSTM, and our GCN-based feasibility method (without a global node, 32 inference steps). As shown in \cref{tab:qp200}, our method achieves the lowest objective gap and ensures strong feasibility. DC3 suffers from worse objective gaps due to post-processing and significant inequality violations, while IPM-LSTM exhibits higher equality constraint violations. DC3 is the fastest due to its simple architecture, followed by our method, with IPM-LSTM being the slowest owing to the computational expense of its LSTM-based architecture. For inference time, we benchmark these methods alongside three traditional solvers: OSQP \citep{stellato2020osqp}, CVXOPT \citep{andersen2013cvxopt}, and Gurobi \citep{gurobi}. While DC3 achieves solver-comparable performance, our method shows no advantage on these small, dense problems. These baselines have inflexible architectures, therefore restricted to fixed problem sizes, and cannot handle sparsity in large problems. In contrast, our MPNN-based approach is trainable on varying-sized datasets, captures problem sparsity, and performs efficiently on larger problem instances.

\begin{table}[t!]
\caption{Experiments on fixed-size instances. We compare our feasibility approach with DC3 \citep{donti2021dc3} and IPM-LSTM \citep{gao2024ipm}. We train each neural network once and report the mean and deviation on the test set. The table shows the relative objective gap, the equality and inequality constraint violations, and the inference time. The inference time of our approach includes the preprocessing time. }
\label{tab:qp200}
\centering
\resizebox{\columnwidth}{!}{
\begin{tabular}{ccccc}
\toprule
     Method & Rel. obj. (\%) &  Eq. cons. vio.  &  Ineq. cons. vio. & Time (sec) \\
    \midrule

    CVXOPT & -- & -- & -- & 0.009\scriptsize$\pm$0.003 \\
    OSQP & -- & -- & -- & 0.003\scriptsize$\pm$0.0003 \\
    Gurobi & -- & -- & -- & 0.006\scriptsize$\pm$0.0003 \\
    \midrule
    
    DC3 & 63.714\scriptsize$\pm$34.837 & 7.534 \(\times\) \(10^{-14}\) & 0.372\scriptsize$\pm$0.069 & 0.005\scriptsize$\pm$0.002 \\
    IPM-LSTM & 0.762\scriptsize$\pm$0.028 & 2.272 \(\times\) \(10^{-4}\) & 4.721 \(\times\) \(10^{-6}\) & 1.435\scriptsize$\pm$0.232 \\
    Feas. (ours) & 0.578\scriptsize$\pm$0.032 & 5.377 \(\times\) \(10^{-7}\) & 0 & 0.281\scriptsize$\pm$0.031 \\
    \bottomrule
\end{tabular}
}
\end{table}

\paragraph{Size generalization (Q2)} In response to Q2, we pre-train GCN-based architecture and evaluate them on larger problem instances, exploring two scaling approaches: (1) increasing size parameters while keeping the density constant and (2) increasing size parameters while maintaining a constant average node degree. We test with 16 and 32 iterations for both approaches to assess the impact of iteration count. As shown in \cref{tab:size_gen_generic}, objective gaps increase for all methods on larger instances, but our approach consistently outperforms the baselines. Violation values also rise for methods lacking feasibility guarantees. Maintaining a constant node degree yields better generalization than fixing density across all candidates. For additional results on SVM and portfolio problems, see \cref{sec:more_exps}.

We evaluate on real-world QPLIB instances \citep{furini2019qplib}, where conventional train-validation splits are impractical due to limited, diverse-sized data. To address this, we pre-train a GCN model on $\num{1000}$ large, sparse generic LCQP problems and test it on selected QPLIB instances with linear constraints, positive definite objectives, and memory manageable sizes (integer variables relaxed to continuous). As shown in \cref{tab:qplib}, our feasibility approach generalizes well to real-world problems, e.g., it obtains a relative objective error of 0.597\% on QPLIB\_3547. While errors are larger for some instances, e.g., QPLIB\_3547, absolute solution values remain satisfactory as the optimal value is near zero. Our method also generalizes to out-of-distribution unconstrained QPs (e.g., QPLIB\_8790 to QPLIB\_8991) with a relative error around $25\%$.

\begin{table}[t!]
\caption{The generalization performance on selected QPLIB instances. We show the details of problem size and density, the optimal objective value, and absolute value and relative error of our prediction.}
\label{tab:qplib}
\centering
\resizebox{\columnwidth}{!}{
\begin{tabular}{cccccccc}
\toprule
     QP ID & $\vec{A}$ dens. & $\vec{Q}$ dens. & \#cons.& \#vars. & Sol. & Pred. & Rel. error (\%)  \\
    \midrule
    
    3547 & 0.001 & 0.167 & 3137 & 1998 & 2.125 \(\times\) \(10^{4}\) & 2.138 \(\times\) \(10^{4}\) & 0.597 \\
    3694 & 0.001 & 0.0003 & 3280 & 3240 & 0.794 & 1.359 & 71.255 \\
    3698 & 0.001 & 0.0003 & 3100 & 3030 & 1.116 & 1.688 & 51.318 \\
    3792 & 0.001 & 0.0003 & 3150 & 3020 & 1.903 & 2.571 &  35.069 \\
    3861 & 0.0008 & 0.0002 & 4650 & 4530 & 1.329 & 1.939 &  45.972 \\
    3871 & 0.004 & 0.001 & 1040 & 1025 & 0.735 & 1.276 & 73.493 \\
    4270 & 0.002 & 0.251 & 1603 & 1600 & 0.183 & 0.593 &  224.023 \\
    \midrule
    8790 & -- & 0.0001 & 0 & 39204 & $-$3.920 \(\times\) \(10^{4}\) & $-$2.940 \(\times\) \(10^{4}\) &  25.000\\
    8792 & -- & 0.0003 & 0 & 15129 & $-$1.513 \(\times\) \(10^{4}\) & $-$1.134 \(\times\) \(10^{4}\) &  25.000\\
    8991 & -- & 0.0003 & 0 & 14400 & $-$1.430 \(\times\) \(10^{4}\) & $-$1.077 \(\times\) \(10^{4}\) &  24.690\\
    \bottomrule
\end{tabular}
}
\end{table}

\paragraph{Efficiency (Q3)}
To investigate computational efficiency, we evaluate the runtime performance on three QP problems in \cref{tab:runtime} and visualize the results on generic QPs in \cref{Plot}. We use the original test set for this evaluation and generate larger instances than those used in training. We compare our methods with the neural network baselines \citet{chen2024qp,pmlr-v238-qian24a} and solvers OSQP, CVXOPT, and Gurobi. We evaluate the neural network-based approaches with pre-trained GCN-based architecture. Both our approaches are evaluated with 16 and 32 iterations, and the impact of the global node on runtime is assessed. We also report data preparation time, including the null space calculation and finding the feasible solution. As shown in \cref{tab:runtime}, \citet{chen2024qp} and \citet{pmlr-v238-qian24a} achieve the fastest runtime due to their simple MPNN architectures and their iteration-free behavior. Despite the same architecture, the runtime of our methods is higher, as it depends on the number of iterations. Notably, runtime differences between our two approaches are minimal, as line search and null-space projection are computationally inexpensive compared to message passing. On generic problems, all the neural network approaches are significantly faster than the traditional QP solvers, even accounting for data preparation time. This gap widens with increasing problem size, showcasing the scalability of neural solvers. Traditional solvers like OSQP and Gurobi excel on SVM problems as they are specifically tailored for sparse problems. As expected, runtime increases slightly when using a global node due to additional convolutions.

\section{Conclusion}

We demonstrated that MPNNs can effectively solve convex LCQPs, significantly extending their known capabilities. Thereto, first, we established that MPNNs can theoretically simulate standard interior-point methods for solving LCQPs. Next, we proposed an enhanced MPNN architecture that ensures the feasibility of the predicted solutions through a novel projection approach. Empirically, our architecture outperformed existing neural approaches regarding solution quality and feasibility in an extensive empirical evaluation. Furthermore, our approaches generalized well to larger problem instances beyond the training set and, in some cases, achieved faster solution times than state-of-the-art solvers such as Gurobi.

\section*{Acknowledgements}
Christopher Morris and Chendi Qian are partially funded by a DFG Emmy Noether grant (468502433) and RWTH Junior Principal Investigator Fellowship under Germany’s Excellence Strategy. We thank Erik Müller for crafting the figures.

\section*{Impact statement}

This paper presents work that aims to advance the field of machine learning. Our work has many potential societal consequences, none of which must be specifically highlighted here.

\bibliography{bibliography}
\bibliographystyle{icml2025}

\newpage
\appendix
\onecolumn

\section{Appendix}

\subsection{Extended notation}\label{notation_app}
A \new{graph} $G$ is a pair $(V(G),E(G))$ with \emph{finite} sets of
\new{vertices} or \new{nodes} $V(G)$ and \new{edges} $E(G) \subseteq \{ \{u,v\} \subseteq V(G) \mid u \neq v \}$. An \new{attributed graph} $G$  is a triple $(V(G),E(G),a)$ with a graph $(V(G),E(G))$ and (vertex-)attribute function $a \colon V(G) \to \Rb^{1 \times d}$, for some $d > 0$. Then $a(v)$ are an \new{(node) attributes} or \new{features} of $v$, for $v$ in $V(G)$. Equivalently, we define an $n$-vertex attributed graph $G \coloneqq (V(G),E(G),a)$ as a pair $\mG=(G,\vec{L})$, where $G = (V(G),E(G))$ and $\vec{L}$ in $\Rb^{n\times d}$ is a \new{node attribute matrix}. Here, we identify $V(G)$ with $[n]$. For a matrix $\vec{L}$ in $\Rb^{n\times d}$ and $v$ in $[n]$, we denote by $\vec{L}$ in $\Rb^{1\times d}$ the $v$th row of $\vec{L}$ such that $\vec{L}_{v} \coloneqq a(v)$.  The \new{neighborhood} of $v$ in $V(G)$ is denoted by $N(v) \coloneqq  \{ u \in V(G) \mid (v, u) \in E(G) \}$.

\subsection{Additional related work}
\label{sec:more_literature}

Here, we discuss additional related work.

\paragraph{Machine learning for constrained optimization}
Training a neural network as a computationally efficient proxy but with constraints is also a widely studied topic, especially in real-world problems such as optimal power flow \citep{chatzos2020high,fioretto2020predicting,nellikkath2022physics}. 
A naive approach would be adding a penalty of constraint violation term to the loss function \citep{chatzos2020high,fioretto2020predicting,nellikkath2022physics,pmlr-v238-qian24a}. Recent methods fall into three categories: leveraging Lagrangian duality, designing specialized neural architectures, and post-processing outputs to enforce feasibility.
The first category applies Lagrangian duality to reformulate problems and solve primal-dual objectives \citep{fioretto2021lagrangian,park2023ss_pdl,kotary2024learning,klamkin2024dual}. While these approaches guarantee feasibility under ideal conditions, minor constraint violations can persist.
The second category focuses on architectural innovations. \citet{frerix2020homogeneous} embed homogeneous inequality constraints into activation functions. DC3 \citep{donti2021dc3} partially satisfies constraints using gradient descent but struggles to generalize to unseen data. LOOP-LC \citep{li2023learning} projects problems into $L_{\infty}$ spaces, which can be challenging to apply.
The final category adjusts neural outputs to satisfy constraints. \citet{chen2023end} develop problem-specific algorithms, while others \citep{pan2020deepopf,gros2020safe} use optimization to project results onto feasible regions. However, these methods are often computationally intensive, problem-specific, or limited in scope \citep{li2024onsmallgnn}.

\paragraph{Machine learning for combinatorial optimization}
Machine learning has been applied widely to combinatorial problems \citep{bengio2021machine,cappart2023combinatorial,peng2021graph}. 
For example, in the field of mixed integer linear programming (MILP), machine learning methods are explored to predict an initial solution and guide the search \citep{ding2020accelerating,khalil2022mip,han2023gnn,nair2020solving}. There are also extensive works for variable selection in branch and bound \citep{alvarez2017machine,khalil2016learning,Gas+2019,nair2020solving,zarpellon2021parameterizing,scavuzzo2022learning}, node selection \citep{He2014LearningTS,labassi2022learning}, and cutting-plane method \citep{paulus2022learning,tang2020reinforcement,turner2022adaptive,chetelat2023continuous}. 
Moreover, there are plenty works on other combinatorial problems, e.g., satisfiability (SAT) problem \citep{selsam2018learning,selsam2019guiding,toenshoff2021graph,shi2022satformer}, traveling salesman problem (TSP) \citep{joshi2019efficient,vinyals2015pointer,min2024unsupervised}, graph coloring \citep{lemos2019graph,li2022rethinking}, graph matching \citep{wang2019learning, Fey+2020, wang2020combinatorial}, among many others. 
As noted by \citet{jin2024unified}, CO problems that are naturally designed on graphs-such as TSP and graph coloring—can be seamlessly encoded into graph structures. CO problems without an inherent graph structure, like SAT problems and mixed-integer linear programming, can also be represented as graphs. For more detailed and exhaustive reviews on MPNN for MILP and other combinatorial optimization problems, we refer to \citet{scavuzzo2024machine,jin2024unified}. 

\subsection{Additional experiments}
\label{sec:more_exps}

Here, we report on additional experiments.

\subsubsection{Synchronized message passing}
We study the update sequence of message passing. We denote the message passing in \cref{eq:tri_MPNN_update1,eq:tri_MPNN_update2,eq:tri_MPNN_update3} and \cref{eq:MPNN_update} as asynchronous, as the node embeddings of some node types are updated first, while some others are updated with the latest updated node embeddings. We design the ablation of synchronous message passing of the form as follows, for tripartite \cref{eq:tri_MPNN_update_sync} and bipartite \cref{eq:bi_MPNN_update_sync}, respectively. 
\begin{equation}
\label{eq:tri_MPNN_update_sync}
\begin{aligned}
    \vec{h}_c^{(l,t)} \coloneqq \textsf{UPD}^{(l)}_{\text{c}}\Bigl[ & \vec{h}_c^{(l-1,t)}, \\
    &\textsf{MSG}^{(l)}_{\text{v} \rightarrow \text{c}}\left(\{\!\!\{ (\vec{h}_v^{(l-1,t)}, \vec{e}_{cv}) \mid v \in {N}\left(c \right) \cap V(I) \}\!\!\}  \right), \\
    &\textsf{MSG}^{(l)}_{\text{g} \rightarrow \text{c}}\left( \vec{h}_g^{(l-1,t)}, \vec{e}_{cg} \right) \Bigr] \in \mathbb{Q}^d, \\
    \vec{h}_g^{(l,t)} \coloneqq \textsf{UPD}^{(l)}_{\text{g}}\Bigl[ & \vec{h}_g^{(l-1,t)}, \\
    &\textsf{MSG}^{(l)}_{\text{v} \rightarrow \text{g}}\left(\{\!\!\{ (\vec{h}_v^{(l-1,t)}, \vec{e}_{vg}) \mid v \in V(I) \}\!\!\}  \right), \\
    &\textsf{MSG}^{(l)}_{\text{c} \rightarrow \text{g}}\left(\{\!\!\{ (\vec{h}_c^{(l-1,t)}, \vec{e}_{cg}) \mid c \in C(I) \}\!\!\}  \right) \Bigr] \in \mathbb{Q}^d, \\
    \vec{h}_v^{(l,t)} \coloneqq \textsf{UPD}^{(l)}_{\text{v}}\Bigl[ & \vec{h}_v^{(l-1,t)}, \\
    &\textsf{MSG}^{(l)}_{\text{v} \rightarrow \text{v}}\left(\{\!\!\{ (\vec{h}_u^{(l-1,t)}, \vec{e}_{uv}) \mid u \in {N}\left(v \right) \cap V(I) \}\!\!\}  \right),\\
    &\textsf{MSG}^{(l)}_{\text{c} \rightarrow \text{v}}\left(\{\!\!\{ (\vec{h}_c^{(l-1,t)}, \vec{e}_{cv}) \mid c \in {N}\left(v \right) \cap C(I) \}\!\!\}  \right), \\
    &\textsf{MSG}^{(l)}_{\text{g} \rightarrow \text{v}}\left(\vec{h}_g^{(l-1,t)}, \vec{e}_{vg}\right)\Bigr] \in \mathbb{Q}^d.
\end{aligned}
\end{equation}

\begin{equation}
\label{eq:bi_MPNN_update_sync}
\begin{aligned}
    \vec{h}_c^{(l,t)} \coloneqq \textsf{UPD}^{(l)}_{\text{c}}\Bigl[ \vec{h}_c^{(l-1,t)},          & \textsf{MSG}^{(l)}_{\text{v} \rightarrow \text{c}}\left(\{\!\!\{ (\vec{h}_v^{(l-1,t)}, \vec{e}_{cv}) \mid v \in {N}\left(c \right) \cap V(I) \}\!\!\}  \right) \Bigr] \in \mathbb{Q}^d, \\
    \vec{h}_v^{(l,t)} \coloneqq \textsf{UPD}^{(l)}_{\text{v}}\Bigl[ \vec{h}_v^{(l-1,t)},          & \textsf{MSG}^{(l)}_{\text{v} \rightarrow \text{v}}\left(\{\!\!\{ (\vec{h}_u^{(l-1,t)}, \vec{e}_{uv}) \mid u \in {N}\left(v \right) \cap V(I) \}\!\!\}  \right), \\
    &\textsf{MSG}^{(l)}_{\text{c} \rightarrow \text{v}}\left(\{\!\!\{ (\vec{h}_c^{(l-1,t)}, \vec{e}_{cv}) \mid u \in {N}\left(v \right) \cap C(I) \}\!\!\}  \right) \Bigr] \in \mathbb{Q}^d.
\end{aligned}
\end{equation}

\begin{table}[t!]
\caption{Ablation experiments of async-/synchronized message passing. We experiment on the generic dataset and GCN model. The table reports the relative objective gap (in percentage) and constraint violation (in normalized absolute value). The training was repeated three times with different random seeds, and we report the mean and standard deviation across runs. For our feasibility method, the standard deviation of constraint violation is omitted as it is not meaningful.}
\label{tab:sync_ablation}
\centering
\resizebox{0.5\columnwidth}{!}{
\begin{tabular}{ccccc}
\toprule
    & \textbf{Method}  & Global node  & Async.  &  Sync.  \\

    \midrule
    {\multirow{8}{*}{Obj. gap [\%]}}   
    
    & {\multirow{2}{*}{Naive \citeyear{chen2024qp}}} & \xxmark   & 2.547\scriptsize$\pm$0.126 & 2.447\scriptsize$\pm$0.091  \\
	 &   & $\checkmark$ & 2.281\scriptsize$\pm$0.150 & 2.400\scriptsize$\pm$0.143  \\
    \cmidrule{2-5}

    & {\multirow{2}{*}{IPM \citeyear{pmlr-v238-qian24a}}} & \xxmark    & 1.894\scriptsize$\pm$0.166 & 2.146\scriptsize$\pm$0.461  \\
	 &   & $\checkmark$ & 1.829\scriptsize$\pm$0.079  & 2.199\scriptsize$\pm$0.515  \\
    \cmidrule{2-5}

    & {\multirow{2}{*}{IPM(ours)}} & \xxmark   & 2.236\scriptsize$\pm$0.249 & 2.175\scriptsize$\pm$0.195 \\
	 &   & $\checkmark$ & 1.661\scriptsize$\pm$0.199 & 2.398\scriptsize$\pm$0.409 \\
    \cmidrule{2-5}

    & {\multirow{2}{*}{Feas.(ours)}}
     & \xxmark  & 0.049\scriptsize$\pm$0.015 & 0.141\scriptsize$\pm$0.012 \\
	 &   & $\checkmark$ & 0.070\scriptsize$\pm$0.004 & 0.207\scriptsize$\pm$0.047  \\
    \midrule
    {\multirow{8}{*}{Cons. vio.}}   

    & {\multirow{2}{*}{Naive \citeyear{chen2024qp}}} & \xxmark   & 0.012\scriptsize$\pm$0.004 & 0.031\scriptsize$\pm$0.004  \\
	 &   & $\checkmark$ & 0.018\scriptsize$\pm$0.001 & 0.029\scriptsize$\pm$0.003  \\
    \cmidrule{2-5}

    & {\multirow{2}{*}{IPM \citeyear{pmlr-v238-qian24a}}} & \xxmark  & 0.013\scriptsize$\pm$0.001 & 0.020\scriptsize$\pm$0.002 \\
	 &   & $\checkmark$ & 0.016\scriptsize$\pm$0.0003 & 0.015\scriptsize$\pm$0.003 \\
    \cmidrule{2-5}

    & {\multirow{2}{*}{IPM(ours)}} & \xxmark   & 0.031\scriptsize$\pm$0.008 & 0.032\scriptsize$\pm$0.001  \\
	 &   & $\checkmark$ & 0.029\scriptsize$\pm$0.009 & 0.046\scriptsize$\pm$0.007 \\
    \cmidrule{2-5}

    & {\multirow{2}{*}{Feas.(ours)}}
     & \xxmark & 1.216 \(\times\) \(10^{-7}\) & 1.441 \(\times\) \(10^{-7}\)  \\
	 &   & $\checkmark$ & 1.142 \(\times\) \(10^{-7}\) & 1.567 \(\times\) \(10^{-7}\) \\

    \bottomrule
\end{tabular}
}
\end{table}

We select the GCN architecture and the generic QP dataset as representative; see \cref{tab:sync_ablation} for results. Our feasibility-guaranteeing MPNNs get better results with asynchronous message passing, but there are no consistent and significant differences for other methods. 

\subsubsection{More experiments on generalization performance}
We report size generalization experiments on SVM and portfolio problems in \cref{tab:size_gen_svm,tab:size_gen_port}.

\begin{table*}[t]
\caption{Soft-margin SVM problem, size generalization. We fix the density hyperparameters or the average degree of nodes compared to the training set. The star symbol * indicates training sizes. The methods with postfix $\text{-G}$ us the global node.}
\label{tab:size_gen_svm}
\centering
\resizebox{.7\textwidth}{!}{
\begin{tabular}{ccccccc}
\toprule
& \textbf{Fix} & -- &  \multicolumn{2}{c}{\textbf{Density}} & \multicolumn{2}{c}{\textbf{Degree}} \\
& \textbf{Size} & 400* & 600 & 800 & 600 & 800 \\
\midrule

{\multirow{12}{*}{Obj. gap [\%]}}
& Naive \citeyear{chen2024qp} & 0.446\scriptsize$\pm$0.102 & 5.416\scriptsize$\pm$1.242 & 15.913\scriptsize$\pm$2.946
& 1.066\scriptsize$\pm$0.295 & 1.003\scriptsize$\pm$0.285 \\
& Naive-G \citeyear{chen2024qp} & 0.352\scriptsize$\pm$0.032 & 7.920\scriptsize$\pm$2.435 & 19.733\scriptsize$\pm$5.212
& 2.046\scriptsize$\pm$0.649 & 1.895\scriptsize$\pm$0.708 \\
\cmidrule{2-7}
& IPM \citeyear{pmlr-v238-qian24a} & 0.175\scriptsize$\pm$0.022 & 5.775\scriptsize$\pm$1.171 & 16.257\scriptsize$\pm$4.036
& 0.946\scriptsize$\pm$0.153 & 0.890\scriptsize$\pm$0.183\\
& IPM-G \citeyear{pmlr-v238-qian24a} & 0.162\scriptsize$\pm$0.009 & 3.660\scriptsize$\pm$0.766 & 9.564\scriptsize$\pm$1.391
& 0.784\scriptsize$\pm$0.265 & 0.722\scriptsize$\pm$0.255 \\
\cmidrule{2-7}
& IPM$_{16}$ (Ours) & 2.352\scriptsize$\pm$0.525 & 5.601\scriptsize$\pm$1.383 & 9.342\scriptsize$\pm$1.674
& 4.965\scriptsize$\pm$1.597 & 5.051\scriptsize$\pm$1.616 \\
& IPM$_{32}$ (Ours) & 0.696\scriptsize$\pm$0.236 & 2.509\scriptsize$\pm$1.238 & 4.965\scriptsize$\pm$2.160
& 4.855\scriptsize$\pm$1.614 & 4.918\scriptsize$\pm$1.697\\
& IPM-G$_{16}$ (Ours) & 1.744\scriptsize$\pm$0.108 & 2.238\scriptsize$\pm$0.893 & 3.705\scriptsize$\pm$1.874
& 2.367\scriptsize$\pm$0.841 & 2.340\scriptsize$\pm$0.937 \\
& IPM-G$_{32}$ (Ours) & 0.485\scriptsize$\pm$0.176 & 2.163\scriptsize$\pm$1.573 & 3.964\scriptsize$\pm$2.062
& 2.436\scriptsize$\pm$0.883 & 2.388\scriptsize$\pm$0.935\\
\cmidrule{2-7}
& Feas.$_{16}$ (Ours) & 0.222\scriptsize$\pm$0.049 & 0.536\scriptsize$\pm$0.022 & 1.841\scriptsize$\pm$0.236
& 0.511\scriptsize$\pm$0.176 & 0.584\scriptsize$\pm$0.231 \\
& Feas.$_{32}$ (Ours) & 0.126\scriptsize$\pm$0.034 & 0.404\scriptsize$\pm$0.053 & 1.519\scriptsize$\pm$0.100
& 0.387\scriptsize$\pm$0.147 & 0.421\scriptsize$\pm$0.174 \\
& Feas.-G$_{16}$ (Ours) & 0.254\scriptsize$\pm$0.007 & 0.503\scriptsize$\pm$0.053 & 1.430\scriptsize$\pm$0.212
& 0.282\scriptsize$\pm$0.031 & 0.297\scriptsize$\pm$0.024 \\
& Feas.-G$_{32}$ (Ours) & 0.190\scriptsize$\pm$0.008 & 0.364\scriptsize$\pm$0.027 & 1.160\scriptsize$\pm$0.123
& 0.185\scriptsize$\pm$0.012 & 0.171\scriptsize$\pm$0.014 \\
\midrule

{\multirow{12}{*}{Cons. vio.}}
& Naive \citeyear{chen2024qp} & 0.013\scriptsize$\pm$0.009 & 0.051\scriptsize$\pm$0.011 & 0.133\scriptsize$\pm$0.025
& 0.017\scriptsize$\pm$0.002 & 0.016\scriptsize$\pm$0.002  \\
& Naive-G \citeyear{chen2024qp} & 0.007\scriptsize$\pm$0.001 & 0.066\scriptsize$\pm$0.021 & 0.179\scriptsize$\pm$0.047
& 0.019\scriptsize$\pm$0.003 & 0.018\scriptsize$\pm$0.003 \\
\cmidrule{2-7}
& IPM \citeyear{pmlr-v238-qian24a} & 0.004\scriptsize$\pm$0.001 & 0.051\scriptsize$\pm$0.010 & 0.155\scriptsize$\pm$0.031
& 0.011\scriptsize$\pm$0.003 & 0.011\scriptsize$\pm$0.003\\
& IPM-G \citeyear{pmlr-v238-qian24a} & 0.003\scriptsize$\pm$0.0001 & 0.035\scriptsize$\pm$0.006 & 0.107\scriptsize$\pm$0.018
& 0.007\scriptsize$\pm$0.001 & 0.007\scriptsize$\pm$0.001 \\
\cmidrule{2-7}
& IPM$_{16}$ (Ours) & 0.015\scriptsize$\pm$0.007 & 0.074\scriptsize$\pm$0.025 & 0.136\scriptsize$\pm$0.027
& 0.079\scriptsize$\pm$0.008 & 0.080\scriptsize$\pm$0.008\\
& IPM$_{32}$ (Ours) & 0.020\scriptsize$\pm$0.010 & 0.082\scriptsize$\pm$0.026  & 0.145\scriptsize$\pm$0.024
& 0.074\scriptsize$\pm$0.012 & 0.075\scriptsize$\pm$0.012\\
& IPM-G$_{16}$ (Ours) & 0.010\scriptsize$\pm$0.005 & 0.047\scriptsize$\pm$0.011 & 0.106\scriptsize$\pm$0.023
& 0.055\scriptsize$\pm$0.011 & 0.055\scriptsize$\pm$0.012 \\
& IPM-G$_{32}$ (Ours) & 0.012\scriptsize$\pm$0.005 & 0.051\scriptsize$\pm$0.012 & 0.111\scriptsize$\pm$0.023
& 0.054\scriptsize$\pm$0.012 & 0.053\scriptsize$\pm$0.013 \\
\cmidrule{2-7}
& Feas.$_{16}$ (Ours) & 3.433 \(\times\) \(10^{-7}\) & 3.148 \(\times\) \(10^{-7}\) & 3.357 \(\times\) \(10^{-7}\)
& 2.553 \(\times\) \(10^{-7}\) & 2.782 \(\times\) \(10^{-7}\) \\
& Feas.$_{32}$ (Ours) & 3.486 \(\times\) \(10^{-7}\) & 3.204 \(\times\) \(10^{-7}\) & 3.412 \(\times\) \(10^{-7}\)
& 2.627 \(\times\) \(10^{-7}\) & 2.859 \(\times\) \(10^{-7}\) \\
& Feas.-G$_{16}$ (Ours) & 3.376 \(\times\) \(10^{-7}\) & 3.054 \(\times\) \(10^{-7}\) & 3.328 \(\times\) \(10^{-7}\)
& 2.485 \(\times\) \(10^{-7}\) & 2.750 \(\times\) \(10^{-7}\) \\
& Feas.-G$_{32}$ (Ours) & 3.489 \(\times\) \(10^{-7}\) & 3.127 \(\times\) \(10^{-7}\) & 3.369 \(\times\) \(10^{-7}\)
& 2.489 \(\times\) \(10^{-7}\) & 2.962 \(\times\) \(10^{-7}\) \\

\bottomrule

\end{tabular}
}
\end{table*}

\begin{table*}[t]
\caption{Markowitz portfolio problem, size generalization. We fix the density hyperparameters or the average degree of nodes as compared to the training set. The size with the star symbol * is where we train the models. The methods with postfix $\text{-G}$ are with the global node. }
\label{tab:size_gen_port}
\centering
\resizebox{.7\textwidth}{!}{
\begin{tabular}{ccccccc}
\toprule
& \textbf{Fix} & -- &  \multicolumn{2}{c}{\textbf{Density}} & \multicolumn{2}{c}{\textbf{Degree}} \\
& \textbf{Size} & 800* & 1000 & 1200 & 1000 & 1200 \\
\midrule

{\multirow{12}{*}{Obj. gap [\%]}}
& Naive \citeyear{chen2024qp} & 11.075\scriptsize$\pm$0.363 & 39.487\scriptsize$\pm$1.472 & 89.422\scriptsize$\pm$8.968
& 37.366\scriptsize$\pm$0.202 & 80.900\scriptsize$\pm$1.199 \\
& Naive-G \citeyear{chen2024qp} & 2.446\scriptsize$\pm$0.756 & 50.303\scriptsize$\pm$2.226 & 114.081\scriptsize$\pm$2.407
& 43.062\scriptsize$\pm$3.737 & 103.827\scriptsize$\pm$4.694 \\
\cmidrule{2-7}
& IPM \citeyear{pmlr-v238-qian24a} & 11.075\scriptsize$\pm$0.363 & 59.571\scriptsize$\pm$3.657 & 123.412\scriptsize$\pm$16.006
& 53.575\scriptsize$\pm$5.673 & 122.935\scriptsize$\pm$17.091 \\
& IPM-G \citeyear{pmlr-v238-qian24a} & 0.717\scriptsize$\pm$0.353 & 60.532\scriptsize$\pm$2.009 & 123.601\scriptsize$\pm$6.519
& 53.641\scriptsize$\pm$3.141 & 127.39\scriptsize$\pm$10.328 \\
\cmidrule{2-7}
& IPM$_{16}$ (Ours) & 1.241\scriptsize$\pm$0.123 & 59.885\scriptsize$\pm$1.284 & 134.475\scriptsize$\pm$4.046
& 48.010\scriptsize$\pm$1.222 & 107.825\scriptsize$\pm$3.340\\
& IPM$_{32}$ (Ours) & 1.168\scriptsize$\pm$0.048 & 58.791\scriptsize$\pm$0.902 & 133.522\scriptsize$\pm$3.104
& 46.299\scriptsize$\pm$0.963 & 104.649\scriptsize$\pm$3.020\\
& IPM-G$_{16}$ (Ours) & 1.672\scriptsize$\pm$0.302 & 50.184\scriptsize$\pm$2.841 & 120.491\scriptsize$\pm$3.303
& 41.294\scriptsize$\pm$3.882 & 97.128\scriptsize$\pm$10.011 \\
& IPM-G$_{32}$ (Ours) & 1.838\scriptsize$\pm$0.720 & 48.149\scriptsize$\pm$3.959 & 116.984\scriptsize$\pm$3.736
& 39.947\scriptsize$\pm$4.891 & 95.537\scriptsize$\pm$10.651 \\
\cmidrule{2-7}
& Feas.$_{16}$ (Ours) & 1.069\scriptsize$\pm$0.135 & 30.945\scriptsize$\pm$7.768 & 69.905\scriptsize$\pm$19.248
& 24.286\scriptsize$\pm$7.778 & 44.202\scriptsize$\pm$12.299 \\
& Feas.$_{32}$ (Ours) & 0.935\scriptsize$\pm$0.062 & 18.755\scriptsize$\pm$0.100 & 43.238\scriptsize$\pm$13.631
& 13.420\scriptsize$\pm$4.255 & 25.170\scriptsize$\pm$6.707 \\
& Feas.-G$_{16}$ (Ours) & 1.076\scriptsize$\pm$0.161 & 8.211\scriptsize$\pm$2.691 & 22.804\scriptsize$\pm$7.514
& 5.348\scriptsize$\pm$1.463 & 10.702\scriptsize$\pm$3.327 \\
& Feas.-G$_{32}$ (Ours) & 1.005\scriptsize$\pm$0.159 & 5.303\scriptsize$\pm$1.279 & 14.530\scriptsize$\pm$4.671
& 3.412\scriptsize$\pm$0.835 & 6.389\scriptsize$\pm$1.285 \\
\midrule

{\multirow{12}{*}{Cons. vio.}}
& Naive \citeyear{chen2024qp} & 0.083\scriptsize$\pm$0.001 & 0.266\scriptsize$\pm$0.011 & 0.397\scriptsize$\pm$0.015
& 0.271\scriptsize$\pm$0.008 & 0.382\scriptsize$\pm$0.009  \\
& Naive-G \citeyear{chen2024qp} & 0.017\scriptsize$\pm$0.006 & 0.184\scriptsize$\pm$0.012 & 0.330\scriptsize$\pm$0.017
& 0.176\scriptsize$\pm$0.012 & 0.331\scriptsize$\pm$0.017 \\
\cmidrule{2-7}
& IPM \citeyear{pmlr-v238-qian24a} & 0.011\scriptsize$\pm$0.005 & 0.237\scriptsize$\pm$0.027 & 0.435\scriptsize$\pm$0.039
& 0.235\scriptsize$\pm$0.034 & 0.440\scriptsize$\pm$0.048 \\
& IPM-G \citeyear{pmlr-v238-qian24a} & 0.006\scriptsize$\pm$0.003 & 0.239\scriptsize$\pm$0.004 & 0.409\scriptsize$\pm$0.011
& 0.231\scriptsize$\pm$0.002 & 0.422\scriptsize$\pm$0.013 \\
\cmidrule{2-7}
& IPM$_{16}$ (Ours) & 0.007\scriptsize$\pm$0.001 & 0.192\scriptsize$\pm$0.010 & 0.352\scriptsize$\pm$0.015
& 0.175\scriptsize$\pm$0.007 & 0.331\scriptsize$\pm$0.008\\
& IPM$_{32}$ (Ours) & 0.008\scriptsize$\pm$0.000 & 0.187\scriptsize$\pm$0.012 & 0.347\scriptsize$\pm$0.016
& 0.169\scriptsize$\pm$0.008 & 0.319\scriptsize$\pm$0.009\\
& IPM-G$_{16}$ (Ours) & 0.011\scriptsize$\pm$0.003 & 0.156\scriptsize$\pm$0.019 & 0.288\scriptsize$\pm$0.040
& 0.153\scriptsize$\pm$0.015 & 0.288\scriptsize$\pm$0.039 \\
& IPM-G$_{32}$ (Ours) & 0.015\scriptsize$\pm$0.005 & 0.162\scriptsize$\pm$0.015 & 0.297\scriptsize$\pm$0.038
& 0.160\scriptsize$\pm$0.011 & 0.293\scriptsize$\pm$0.037 \\
\cmidrule{2-7}
& Feas.$_{16}$ (Ours) & 2.228 \(\times\) \(10^{-8}\) & 3.054 \(\times\) \(10^{-8}\) & 3.377 \(\times\) \(10^{-8}\)
& 2.779 \(\times\) \(10^{-8}\) & 3.249 \(\times\) \(10^{-8}\) \\
& Feas.$_{32}$ (Ours) & 2.179 \(\times\) \(10^{-8}\) & 2.873 \(\times\) \(10^{-8}\) & 3.491 \(\times\) \(10^{-8}\)
& 2.693 \(\times\) \(10^{-8}\) & 3.386 \(\times\) \(10^{-8}\) \\
& Feas.-G$_{16}$ (Ours) & 2.086 \(\times\) \(10^{-8}\) & 2.570 \(\times\) \(10^{-8}\) & 2.980 \(\times\) \(10^{-8}\)
& 2.384 \(\times\) \(10^{-8}\) & 3.203 \(\times\) \(10^{-8}\) \\
& Feas.-G$_{32}$ (Ours) & 3.427 \(\times\) \(10^{-8}\) & 8.568 \(\times\) \(10^{-8}\) & 2.235 \(\times\) \(10^{-8}\)
& 2.692 \(\times\) \(10^{-8}\) & 2.160 \(\times\) \(10^{-8}\) \\

\bottomrule

\end{tabular}
}
\end{table*}

\subsubsection{LP as special QP}
\label{sec:lp_as_qp}
We observe that LPs are special cases of QPs. If we remove quadratic term in the objective of \cref{eq:standard_qp}, we arrive at a standard LP form
\begin{equation*}
\begin{aligned}
\label{eq:standard_lp}
\min_{\vec{x} \in \mathbb{Q}^n_{\geq 0}} \quad & \vec{c}\trans \vec{x} \\
\text{s.t.} \quad & \vec{A} \vec{x} = \vec{b}.
\end{aligned}
\end{equation*}

Since there is no quadratic matrix, no edges between variable nodes exist. The graph representation is similar to \citet{chen2022representing} for bipartite graphs and \citet{pmlr-v238-qian24a,ding2020accelerating} for tripartite graphs. We generate LP instances by relaxing well-known mixed-integer linear programming problems, similar to \citet{pmlr-v238-qian24a}; see  \cref{tab:main_lp} for results.  

\begin{table}[t]
\caption{Experiment results on LP instances. The methods with postfix -G is with the global node. We do not show the violation of the feasibility method as it is not informative. }
\label{tab:main_lp}
\centering
\resizebox{0.6\columnwidth}{!}{
\begin{tabular}{ccccccc}
\toprule
    & \textbf{Method}  & MPNN  & Setcover & Indset  & Cauc & Fac \\

	\midrule
	{\multirow{8}{*}{Obj. gap [\%]}}   
    & {\multirow{2}{*}{Naive \citeyear{chen2022representing}}}
     & GCN   & 0.743\scriptsize$\pm$0.013 & 0.380\scriptsize$\pm$0.035  & 0.630\scriptsize$\pm$0.074 &  0.389\scriptsize$\pm$0.037 \\
	 &   & GIN & 0.681\scriptsize$\pm$0.017 & 0.408\scriptsize$\pm$0.019  & 0.465\scriptsize$\pm$0.008 &  0.329\scriptsize$\pm$0.003 \\
    \cmidrule{2-7}

    & {\multirow{2}{*}{Naive-G \citeyear{chen2022representing}}}
     & GCN   & 0.706\scriptsize$\pm$0.036 & 0.357\scriptsize$\pm$0.021  & 0.557\scriptsize$\pm$0.092 &  0.336\scriptsize$\pm$0.028 \\
	 &   & GIN & 0.641\scriptsize$\pm$0.024 & 0.401\scriptsize$\pm$0.036  & 0.460\scriptsize$\pm$0.026 &  0.620\scriptsize$\pm$0.089 \\
    \cmidrule{2-7}
    
    & {\multirow{2}{*}{Feas.}}  & GCN &  0.101\scriptsize$\pm$0.012 & 0.065\scriptsize$\pm$0.005  & 0.425\scriptsize$\pm$0.041 &  0.072\scriptsize$\pm$0.016 \\
    &  & GIN      & 0.123\scriptsize$\pm$0.026 & 0.088\scriptsize$\pm$0.014  & 0.529\scriptsize$\pm$0.007 &  0.063\scriptsize$\pm$0.034 \\

    \cmidrule{2-7}
    
    & {\multirow{2}{*}{Feas.-G}}  & GCN  & 0.120\scriptsize$\pm$0.026 & 0.089\scriptsize$\pm$0.030  & 0.338\scriptsize$\pm$0.034 &  0.069\scriptsize$\pm$0.007 \\
    &  & GIN      & 0.151\scriptsize$\pm$0.038 & 0.080\scriptsize$\pm$0.008  & 0.333\scriptsize$\pm$0.010 &  0.033\scriptsize$\pm$0.009 \\
    
    \midrule

    {\multirow{4}{*}{Cons. vio.}} 
    & {\multirow{2}{*}{Naive \citeyear{chen2022representing}}}
    & GCN   & 0.024\scriptsize$\pm$0.003 & 0.023\scriptsize$\pm$0.003  & 0.028\scriptsize$\pm$0.002 &  0.012\scriptsize$\pm$0.001 \\
    &   & GIN   & 0.026\scriptsize$\pm$0.003 & 0.025\scriptsize$\pm$0.002  & 0.025\scriptsize$\pm$0.002 &  0.009\scriptsize$\pm$0.002 \\

    \cmidrule{2-7}

    & {\multirow{2}{*}{Naive-G \citeyear{chen2022representing}}}
    & GCN   & 0.027\scriptsize$\pm$0.003 & 0.024\scriptsize$\pm$0.003  & 0.028\scriptsize$\pm$0.001 &  0.017\scriptsize$\pm$0.005 \\
    &   & GIN   & 0.033\scriptsize$\pm$0.007 & 0.025\scriptsize$\pm$0.001  & 0.025\scriptsize$\pm$0.001 &  0.014\scriptsize$\pm$0.003 \\
    
    \bottomrule
\end{tabular}
}
\end{table}

\subsection{Datasets}

Here, we give details on dataset generation. 

\label{sec:ds}
\paragraph{Generic QP} For generic QP problems, we consider the standard form of QP but with inequalities,
\begin{equation}
\begin{aligned}
\label{eq:ineq_qp}
\min_{\vec{x} \in \mathbb{R}^n_{\geq 0}} \quad & \frac{1}{2} \vec{x}\trans \vec{Q} \vec{x} + \vec{c}\trans \vec{x} \\
\text{s.t.} \quad & \vec{A} \vec{x} \leq \vec{b}.
\end{aligned}
\end{equation}
We generate the matrix $\vec{A}$ and vectors $\vec{c}, \vec{b}$ with entries drawn i.i.d.\ from the standard normal distribution $\mathcal{N}(0, 1)$. To maintain sparsity, we independently drop out each entry of $\vec{A}$ with probability $\tau$ using a Bernoulli distribution $\mathcal{B}(\tau)$, setting the dropped entries to zero. We generate the quadratic matrix $\vec{Q}$ simply with the $\texttt{make\_sparse\_spd\_matrix}$ function from SciPy given the desired density. Finally, we add slack variables to the constraints to make them into equalities. 

\paragraph{Soft margin SVM} For the QP problems generated from SVMs \citep{bishop2006pattern}, we follow the form,
\begin{equation*}
\begin{aligned}
\label{eq:svm_qp}
\min_{\vec{w}} \quad & \vec{w}\trans \vec{w} + \lambda \bm{1}\trans \boldsymbol{\xi} \\
\text{s.t.} \quad & \vec{y} \odot \vec{X} \vec{w} \geq \bm{1} - \boldsymbol{\xi} \\
& \boldsymbol{\xi} \in \mathbb{R}^n_{\geq 0}.
\end{aligned}
\end{equation*}
The $\vec{w} \in \mathbb{R}^n$ above denotes the vector of learnable parameters in the SVM we try to optimize, and $(\vec{X} \, \vec{y} \in \mathbb{R}^m)$ are the data points. Note that we have no constraints on $m, n$, and $\vec{X}$ must not be full rank. Here, $\boldsymbol{\xi} \in \mathbb{R}^m$ is the margin parameters we try to minimize, and $\odot$ is element-wise multiplication. Given density hyperparameter $\tau \in (0,1)$, we generate two sub-matrices $\vec{X}_1, \vec{X}_2 \in \mathbb{R}^{\frac{m}{2} \times n}$, with the entries drawn i.i.d.\ from the normal distributions $\mathcal{N}(-\dfrac{1}{n \tau}, \dfrac{1}{n \tau})$ and $\mathcal{N}(\dfrac{1}{n \tau}, \dfrac{1}{n \tau})$, respectively, and apply random drop for sparsity. The labels are $\{1, -1\}$ for the data points in the two sub-matrices. Finally, we also add slack variables to turn the inequalities into equalities.

\paragraph{Markowitz portfolio optimization} There are various formulations of the Markowitz portfolio optimization problem. We consider the following form,
\begin{equation*}
\begin{aligned}
\min_{\vec{x} \in \mathbb{R}^n_{\geq 0}} \quad & \vec{x}\trans \boldsymbol{\Sigma} \vec{x} \\
\text{s.t.} \quad & \boldsymbol{\mu}\trans \vec{x} = r \\
& \bm{1}\trans \vec{x} = 1.
\end{aligned}
\end{equation*}
We generate the symmetric, PSD matrix $\boldsymbol{\Sigma}$ again with the $\texttt{make\_sparse\_spd\_matrix}$ function from SciPy, and we sample the entries of $\boldsymbol{\mu}$ i.i.d.\ from the normal distribution $\mathcal{N}(0, 1)$. Here, $r$ is sampled from the uniform distribution $\mathcal{U}(0, 1)$. 

\paragraph{LP instances}
For the LP instances in \cref{sec:lp_as_qp}, we follow the setting of \citet{pmlr-v238-qian24a}.

\paragraph{Dataset hyperparameters}
Here, we list the hyperparameters of our dataset generation. For generic QP problems, we generate QPs of the form \cref{eq:ineq_qp} and use equality for the constraints. \Cref{tab:ds_gen} lists the configurations for training and size generalization experiments and the hyperparameters with which we generate a dataset to train a GCN for QPLIB experiments. 

\begin{table}[t]
\caption{Hyperparameters for generating generic QP instances. The comments in the bracket indicate whether we fix the density or average node degree compared with the training data.}
\label{tab:ds_gen}
\centering
\resizebox{0.7\columnwidth}{!}{
\begin{tabular}{cccccc}
\toprule
  \textbf{Dataset}  & \#cons.  & \#vars.  & $\vec{A}$ dens. & $\vec{Q}$ dens.  & nums. \\
    \midrule
  Training & 400 & 400 & 0.01  & 0.01 & 1000 \\
  Larger (dens.) & 600 & 600 & 0.01 & 0.01 & 100 \\
  Largest (dens.) & 800 & 800 & 0.01 & 0.01 & 100 \\
    \midrule
  Larger (deg.) & 600 & 600 & 0.005 & 0.007 & 100 \\
  Largest (deg.) & 800 & 800 & 0.004 & 0.005 & 100 \\
  \midrule
  QPLIB & [2000, 3000] & [2000, 3000] & [\num{1.e-4}, \num{1.e-3}] & [\num{1.e-5}, \num{1.e-3}] & 1000 \\
    \bottomrule
\end{tabular}
}
\end{table}
We generate instances for SVM problems with the hyperparameters in \cref{tab:ds_svm}. There is no hyperparameter for the density of the quadratic matrix, as it is always diagonal. 

\begin{table}[t]
\caption{Hyperparameters for generating SVM instances. The comments in the bracket indicate whether we fix the density or average node degree compared with the training data.}
\label{tab:ds_svm}
\centering
\resizebox{0.4\columnwidth}{!}{
\begin{tabular}{ccccc}
\toprule
  Dataset  & \#cons.  & \#vars.  &$\vec{A}$ dens.  & nums. \\
    \midrule
  Training & 400 & 400 & 0.01 & 1000 \\
  Larger (dens.) & 600 & 600 & 0.01 & 100 \\
  Largest (dens.) & 800 & 800 & 0.01 & 100 \\
    \midrule
  Larger (deg.) & 600 & 600 & 0.008 & 100 \\
  Largest (deg.) & 800 & 800 & 0.006 & 100 \\
    \bottomrule
\end{tabular}
}
\end{table}

The hyperparameters of portfolio problems are shown in \cref{tab:ds_port}. There is no hyperparameter for the number of constraints, as it is a constant. However, we have control over the density of the quadratic matrix. 
\begin{table}[t]
\caption{Hyperparameters for generating portfolio instances. The comments in the bracket indicate whether we fix the density or average node degree compared with the training data.}
\label{tab:ds_port}
\centering
\resizebox{0.3\columnwidth}{!}{
\begin{tabular}{ccccc}
\toprule
  Dataset & \#vars. & $\vec{Q}$ dens.  & nums. \\
    \midrule
  Training & 800 & 0.01 & 1000 \\
  Larger (dens.) & 1000 & 0.01 & 100 \\
  Largest (dens.) & 1200 & 0.01 & 100 \\
    \midrule
  Larger (deg.) & 1000 & 0.008 & 100 \\
  Largest (deg.) & 1200 & 0.006 & 100 \\
    \bottomrule
\end{tabular}
}
\end{table}
Finally, the hyperparameters for the LP instances in \cref{tab:main_lp} are shown in \cref{tab:ds_setc,tab:ds_mis,tab:ds_cauc,tab:ds_fac}.

\begin{table}[t]
\caption{Hyperparameters for generating set cover problem instances.}
\label{tab:ds_setc}
\centering
\resizebox{0.4\columnwidth}{!}{
\begin{tabular}{ccccc}
\toprule
  Dataset  & \#cons.  & \#vars.  & $\vec{A}$ dens.  & nums. \\
    \midrule
  Set cover & [200,300] & [300,400] & 0.008 & 1000 \\
    \bottomrule
\end{tabular}
}
\end{table}

\begin{table}[t]
\caption{Hyperparameters for generating maximum independent set problem instances.}
\label{tab:ds_mis}
\centering
\resizebox{0.4\columnwidth}{!}{
\begin{tabular}{cccc}
\toprule
  Dataset  & \#nodes  & $p(u,v), u,v \in E(G)$  & nums. \\
    \midrule
  Max ind. set & [250,300] & 0.01 & 1000 \\
    \bottomrule
\end{tabular}
}
\end{table}

\begin{table}[t]
\caption{Hyperparameters for generating combinatorial auction problem instances.}
\label{tab:ds_cauc}
\centering
\resizebox{0.3\columnwidth}{!}{
\begin{tabular}{cccc}
\toprule
  Dataset  & \#items.  & \#bids  & nums. \\
    \midrule
  Comb. auc. & [300,400] & [300,400] & 1000 \\
    \bottomrule
\end{tabular}
}
\end{table}

\begin{table}[t]
\caption{Hyperparameters for generating capacitated facility location problem instances.}
\label{tab:ds_fac}
\centering
\resizebox{0.3\columnwidth}{!}{
\begin{tabular}{ccccc}
\toprule
  Dataset  & \#custom  & \#fac. & ratio  & nums. \\
    \midrule
  Cap. fac. loc. & [60,70] & 5 & 0.5 & 1000 \\
    \bottomrule
\end{tabular}
}
\end{table}

\begin{table*}[t]
\caption{Generic QP problem, size generalization. We fix the density hyperparameters or the average degree of nodes as compared to the training set. The size with the star symbol * is where we train the models on. The methods with postfix $\text{-G}$ are with the global node.}
\label{tab:size_gen_generic}
\centering
\resizebox{.7\textwidth}{!}{
\begin{tabular}{ccccccc}
\toprule
& \textbf{Fix} & -- &  \multicolumn{2}{c}{\textbf{Density}} & \multicolumn{2}{c}{\textbf{Degree}} \\
&  \textbf{Size} & 400* & 600 & 800 & 600 & 800 \\
\midrule

{\multirow{12}{*}{Obj. gap [\%]}}
& Naive \citeyear{chen2024qp} & 2.547\scriptsize$\pm$0.126 & 11.480\scriptsize$\pm$0.297 & 20.663\scriptsize$\pm$0.813
& 1.787\scriptsize$\pm$0.196 & 1.402\scriptsize$\pm$0.127  \\
& Naive-G \citeyear{chen2024qp} & 2.281\scriptsize$\pm$0.150 & 8.461\scriptsize$\pm$1.078 & 15.443\scriptsize$\pm$1.500
& 1.744\scriptsize$\pm$0.134 & 1.302\scriptsize$\pm$0.091 \\
\cmidrule{2-7}
& IPM \citeyear{pmlr-v238-qian24a} & 1.894\scriptsize$\pm$0.166 & 10.521\scriptsize$\pm$1.057 & 20.936\scriptsize$\pm$1.106
& 1.510\scriptsize$\pm$0.126 & 1.153\scriptsize$\pm$0.063 \\
& IPM-G \citeyear{pmlr-v238-qian24a} & 1.829\scriptsize$\pm$0.079 & 10.987\scriptsize$\pm$1.132 & 20.129\scriptsize$\pm$1.782
& 1.541\scriptsize$\pm$0.021 & 1.062\scriptsize$\pm$0.132 \\
\cmidrule{2-7}
& IPM$_{16}$ (Ours) & 10.679\scriptsize$\pm$4.351 & 22.448\scriptsize$\pm$1.307 & 29.633\scriptsize$\pm$1.739
& 11.982\scriptsize$\pm$4.683 & 11.989\scriptsize$\pm$5.584 \\
& IPM$_{32}$ (Ours) & 2.236\scriptsize$\pm$0.249 & 7.685\scriptsize$\pm$2.917 & 19.310\scriptsize$\pm$4.806
& 2.003\scriptsize$\pm$0.472 & 1.856\scriptsize$\pm$0.370\\
& IPM-G$_{16}$ (Ours) & 10.715\scriptsize$\pm$5.344 & 12.170\scriptsize$\pm$5.604 & 14.648\scriptsize$\pm$8.498
& 10.692\scriptsize$\pm$5.573 & 9.694\scriptsize$\pm$6.355 \\
& IPM-G$_{32}$ (Ours) & 1.661\scriptsize$\pm$0.199 & 4.961\scriptsize$\pm$1.219 & 3.462\scriptsize$\pm$0.692
& 1.648\scriptsize$\pm$0.238 & 1.350\scriptsize$\pm$0.243 \\
\cmidrule{2-7}
& Feas.$_{16}$ (Ours) & 0.119\scriptsize$\pm$0.013 & 0.948\scriptsize$\pm$0.111 & 5.867\scriptsize$\pm$0.667
& 0.118\scriptsize$\pm$0.062 & 0.127\scriptsize$\pm$0.002 \\
& Feas.$_{32}$ (Ours) & 0.049\scriptsize$\pm$0.015 & 0.615\scriptsize$\pm$0.064 & 4.839\scriptsize$\pm$0.602
& 0.045\scriptsize$\pm$0.011 & 0.046\scriptsize$\pm$0.007 \\
& Feas.-G$_{16}$ (Ours) & 0.163\scriptsize$\pm$0.003 & 1.480\scriptsize$\pm$0.049 & 8.609\scriptsize$\pm$0.660
& 0.180\scriptsize$\pm$0.037 & 0.184\scriptsize$\pm$0.005 \\
& Feas.-G$_{32}$ (Ours) & 0.070\scriptsize$\pm$0.004 & 0.991\scriptsize$\pm$0.031 & 7.607\scriptsize$\pm$0.607
& 0.077\scriptsize$\pm$0.010 & 0.071\scriptsize$\pm$0.005 \\
\midrule

{\multirow{12}{*}{Cons. vio.}}
& Naive \citeyear{chen2024qp} & 0.012\scriptsize$\pm$0.004 & 0.053\scriptsize$\pm$0.002 & 0.099\scriptsize$\pm$0.002
& 0.014\scriptsize$\pm$0.002 & 0.014\scriptsize$\pm$0.002  \\
& Naive-G \citeyear{chen2024qp} & 0.018\scriptsize$\pm$0.001 & 0.050\scriptsize$\pm$0.004 & 0.095\scriptsize$\pm$0.006
& 0.017\scriptsize$\pm$0.001 & 0.017\scriptsize$\pm$0.001 \\
\cmidrule{2-7}
& IPM \citeyear{pmlr-v238-qian24a} & 0.013\scriptsize$\pm$0.001 & 0.048\scriptsize$\pm$0.002 & 0.092\scriptsize$\pm$0.003
& 0.012\scriptsize$\pm$0.001 & 0.011\scriptsize$\pm$0.001 \\
& IPM-G \citeyear{pmlr-v238-qian24a} & 0.016\scriptsize$\pm$0.0003 & 0.048\scriptsize$\pm$0.002 & 0.089\scriptsize$\pm$0.003
& 0.016\scriptsize$\pm$0.001 & 0.016\scriptsize$\pm$0.0003\\
\cmidrule{2-7}
& IPM$_{16}$ (Ours) & 0.028\scriptsize$\pm$0.008 & 0.040\scriptsize$\pm$0.006 & 0.064\scriptsize$\pm$0.006
& 0.028\scriptsize$\pm$0.008 & 0.027\scriptsize$\pm$0.008\\
& IPM$_{32}$ (Ours) & 0.031\scriptsize$\pm$0.008 & 0.050\scriptsize$\pm$0.006 & 0.078\scriptsize$\pm$0.007
& 0.030\scriptsize$\pm$0.009 & 0.030\scriptsize$\pm$0.008\\
& IPM-G$_{16}$ (Ours) & 0.026\scriptsize$\pm$0.007 & 0.048\scriptsize$\pm$0.007 & 0.069\scriptsize$\pm$0.004
& 0.026\scriptsize$\pm$0.008 & 0.026\scriptsize$\pm$0.007 \\
& IPM-G$_{32}$ (Ours) & 0.029\scriptsize$\pm$0.009 & 0.053\scriptsize$\pm$0.007 & 0.076\scriptsize$\pm$0.004
& 0.029\scriptsize$\pm$0.009 & 0.028\scriptsize$\pm$0.009 \\
\cmidrule{2-7}
& Feas.$_{16}$ (Ours) & 9.786  \(\times\) \(10^{-8}\) & 1.138  \(\times\) \(10^{-7}\) & 1.279  \(\times\) \(10^{-7}\)
& 1.093  \(\times\) \(10^{-7}\) & 1.178  \(\times\) \(10^{-7}\) \\
& Feas.$_{32}$ (Ours) & 1.215  \(\times\) \(10^{-7}\) & 1.376  \(\times\) \(10^{-7}\) & 1.516  \(\times\) \(10^{-7}\)
& 1.318  \(\times\) \(10^{-7}\) & 1.396  \(\times\) \(10^{-7}\) \\
& Feas.-G$_{16}$ (Ours) & 9.080  \(\times\) \(10^{-8}\) & 1.161  \(\times\) \(10^{-7}\) & 1.546  \(\times\) \(10^{-7}\)
& 1.021  \(\times\) \(10^{-7}\) & 1.148  \(\times\) \(10^{-7}\) \\
& Feas.-G$_{32}$ (Ours) & 1.142  \(\times\) \(10^{-7}\) & 1.447  \(\times\) \(10^{-7}\) & 1.769  \(\times\) \(10^{-7}\)
& 1.228  \(\times\) \(10^{-7}\) & 1.346  \(\times\) \(10^{-7}\) \\

\bottomrule

\end{tabular}
}
\end{table*}

\begin{table*}[t]
\caption{Computational efficiency on various problems trained with GCN architecture. The GCN was trained on datasets of smaller size and evaluated on larger sizes. The larger instances are generated by increasing the problem size and fixing the density parameters. We compare both our methods at 16 and 32 iterations against the naive MPNN approach \citet{chen2024qp}, the fix-step IPM-MPNN \citet{pmlr-v238-qian24a} and three traditional QP solvers. The methods with postfix $\text{-G}$ are with the global node. We report the neural networks based results with three models pretrained with various random seeds, with the mean and standard deviation reported.}
\label{tab:runtime}
\centering
\resizebox{1.\textwidth}{!}{
\begin{tabular}{c|ccc|ccc|ccc}
\toprule
 \textbf{Problem} & \multicolumn{3}{c}{ \textbf{Generic}} & \multicolumn{3}{c}{ \textbf{SVM}} & \multicolumn{3}{c}{ \textbf{Portfolio}} \\
 \textbf{Size} & 400* & 600 & 800 & 400* & 600 & 800 & 800* & 1000 & 1200  \\

\midrule
CVXOPT & 0.764\scriptsize$\pm$0.198 & 2.181\scriptsize$\pm$0.168 & 5.177\scriptsize$\pm$0.757 
& 0.495\scriptsize$\pm$0.050 & 1.205\scriptsize$\pm$0.331 & 2.503\scriptsize$\pm$0.606 
& 0.178\scriptsize$\pm$0.063 & 0.214\scriptsize$\pm$0.077 & 0.411\scriptsize$\pm$0.085 \\
OSQP & 0.485\scriptsize$\pm$0.165 & 1.470\scriptsize$\pm$0.109 & 3.865\scriptsize$\pm$0.259 
& 0.019\scriptsize$\pm$0.001 & 0.072\scriptsize$\pm$0.008 & 0.270\scriptsize$\pm$0.085 
& 0.040\scriptsize$\pm$0.006 & 0.065\scriptsize$\pm$0.013 & 0.121\scriptsize$\pm$0.042 \\
Gurobi & 0.813\scriptsize$\pm$0.088 & 2.691\scriptsize$\pm$0.158 & 6.682\scriptsize$\pm$0.579 
& 0.013\scriptsize$\pm$0.001 & 0.057\scriptsize$\pm$0.016 & 0.125\scriptsize$\pm$0.018 
& 0.221\scriptsize$\pm$0.032 & 0.303\scriptsize$\pm$0.022 & 0.520\scriptsize$\pm$0.063 \\
\midrule
Naive \citeyear{chen2024qp} & 0.010\scriptsize$\pm$0.002 & 0.026\scriptsize$\pm$0.008 & 0.018\scriptsize$\pm$0.005 
& 0.015\scriptsize$\pm$0.003 & 0.017\scriptsize$\pm$0.003 & 0.010\scriptsize$\pm$0.004 
& 0.008\scriptsize$\pm$0.002 & 0.009\scriptsize$\pm$0.003 & 0.010\scriptsize$\pm$0.005 \\
Naive-G \citeyear{chen2024qp} & 0.021\scriptsize$\pm$0.005 & 0.045\scriptsize$\pm$0.061 & 0.061\scriptsize$\pm$0.007 & 0.022\scriptsize$\pm$0.006 & 0.025\scriptsize$\pm$0.005 & 0.028\scriptsize$\pm$0.005 & 0.022\scriptsize$\pm$0.006 & 0.023\scriptsize$\pm$0.007 & 0.024\scriptsize$\pm$0.009\\
\midrule
IPM \citeyear{pmlr-v238-qian24a} & 0.014\scriptsize$\pm$0.002 & 0.031\scriptsize$\pm$0.002 & 0.048\scriptsize$\pm$0.001 & 0.010\scriptsize$\pm$0.001 & 0.010\scriptsize$\pm$0.001 & 0.013\scriptsize$\pm$0.002 & 0.012\scriptsize$\pm$0.003 & 0.018\scriptsize$\pm$0.002 & 0.017\scriptsize$\pm$0.001 \\
IPM-G \citeyear{pmlr-v238-qian24a} & 0.025\scriptsize$\pm$0.0004 & 0.036\scriptsize$\pm$0.002 & 0.057\scriptsize$\pm$0.003 
& 0.022\scriptsize$\pm$0.004 & 0.025\scriptsize$\pm$0.003 & 0.026\scriptsize$\pm$0.002 
& 0.022\scriptsize$\pm$0.002 & 0.022\scriptsize$\pm$0.002 & 0.019\scriptsize$\pm$0.002 \\
\midrule
IPM$_{16}$ (Ours) & 0.150\scriptsize$\pm$0.009 & 0.377\scriptsize$\pm$0.006 & 0.704\scriptsize$\pm$0.002 & 0.136\scriptsize$\pm$0.010 & 0.137\scriptsize$\pm$0.002 & 0.138\scriptsize$\pm$0.003 & 0.136\scriptsize$\pm$0.009 & 0.124\scriptsize$\pm$0.002 & 0.152\scriptsize$\pm$0.009\\
IPM$_{32}$ (Ours) & 0.298\scriptsize$\pm$0.010 & 0.751\scriptsize$\pm$0.003 & 1.410\scriptsize$\pm$0.005 & 0.256\scriptsize$\pm$0.014 & 0.247\scriptsize$\pm$0.011 & 0.238\scriptsize$\pm$0.007 & 0.261\scriptsize$\pm$0.017 & 0.254\scriptsize$\pm$0.009 & 0.289\scriptsize$\pm$0.007\\
IPM-G$_{16}$ (Ours) & 0.277\scriptsize$\pm$0.003 & 0.479\scriptsize$\pm$0.006 & 0.824\scriptsize$\pm$0.001 
& 0.256\scriptsize$\pm$0.008 & 0.238\scriptsize$\pm$0.002 & 0.240\scriptsize$\pm$0.005 
& 0.274\scriptsize$\pm$0.010 & 0.242\scriptsize$\pm$0.006 & 0.259\scriptsize$\pm$0.006 \\
IPM-G$_{32}$ (Ours) & 0.533\scriptsize$\pm$0.006 & 0.980\scriptsize$\pm$0.004 & 1.612\scriptsize$\pm$0.001 
& 0.497\scriptsize$\pm$0.008 & 0.459\scriptsize$\pm$0.009 & 0.488\scriptsize$\pm$0.007 
& 0.500\scriptsize$\pm$0.008 & 0.494\scriptsize$\pm$0.005 & 0.522\scriptsize$\pm$0.006 \\
\midrule
Prep. & 0.052\scriptsize$\pm$0.002 & 0.148\scriptsize$\pm$0.010 & 0.314\scriptsize$\pm$0.005 
& 0.034\scriptsize$\pm$0.001 & 0.091\scriptsize$\pm$0.002 & 0.184\scriptsize$\pm$0.009 
& 0.002\scriptsize$\pm$0.0004 & 0.004\scriptsize$\pm$0.001 & 0.005\scriptsize$\pm$0.001 \\
Feas.$_{16}$ (Ours) & 0.177\scriptsize$\pm$0.003 & 0.382\scriptsize$\pm$0.001 & 0.704\scriptsize$\pm$0.001 
& 0.161\scriptsize$\pm$0.001 & 0.166\scriptsize$\pm$0.002 & 0.140\scriptsize$\pm$0.001 
& 0.177\scriptsize$\pm$0.005 & 0.171\scriptsize$\pm$0.002 & 0.170\scriptsize$\pm$0.004 \\
Feas.$_{32}$ (Ours) & 0.294\scriptsize$\pm$0.001 & 0.756\scriptsize$\pm$0.004 & 1.414\scriptsize$\pm$0.002
& 0.295\scriptsize$\pm$0.005 & 0.283\scriptsize$\pm$0.006 & 0.264\scriptsize$\pm$0.002
& 0.305\scriptsize$\pm$0.020 & 0.299\scriptsize$\pm$0.011 & 0.310\scriptsize$\pm$0.010 \\
Feas.-G$_{16}$ (Ours) & 0.324\scriptsize$\pm$0.006 & 0.482\scriptsize$\pm$0.001 & 0.821\scriptsize$\pm$0.001 & 0.297\scriptsize$\pm$0.001 & 0.255\scriptsize$\pm$0.001 & 0.254\scriptsize$\pm$0.001 & 0.262\scriptsize$\pm$0.009 & 0.276\scriptsize$\pm$0.008 & 0.270\scriptsize$\pm$0.002 \\
Feas.-G$_{32}$ (Ours) & 0.602\scriptsize$\pm$0.025 & 0.955\scriptsize$\pm$0.003 & 1.652\scriptsize$\pm$0.014 & 0.482\scriptsize$\pm$0.002 & 0.508\scriptsize$\pm$0.005 & 0.481\scriptsize$\pm$0.013 & 0.470\scriptsize$\pm$0.015 & 0.517\scriptsize$\pm$0.018 & 0.543\scriptsize$\pm$0.015 \\
\bottomrule
\end{tabular}
}
\end{table*}

\begin{figure}
\centering
\scalebox{0.8}{\definecolor{lgray}{HTML}{D6D6D6}

\begin{tikzpicture}
    \pgfplotsset{
        width=0.6\textwidth,
        height=0.4\textwidth,
        every axis/.append style={
            xlabel={\footnotesize Size},
            ylabel={\footnotesize Time (sec.)},
            xlabel style={yshift=0.2cm},
            ylabel style={yshift=-0.1cm},
            xtick={400, 600, 800},
            tick label style={font=\footnotesize},
            legend style={font=\footnotesize, at={(1.05,1)}, draw=none, anchor=north west}
        }
    }

    \begin{axis}[ymode=log, ymin=0.008, ymax=7.2, title={Inference time performance comparison}]
        \addplot+[color=lviolet, mark=*, mark size=1pt, mark options={fill=lviolet}, error bars/.cd, y dir=both, y explicit, error bar style={draw=lviolet!30}, error mark options={draw=lviolet!30,rotate=90}] 
        coordinates {(400, 0.010) +- (0, 0.002) (600, 0.026) +- (0, 0.008) (800, 0.061) +- (0, 0.007)};
        \addlegendentry{Naive}
        
        \addplot+[color=lgreen, mark=*, mark size=1pt, mark options={fill=lgreen}, error bars/.cd, y dir=both, y explicit, error bar style={draw=lgreen!30}, error mark options={draw=lgreen!30,rotate=90}]  
        coordinates {(400, 0.177) +- (0, 0.003) (600, 0.382) +- (0, 0.001) (800, 0.704) +- (0, 0.001)};
        \addlegendentry{Feas.$_{16}$}
        
        \addplot+[color=lred, mark=*, mark size=1pt, mark options={fill=lred}, error bars/.cd, y dir=both, y explicit, error bar style={draw=lred!30}, error mark options={draw=lred!30,rotate=90}] 
        coordinates {(400, 0.294) +- (0, 0.001) (600, 0.756) +- (0, 0.004) (800, 1.414) +- (0, 0.002)};
        \addlegendentry{Feas.$_{32}$}
        
        \addplot+[color=lrose, mark=*, mark size=1pt, mark options={fill=lrose}, error bars/.cd, y dir=both, y explicit, error bar style={draw=lrose!30}, error mark options={draw=lrose!30,rotate=90}] 
        coordinates {(400, 0.052) +- (0, 0.002) (600, 0.148) +- (0, 0.010) (800, 0.314) +- (0, 0.005)};
        \addlegendentry{Prep.}
        
        \addplot+[color=lblue, mark=*, mark size=1pt, mark options={fill=lblue}, error bars/.cd, y dir=both, y explicit, error bar style={draw=lblue!30}, error mark options={draw=blue!30,rotate=90}] 
        coordinates {(400, 0.764) +- (0, 0.198) (600, 2.181) +- (0, 0.168) (800, 5.177) +- (0, 0.757)};
        \addlegendentry{CVXOPT}
        
        \addplot+[color=lorange, solid, mark=*, mark size=1pt, mark options={fill=lorange}, error bars/.cd, y dir=both, y explicit, error bar style={draw=lorange!30}, error mark options={draw=lorange!30,rotate=90}] 
        coordinates {(400, 0.485) +- (0, 0.165) (600, 1.470) +- (0, 0.109) (800, 3.865) +- (0, 0.259)};
        \addlegendentry{OSQP}

        \addplot+[color=lgray, solid, mark=*, mark size=1pt, mark options={fill=lgray}, error bars/.cd, y dir=both, y explicit, error bar style={draw=lgray!30}, error mark options={draw=lgray!30,rotate=90}] 
        coordinates {(400, 0.813) +- (0, 0.088) (600, 2.691) +- (0, 0.158) (800, 6.682) +- (0, 0.579)};
        \addlegendentry{Gurobi}
    \end{axis}
    
\end{tikzpicture}}
\caption{Comparison of inference runtimes on generic problem datasets, featuring GCN without the global node and three alternative solvers.}
\label{Plot}
\end{figure}
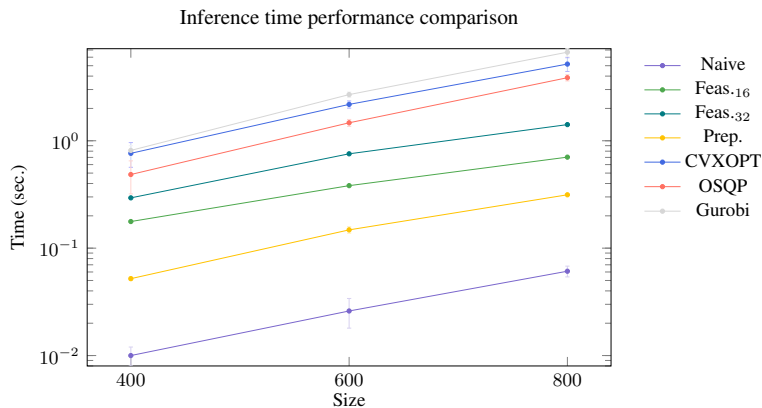

\subsection{Log-barrier function in search}
\label{sec:log_barrier_bias}
In practical implementations, there are situations where the current solution $\vec{x}^{(t)}$ is near the boundary of the positive orthant $\mathbb{Q}^n_{\geq 0}$ and the prediction of displacement is inaccurate. In such cases, the step length would be small to not violate the non-negative constraint. At the next iteration, due to the continuous nature of neural networks, the prediction will again be inaccurate since the current solution hardly moved, and a small step size will be picked. Therefore, the current solution will likely be stuck at some suboptimal point. Recall the log barrier function in IPM \citep{nocedal2006numerical}, the function 
\begin{equation}
\label{eq:barrier_func}
f \colon \mathbb{Q}^n \rightarrow \mathbb{Q}, f(\vec{x}^{(t)}) \coloneq - \bm{1}\trans \log(\vec{x}^{(t)}),
\end{equation}
where $\log$ is an element-wise operation, is added in the objective function to prevent Newton's step from being too aggressive and violating the non-negative constraint. Here, we incorporate the same function to encourage the current solution to move away from the orthant boundary. Calculating the gradient of this log barrier function w.r.t.\  $\vec{x}^{(t)}$, we have the direction vector $\nabla_{\vec{x}^{(t)}} f \coloneq - 1 / {\vec{x}^{(t)}}$. We directly subtract this vector from the predicted displacement vector $\vec{d}^{(t)}$ before applying null-space projection to them. Intuitively, this pushes the entries of the current solution that are very close to zero to larger positive numbers. To minimize the negative effect of the log barrier function on the convergence, we introduce a discount coefficient $\tau$ that scales itself down at each iteration. When $\tau \rightarrow 0$, the algorithm can still converge. However, the log barrier function still has a limitation. That is, it is not guaranteed that the entries in $1 / {\vec{x}^{(t)}}$ after the projection $\boldsymbol{\Pi}_{\vec{A}} \dfrac{1}{\vec{x}^{(t)}}$ are still positive, which may drive some entries of $\vec{x}^{(t)}$ even closer to 0. In response to this challenge, we show that this log barrier force after the null space projection still effectively pushes the current solution away from the positive orthant boundary. Formally, based on the \cref{eq:barrier_func}, we want to show that
\begin{equation}
\label{eq:decrease_logbarrier}
f \left(\vec{x} + \alpha \boldsymbol{\Pi}_{\vec{A}} \frac{1}{\vec{x}}\right) \leq f\left(\vec{x} \right),
\end{equation}
for sufficiently small step length $\alpha$.
Since  $\alpha$ is a small number, we treat the term $\alpha \boldsymbol{\Pi}_{\vec{A}} \frac{1}{\vec{x}}$ as a small perturbation and perform a Taylor expansion,
\begin{equation*}
\begin{aligned}
f \left(\vec{x} + \alpha \boldsymbol{\Pi}_{\vec{A}} \frac{1}{\vec{x}} \right) & \approx f(\vec{x}) + \nabla f(\vec{x})\trans \alpha \boldsymbol{\Pi}_{\vec{A}} \frac{1}{\vec{x}} \\
 & = f(\vec{x}) - \alpha \frac{1}{\vec{x}}\trans \boldsymbol{\Pi}_{\vec{A}} \frac{1}{\vec{x}}.
\end{aligned}
\end{equation*}
We notice the second term is an inner product, hence non-negative. So for sufficiently small $\alpha$ \cref{eq:decrease_logbarrier} holds. 
Now, we would like to find an upper bound for $\alpha$. We perform the second order Taylor expansion,
\begin{equation*}
\begin{aligned}
f \left(\vec{x} + \alpha \boldsymbol{\Pi}_{\vec{A}} \frac{1}{\vec{x}} \right) & \approx f(\vec{x}) + \nabla f(\vec{x})\trans \alpha \boldsymbol{\Pi}_{\vec{A}} \frac{1}{\vec{x}} + \left(\alpha \boldsymbol{\Pi}_{\vec{A}} \frac{1}{\vec{x}}\right)\trans \nabla^2 f(\vec{x}) \left(\alpha \boldsymbol{\Pi}_{\vec{A}} \frac{1}{\vec{x}}\right)\\
& = f(\vec{x}) - \alpha \frac{1}{\vec{x}}\trans \boldsymbol{\Pi}_{\vec{A}} \frac{1}{\vec{x}} + \left(\alpha \boldsymbol{\Pi}_{\vec{A}} \frac{1}{\vec{x}}\right)\trans \text{diag}\left(\frac{1}{\vec{x}^2}\right) \left(\alpha \boldsymbol{\Pi}_{\vec{A}} \frac{1}{\vec{x}}\right) \\
& < f(\vec{x}).
\end{aligned}
\end{equation*}
 We want to ensure that 
\begin{equation*}
\alpha \frac{1}{\vec{x}}\trans \boldsymbol{\Pi}_{\vec{A}} \frac{1}{\vec{x}} - \left(\alpha \boldsymbol{\Pi}_{\vec{A}} \frac{1}{\vec{x}}\right)\trans \text{diag}\left(\frac{1}{\vec{x}^2}\right) \left(\alpha \boldsymbol{\Pi}_{\vec{A}} \frac{1}{\vec{x}}\right) > 0.
\end{equation*}
 We obtain the upper bound
\begin{equation*}
\alpha < \dfrac{\frac{1}{\vec{x}}\trans \boldsymbol{\Pi}_{\vec{A}} \frac{1}{\vec{x}}}{\left(\boldsymbol{\Pi}_{\vec{A}} \frac{1}{\vec{x}}\right)\trans \text{diag}\left(\frac{1}{\vec{x}^2}\right) \left(\boldsymbol{\Pi}_{\vec{A}} \frac{1}{\vec{x}}\right)}.
\end{equation*}

\subsection{Derivation of the IPM} 
\label{sec:ipm_derive}
In this section, we only consider QPs as LPs are special cases of QPs where $\vec{Q}$ is set to an all-zero matrix. Let us first recap the standard form of QPs with linear equality constraints, 
\begin{equation*}
\begin{aligned}
\min_{{\vec{x} \in \mathbb{R}^n_{\geq 0}}} \quad & \frac{1}{2} \vec{x}\trans \vec{Q} \vec{x} + \vec{c}\trans \vec{x} \\
\text{s.t.} \quad & \vec{A} \vec{x} = \vec{b}.
\end{aligned}
\end{equation*}
By adding Lagrangian multipliers, we obtain the Lagrangian, 
\begin{equation*}
\mathcal{L}(\vec{x}, \boldsymbol{\lambda}, \vec{s}) \coloneq \frac{1}{2} \vec{x}\trans \vec{Q} \vec{x} + \vec{c}\trans \vec{x} - \boldsymbol{\lambda} \trans \left( \vec{A} \vec{x} - \vec{b} \right) - \vec{s}\trans \vec{x}, 
\end{equation*}
with $\vec{x}, \vec{s} \in \mathbb{R}^n_{\geq 0}, \boldsymbol{\lambda} \in \mathbb{R}^m$. We can derive the \new{Karush–Kuhn–Tucker} (KKT) conditions for the Lagrangian,
\begin{equation}
\label{eq:kkt}
\begin{aligned}
\vec{A} \vec{x} &= \vec{b} \\
\vec{Q} \vec{x} + \vec{c} - \vec{A}\trans \boldsymbol{\lambda} - \vec{s} &= \bm{0} \\
\vec{x} & \geq \bm{0} \\
\vec{s} & \geq \bm{0} \\
\vec{x}_i \vec{s}_i &= \bm{0}.
\end{aligned}
\end{equation}
According to Lagrangian duality theory~\citep{nocedal2006numerical}, the KKT condition is the necessary condition for optimality, and it is also a sufficient condition in our case of QPs. Thus, our goal is to find the solution $\left(\vec{x}, \boldsymbol{\lambda}, \vec{s}\right)$ that satisfies the KKT condition above. Let us consider the function
\begin{equation*}
F\left(\vec{x}, \boldsymbol{\lambda}, \vec{s}\right) \coloneq \begin{bmatrix}
\vec{A} \vec{x} - \vec{b} \\
\vec{Q} \vec{x} + \vec{c} - \vec{A}\trans \boldsymbol{\lambda} - \vec{s} \\
\vec{X} \vec{S} \bm{1}
\end{bmatrix},
\end{equation*}
where $\vec{X}$ and $\vec{S}$ are diagonal matrices of the vectors $\vec{x}$ and $\vec{s}$, respectively. The solution to the KKT condition is equivalent to the solution of $F\left(\vec{x}, \boldsymbol{\lambda}, \vec{s}\right) = \bm{0}$, with $\vec{x}, \vec{s} \geq \bm{0}$. We search the zero point of the function $F$ iteratively with Newton's method. That is, given current point of $\left(\vec{x}, \boldsymbol{\lambda}, \vec{s}\right)$, we aim to find the search direction $\left( \Delta \vec{x}, \Delta \boldsymbol{\lambda}, \Delta \vec{s} \right)$ by solving 
\begin{equation*}
J(F)  \left( \Delta \vec{x}, \Delta \boldsymbol{\lambda}, \Delta \vec{s} \right) = -F\left(\vec{x}, \boldsymbol{\lambda}, \vec{s}\right),
\end{equation*}
where $J(\cdot)$ denotes the Jacobian of a function. The matrix form is 
\begin{equation*}
\begin{bmatrix}
\vec{A} & \bm{0} & \bm{0} \\
\vec{Q} & -\vec{A}\trans & -\vec{I} \\
\vec{S} & \bm{0} & \vec{X}
\end{bmatrix} 
\begin{bmatrix}
\Delta \vec{x} \\
\Delta \boldsymbol{\lambda} \\
\Delta \vec{s}
\end{bmatrix} = 
\begin{bmatrix}
-\vec{A} \vec{x} + \vec{b} \\
-\vec{Q} \vec{x} - \vec{c} + \vec{A}\trans \boldsymbol{\lambda} + \vec{s} \\
- \vec{X} \vec{S} \bm{1}
\end{bmatrix}.
\end{equation*}
Ideally, by solving this equation, we can obtain the search direction. By doing a line search along the direction, the positivity constraints can be satisfied, and we can eventually converge to the optimal solution. However, the direction from Newton's method might not ensure the positivity constraints, and the corresponding search step size might be too small. Thus similar to the IPM for LPs ~\citep{nocedal2006numerical}, we utilize the log barrier function and add the term $-\mu \log(\vec{x})\trans \bm{1}$ with $\mu \geq 0$ to the objective to replace the positivity constraints $\vec{x} \geq \bm{0}$. To be specific, the original QP problem becomes 
\begin{equation*}
\begin{aligned}
\min_{\vec{x}} \quad & \frac{1}{2} \vec{x}\trans \vec{Q} \vec{x} + \vec{c}\trans \vec{x} -\mu \log(\vec{x})\trans \bm{1} \\
\text{s.t.} \quad & \vec{A} \vec{x} = \vec{b}.
\end{aligned}
\end{equation*}
Correspondingly, we have a new Lagrangian 
\begin{equation*}
\mathcal{L}(\vec{x}, \boldsymbol{\lambda}, \vec{s}) \coloneq \frac{1}{2} \vec{x}\trans \vec{Q} \vec{x} + \vec{c}\trans \vec{x} - \boldsymbol{\lambda} \trans \left( \vec{A} \vec{x} - \vec{b} \right) -\mu \log(\vec{x})\trans \bm{1}.
\end{equation*}
The second equation in \cref{eq:kkt} becomes $\vec{Q} \vec{x} + \vec{c} - \vec{A}\trans \boldsymbol{\lambda} - \dfrac{\mu}{\vec{x}} = \bm{0}$, which is nonlinear, making optimization via Newton's method hard. Hence, we introduce the term $\vec{s} \coloneq \dfrac{\mu}{\vec{x}}$, and the new equation coincides with the second equation in \cref{eq:kkt}. However, the last equation in \cref{eq:kkt} becomes $\vec{x}_i \vec{s}_i = \mu$. According to the path following definition \citep{nocedal2006numerical}, when $\mu \rightarrow 0$, the QP with log barrier function recovers the original problem. In practice, $\mu$ is defined as the dot product of the current solution $\vec{x}\trans \vec{s}/n$, multiplied with a scaling factor $\sigma$. The corresponding Newton equation is
\begin{equation*}
\begin{bmatrix}
\vec{A} & \bm{0} & \bm{0} \\
\vec{Q} & -\vec{A}\trans & -\vec{I} \\
\vec{S} & \bm{0} & \vec{X}
\end{bmatrix} 
\begin{bmatrix}
\Delta \vec{x} \\
\Delta \boldsymbol{\lambda} \\
\Delta \vec{s}
\end{bmatrix} = 
\begin{bmatrix}
-\vec{A} \vec{x} + \vec{b} \\
-\vec{Q} \vec{x} - \vec{c} + \vec{A}\trans \boldsymbol{\lambda} + \vec{s} \\
- \vec{X} \vec{S} \bm{1} + \sigma \mu \bm{1}
\end{bmatrix}.
\end{equation*}
Now, let us take a closer look at the Newton equations. The last row gives us 
\begin{equation*}
\label{eq:eliminate_s}
\Delta \vec{s} \coloneqq - \vec{X}^{-1} \vec{S} \Delta \vec{x}  - \vec{s} + \vec{X}^{-1} \sigma \mu \bm{1},
\end{equation*}
applying this to the second row
\begin{equation*}
\vec{Q}\Delta \vec{x} - \vec{A}\trans \Delta \boldsymbol{\lambda} -\Delta \vec{s} =  -\vec{Q} \vec{x} - \vec{c} + \vec{A}\trans \boldsymbol{\lambda} + \vec{s},
\end{equation*}
we have 
\begin{equation*}
\left(\vec{Q} + \vec{X}^{-1} \vec{S} \right) \Delta \vec{x} - \vec{A}\trans \Delta \boldsymbol{\lambda} = \vec{X}^{-1} \sigma \mu \bm{1} -\vec{Q} \vec{x} - \vec{c} + \vec{A}\trans \boldsymbol{\lambda}.
\end{equation*}
Combining with the first row, we can see that we have reached an \say{atomic} linear system, 
\begin{equation}
\label{eq:augmented_eq}
\begin{bmatrix}
\vec{Q} + \vec{X}^{-1} \vec{S} & - \vec{A}\trans \\
- \vec{A} & \bm{0}
\end{bmatrix}
\begin{bmatrix}
\Delta \vec{x} \\
\Delta \boldsymbol{\lambda}
\end{bmatrix} = 
\begin{bmatrix}
\vec{X}^{-1} \sigma \mu \bm{1} -\vec{Q} \vec{x} - \vec{c} + \vec{A}\trans \boldsymbol{\lambda}\\
\vec{A} \vec{x} - \vec{b}
\end{bmatrix},
\end{equation}
where it is hard to perform Gaussian elimination due to the nontrivial computation of $\left(\vec{Q} + \vec{X}^{-1} \vec{S}\right)^{-1}$. In contrast, in LPs, the $\vec{Q}$ term in the matrix $\begin{bmatrix}
\vec{Q} + \vec{X}^{-1} \vec{S} & - \vec{A}\trans \\
- \vec{A} & \bm{0}
\end{bmatrix}$ vanishes, and we can further eliminate variable $\Delta \vec{x}$. However, the matrix $\begin{bmatrix}
\vec{Q} + \vec{X}^{-1} \vec{S} & - \vec{A}\trans \\
- \vec{A} & \bm{0}
\end{bmatrix}$ has full rank by definition. If we initialize the vectors $\vec{x}, \vec{s}$ with positive values, we know the matrix is positive definite from Schur complement theory \citep{gallier2010schur}. Hence, we use the conjugate gradient (CG) method \citep{nocedal2006numerical} for the joint solution of $\left(\Delta \vec{x}, \Delta \boldsymbol{\lambda}\right)$ by solving \cref{eq:augmented_eq}. For the standard CG algorithm, see \cref{alg:cg-ref}. For a parametrized CG algorithm specific to our problem \cref{eq:augmented_eq}, see \cref{alg:cg-ipm}.

\begin{algorithm}[t]
\caption{Standard conjugate gradient algorithm for reference.}
\label{alg:cg-ref}
\begin{algorithmic}[1]
\REQUIRE An equation $\vec{A} \vec{x} = \vec{b}$, where $\vec{A} \in \mathbb{R}^{n \times n}$ is a real-valued, symmetric, positive-definite matrix. 
\STATE $\vec{r}^{(0)} \leftarrow \vec{b}$
\STATE $\vec{x}^{(0)} \leftarrow \bm{0}$
\STATE $\vec{p}^{(0)} \leftarrow \vec{r}^{(0)}$
\FOR{$t \in [n]$}
\STATE $\alpha^{(t)} \leftarrow \dfrac{{\vec{r}^{(t-1)}}\trans \vec{r}^{(t-1)}}{{\vec{r}^{(t-1)}}\trans \vec{A} \vec{r}^{(t-1)}}$
\STATE $\vec{x}^{(t)} \leftarrow \vec{x}^{(t-1)} + \alpha^{(t)} \vec{p}^{(t-1)}$
\STATE $\vec{r}^{(t)} \leftarrow \vec{r}^{(t-1)} - \alpha^{(t)} \vec{A} \vec{p}^{(t-1)}$
\STATE $\beta^{(t)} \leftarrow \dfrac{{\vec{r}^{(t)}}\trans \vec{r}^{(t)}}{{\vec{r}^{(t-1)}}\trans \vec{r}^{(t-1)}}$
\STATE $\vec{p}^{(t)} \leftarrow \beta^{(t)} \vec{p}^{(t-1)} + \vec{r}^{(t)}$
\ENDFOR
\RETURN Solution $\vec{x}^{(n)}$.
\end{algorithmic}
\end{algorithm}

\begin{algorithm}[t]
\caption{Conjugate gradient algorithm for the linear equations \cref{eq:augmented_eq}.}
\label{alg:cg-ipm}
\begin{algorithmic}[1]
\REQUIRE A QP instance $\left( \vec{Q}, \vec{A}, \vec{b}, \vec{c} \right)$, initial solution $\vec{x}, \vec{s} > \bm{0}$, $\boldsymbol{\lambda} \in \mathbb{R}^m$, hyperparameter $\sigma, \mu$. 
\STATE $\vec{X} \leftarrow \text{diag}(\vec{x})$
\STATE $\vec{S} \leftarrow \text{diag}(\vec{s})$
\STATE $\vec{P} \leftarrow \begin{bmatrix}
\vec{Q} + \vec{X}^{-1} \vec{S} & - \vec{A}\trans \\
- \vec{A} & \bm{0}
\end{bmatrix}$
\STATE $\vec{r}^{(0)} \leftarrow \begin{bmatrix}
\vec{X}^{-1} \sigma \mu \bm{1} -\vec{Q} \vec{x} - \vec{c} + \vec{A}\trans \boldsymbol{\lambda}\\
\vec{A} \vec{x} - \vec{b}
\end{bmatrix}$
\STATE $\vec{w}^{(0)} \leftarrow \bm{0}$
\STATE $\vec{p}^{(0)} \leftarrow \vec{r}^{(0)}$
\FOR{$t \in [n+m]$}
\STATE $\alpha^{(t)} \leftarrow \dfrac{{\vec{r}^{(t-1)}}\trans \vec{r}^{(t-1)}}{{\vec{r}^{(t-1)}}\trans \vec{P} \vec{r}^{(t-1)}}$
\STATE $\vec{w}^{(t)} \leftarrow \vec{w}^{(t-1)} + \alpha^{(t)} \vec{p}^{(t-1)}$
\STATE $\vec{r}^{(t)} \leftarrow \vec{r}^{(t-1)} - \alpha^{(t)} \vec{P} \vec{p}^{(t-1)}$
\STATE $\beta^{(t)} \leftarrow \dfrac{{\vec{r}^{(t)}}\trans \vec{r}^{(t)}}{{\vec{r}^{(t-1)}}\trans \vec{r}^{(t-1)}}$
\STATE $\vec{p}^{(t)} \leftarrow \beta^{(t)} \vec{p}^{(t-1)} + \vec{r}^{(t)}$
\ENDFOR
\STATE $\Delta \vec{x} \leftarrow \vec{w}^{(n+m)}_{[1:n]}$
\STATE $\Delta \boldsymbol{\lambda} \leftarrow \vec{w}^{(n+m)}_{[n+1:n+m]}$
\RETURN Solution $\Delta \vec{x}, \Delta \boldsymbol{\lambda}$.
\end{algorithmic}
\end{algorithm}

\subsection{Omitted proofs}
\label{sec:proof_appendix}

Here, we outline the missing proofs from the main paper. 

\subsubsection{MPNNs can simulate IPMs}

Here we prove \cref{thm:mpnn_can_ipm} from the main paper in detail.

\begin{lemma}\label{thm:cg-is-MPNN}
There exists an MPNN $f_{\textsf{MPNN,IPM}}$ composed of $\cO(1)$ layers and $\cO(m+n)$ successive message-passing steps that reproduces each iteration of the IPM algorithm for LCQPs, in the sense that for any LCQP instance $I = \left(\vec{Q}, \vec{A},\vec{b},\vec{c} \right)$ and any primal-dual point $\left(\vec{x}^{(t)}, \boldsymbol{\lambda}^{(t)}, \vec{s}^{(t)}\right)$ with $t > 0$, $f_{\textsf{MPNN,IPM}}$ maps the graph $G(I)$ carrying $[\vec{x}^{(t-1)}, \vec{s}^{(t-1)}]$ on the variable nodes, $[\boldsymbol{\lambda}^{(t-1)}]$ on the constraint nodes, and $[\mu, \sigma]$ on the global node to the same graph $G(I)$ carrying the output $[\vec{x}^{(t)}, \vec{s}^{(t)}]$ and $[\boldsymbol{\lambda}^{(t)}]$ of \cref{alg:ipm-practice} on the variable and constraint nodes, respectively.
\end{lemma}

\begin{proof} For a QP instance $I$, we construct an undirected, heterogeneous graph $G(I)$ with three node types, variable nodes $V(I)$, constraint nodes $C(I)$, and a single global node $\{g(I)\}$, as already described in \cref{sec:graph_construct}. The corresponding node embedding matrices are denoted as $\vec{V}, \vec{C}, \vec{G}$ respectively. Across different node types, we have four edge types. The intra-variable-nodes connections are given by the non-zero entries of the $\vec{Q}$ matrix, and we use the values $\vec{Q}_{ij}$ as the edge attributes. For the edges connecting variable and constraint nodes, we use the non-zero entries of the $\vec{A}$ matrix and also the values $\vec{A}_{ij}$ as the edge attributes. We design the global node to be connected to all variable and constraint nodes with uniform edge attributes. By default, we assume all the vectors to be \emph{column} vectors, but we refer to a row $i$ of a matrix $\vec{M}$ in the form of $\vec{M}_i$ and assume it to be a \emph{row} vector. Specifically, we construct the graph with node feature matrices initialized as follows,
\begin{equation*}
\begin{aligned}
\vec{V}^{(\text{init},0)} & \coloneqq \begin{bmatrix} \vec{x} & \vec{s} & \vec{c} \end{bmatrix} \in \mathbb{R}^{n \times 3}, \\
\vec{C}^{(\text{init},0)} & \coloneqq \begin{bmatrix} \boldsymbol{\lambda} & \vec{b} \end{bmatrix} \in \mathbb{R}^{m \times 2}, \\
\vec{G}^{(\text{init},0)} & \coloneqq \begin{bmatrix}\sigma \mu \end{bmatrix} \in \mathbb{R}.
\end{aligned}
\end{equation*}

Note that we use superscripts of the form $(\text{phase},i)$ where phase stands for the initialization or some iteration $t$ in \cref{alg:cg-ipm}, and $i$ is an index for intermediate operations that appears when necessary. We set edge features
\begin{equation*}
\begin{aligned}
\vec{e}_{v \rightarrow u} & \coloneqq \vec{Q}_{vu}, \forall v, u \in V(I) \\
\vec{e}_{c \rightarrow v} & \coloneqq \vec{A}_{cv}, \forall c \in C(I), v \in V(I) \\
\vec{e}_{v \rightarrow g} & \coloneqq 1, \forall v \in V(I) \\
\vec{e}_{c \rightarrow g} & \coloneqq 1, \forall c \in C(I). \\
\end{aligned}
\end{equation*}
From now on, we drop the notations of $\vec{e}$ and use $\vec{Q}, \vec{A}$ for ease of notation. 

In principle, we would like all the message-passing steps to fall in the framework of the following synchronous, heterogeneous message-passing functions, 
\begin{equation}
\begin{aligned}
\label{eq:proof_mpnn_framework}
    \vec{V}_v^{(t)} \coloneqq  \textsf{UPD}^{(t)}_{\text{v}} \Bigl[
     &\vec{V}_v^{(t-1)}, \\
    & \textsf{AGG}^{(t)}_{\text{c} \rightarrow \text{v}}\left(\{\!\!\{ \textsf{MSG}^{(t)}_{\text{c} \rightarrow \text{v}}\left(\vec{C}_c^{(t-1)}, \vec{A}_{cv} \right) \mid c \in {N}\left(v \right) \cap C(I) \}\!\!\}  \right), \\
    & \textsf{AGG}^{(t)}_{\text{v} \rightarrow \text{v}}\left(\{\!\!\{ \textsf{MSG}^{(t)}_{\text{v} \rightarrow \text{v}}\left(\vec{V}_u^{(t-1)}, \vec{Q}_{vu} \right) \mid u \in {N}\left(v \right) \cap V(I) \}\!\!\}  \right), \\
    & \textsf{MSG}^{(t)}_{\text{g} \rightarrow \text{v}}\left( \vec{G}^{(t-1)} \right) \Bigr], \forall v \in V(I). \\
    \vec{C}_c^{(t)} \coloneqq  \textsf{UPD}^{(t)}_{\text{c}} \Bigl[
    & \vec{C}_c^{(t-1)}, \\
    & \textsf{AGG}^{(t)}_{\text{v} \rightarrow \text{c}}\left(\{\!\!\{ \textsf{MSG}^{(t)}_{\text{v} \rightarrow \text{c}}\left(\vec{V}_v^{(t-1)}, \vec{A}_{cv}\right) \mid v \in {N}\left(c \right) \cap V(I) \}\!\!\}  \right), \\
    & \textsf{MSG}^{(t)}_{\text{g} \rightarrow \text{c}}\left( \vec{G}^{(t-1)} \right) \Bigr], \forall c \in C(I). \\
    \vec{G}^{(t)} \coloneqq  \textsf{UPD}^{(t)}_{\text{g}} \Bigl[
    & \vec{G}^{(t-1)}, \\
    & \textsf{AGG}^{(t)}_{\text{v} \rightarrow \text{g}}\left(\{\!\!\{ \textsf{MSG}^{(t)}_{\text{v} \rightarrow \text{g}}\left(\vec{V}_v^{(t-1)} \right) \mid v \in V(I) \}\!\!\}  \right), \\
    & \textsf{AGG}^{(t)}_{\text{c} \rightarrow \text{g}}\left(\{\!\!\{ \textsf{MSG}^{(t)}_{\text{c} \rightarrow \text{g}}\left(\vec{C}_c^{(t-1)} \right)\mid c \in C(I) \}\!\!\}  \right) \Bigr],
\end{aligned}
\end{equation}

where $\textsf{UPD}_{\text{v}}^{(t)}$ is the update function shared by all the variable nodes $v \in V(I)$, similar with $\textsf{UPD}_{\text{c}}^{(t)}$ and $\textsf{UPD}_{\text{g}}^{(t)}$. Functions $\textsf{AGG}^{(t)}_{\circ \rightarrow \odot}$ and $\textsf{MSG}^{(t)}_{\circ \rightarrow \odot}$ represent the aggregation and message functions from a node type $\circ$ to another $\odot$, also shared on the edges with the same edge type.
We assume all the $\textsf{UPD}, \textsf{AGG}, \textsf{MSG}$ are arbitrary continuous functions. 

In the following, we show how a message-passing architecture can simulate the CG algorithm~\cref{alg:cg-ipm}.

\paragraph{Lines 4-6} Let us first take a look at the initialization
\begin{equation*}
\vec{r}^{(0)} \leftarrow \begin{bmatrix}
\vec{X}^{-1} \sigma \mu \bm{1} -\vec{Q} \vec{x} - \vec{c} + \vec{A}\trans \boldsymbol{\lambda}\\
\vec{A} \vec{x} - \vec{b}
\end{bmatrix}.
\end{equation*}
That is, we need the vector $\vec{X}^{-1} \sigma \mu \bm{1} -\vec{Q} \vec{x} - \vec{c} + \vec{A}\trans \boldsymbol{\lambda}$ to be stored in the variable nodes, and $\vec{A} \vec{x} - \vec{b}$ in the constraint nodes. We notice that $\vec{Q} \vec{x}$ is a simple graph convolution on the intra-variable nodes, and $\vec{A}\trans \boldsymbol{\lambda}$ is convolution from constraint nodes to variable nodes. Therefore, we can compute $\vec{X}^{-1} \sigma \mu \bm{1} -\vec{Q} \vec{x} - \vec{c} + \vec{A}\trans \boldsymbol{\lambda}$ with one message-passing step by constructing the message functions for every variable node $v \in V(I)$,
\begin{equation}
\begin{aligned}
\label{eq:proof_cg_init_msg}
\textsf{MSG}^{(\text{init},1)}_{\text{g} \rightarrow \text{v}} \left( \vec{G}^{(\text{init},0)} \right) &\coloneq \vec{G}^{(\text{init},0)} = \sigma \mu, \\
\textsf{MSG}^{(\text{init},1)}_{\text{c} \rightarrow \text{v}} \left( \vec{C}^{(\text{init},0)}_c, \vec{A}_{cv} \right) &\coloneq \vec{A}_{cv} \vec{C}^{(\text{init},0)}_c \begin{bmatrix} 1 \\ 0 \end{bmatrix} = \vec{A}_{cv} \boldsymbol{\lambda}_c, \forall c \in N(v) \cap C(I), \\
\textsf{MSG}^{(\text{init},1)}_{\text{v} \rightarrow \text{v}} \left( \vec{V}^{(\text{init},0)}_u, \vec{Q}_{vu} \right) &\coloneq \vec{Q}_{vu} \vec{V}^{(\text{init},0)}_u \begin{bmatrix} 1 \\ 0 \\ 0 \end{bmatrix}  = \vec{Q}_{vu} \vec{x}_u, \forall u \in N(v) \cap V(I),
\end{aligned}
\end{equation}
and the aggregation functions,
\begin{equation}
\begin{aligned}
\label{eq:proof_cg_init_agg}
\textsf{AGG}^{(\text{init},1)}_{\text{c} \rightarrow \text{v}} \left( \{\!\!\{\vec{A}_{cv} \boldsymbol{\lambda}_c \mid c \in N(v) \cap C(I) \}\!\!\} \right) &\coloneq \sum_{c \in N(v) \cap C(I)} \vec{A}_{cv} \boldsymbol{\lambda}_c = \vec{A}\trans_v \boldsymbol{\lambda}, \\
\textsf{AGG}^{(\text{init},1)}_{\text{v} \rightarrow \text{v}} \left( \{\!\!\{\vec{Q}_{vu} \vec{x}_u \mid u \in N(v) \cap V(I) \}\!\!\} \right) &\coloneq \sum_{u \in N(v) \cap V(I)} \vec{Q}_{vu} \vec{x}_u = \vec{Q}_{v} \vec{x},
\end{aligned}
\end{equation}
and finally, we design the update function,
\begin{equation}
\begin{aligned}
\label{eq:proof_cg_init_upd}
\vec{V}_v^{(\text{init},1)} & \coloneqq \textsf{UPD}^{(\text{init},1)}_{\text{v}} \Bigl[ \vec{V}_v^{(\text{init},0)}, \vec{A}\trans_v \boldsymbol{\lambda}, \vec{Q}_{v} \vec{x}, \sigma \mu \Bigr] \\
& = \dfrac{\sigma \mu}{\vec{V}_v^{(\text{init},0)} \begin{bmatrix} 1 \\ 0 \\ 0 \end{bmatrix}} - \vec{Q}_{v} \vec{x} + \vec{A}\trans_v \boldsymbol{\lambda} - \vec{V}_v^{(\text{init},0)} \begin{bmatrix} 0 \\ 0 \\ 1 \end{bmatrix} \\
& = \dfrac{\sigma \mu}{\vec{x}_v} - \vec{Q}_{v} \vec{x} + \vec{A}\trans_v \boldsymbol{\lambda} - \vec{c}_v.
\end{aligned}
\end{equation}
The above functions form the framework in \cref{eq:proof_mpnn_framework}. When we stack all the values $\vec{V}_v^{(\text{init},1)}$ into a vector, we end up with the needed vector $\vec{X}^{-1} \sigma \mu \bm{1} -\vec{Q} \vec{x} - \vec{c} + \vec{A}\trans \boldsymbol{\lambda}$. Note that $\vec{X}^{-1} \bm{1} = \dfrac{\bm{1}}{\vec{x}}$ can be obtained by element-wise reciprocal operation as a non-linear activation function.

The functions \cref{eq:proof_cg_init_msg,eq:proof_cg_init_agg,eq:proof_cg_init_upd} give us a first impression of how our proof by construction works. For convenience, we store the previous vectors $\vec{x}, \vec{s}$ in the node embedding matrices for the upcoming operations. Besides the initialization for $\vec{r}^{(0)}$ above, we need to copy $\vec{r}^{(0)}$ into $\vec{p}^{(0)}$, and assign some trivial values to $\vec{w}^{(0)}$. Hence, we need to modify the update functions,
\begin{equation*}
\begin{aligned}
\vec{V}_v^{(\text{init},1)} & \coloneqq \textsf{UPD}^{(\text{init},1)}_{\text{v}} \Bigl[ \vec{V}_v^{(\text{init},0)}, \vec{A}\trans_v \boldsymbol{\lambda}, \vec{Q}_{v} \vec{x}, \sigma \mu \Bigr] \\
& = \left( \dfrac{\sigma \mu}{\vec{V}_v^{(\text{init},0)} \begin{bmatrix} 1 \\ 0 \\ 0 \end{bmatrix}} - \vec{Q}_{v} \vec{x} + \vec{A}\trans_v \boldsymbol{\lambda} - \vec{V}_v^{(\text{init},0)} \begin{bmatrix} 0 \\ 0 \\ 1 \end{bmatrix} \right) \begin{bmatrix} 1 & 1 & 0 & 0 & 0 \end{bmatrix} \\
& + \vec{V}_v^{(\text{init},0)} \begin{bmatrix} 0 & 0 & 0 & 1 & 0  \\
0 & 0 & 0 & 0 & 1  \\
0 & 0 & 0 & 0 & 0  \\
\end{bmatrix} \\
& = \begin{bmatrix} \dfrac{\sigma \mu}{\vec{x}_v} - \vec{Q}_{v} \vec{x} + \vec{A}\trans_v \boldsymbol{\lambda} - \vec{c}_v & \dfrac{\sigma \mu}{\vec{x}_v} - \vec{Q}_{v} \vec{x} + \vec{A}\trans_v \boldsymbol{\lambda} - \vec{c}_v & 0 & \vec{x}_v & \vec{s}_v \end{bmatrix} \in \mathbb{R}^{5},
\end{aligned}
\end{equation*}
giving us a row vector for node $v$.

Similarly, we need the result $\vec{A} \vec{x} - \vec{b}$ for the initialization of the constraint nodes. Notice that $\vec{A}\vec{x}$ is a simple graph convolution from the variable to the constraint nodes. We construct the message functions for each constraint node $c \in C(I)$,
\begin{equation*}
\begin{aligned}
\textsf{MSG}^{(\text{init},1)}_{\text{v} \rightarrow \text{c}} \left( \vec{V}^{(\text{init},0)}_v, \vec{A}_{cv} \right) &\coloneq \vec{A}_{cv} \vec{V}^{(\text{init},0)}_v \begin{bmatrix} 1 \\ 0 \\ 0 \end{bmatrix} = \vec{A}_{cv} \vec{x}_v, \forall v \in N(c) \cap V(I), \\
\end{aligned}
\end{equation*}
and the aggregation function,
\begin{equation*}
\textsf{AGG}^{(\text{init},1)}_{\text{v} \rightarrow \text{c}} \left( \{\!\!\{\vec{A}_{cv} \vec{x}_v \mid v \in N(c) \cap V(I) \}\!\!\} \right) \coloneq \sum_{v \in N(c) \cap V(I)} \vec{A}_{cv} \vec{x}_v = \vec{A}_{c} \vec{x},
\end{equation*}
and the update function,
\begin{equation*}
\begin{aligned}
\vec{C}_c^{(\text{init},1)} & \coloneqq \textsf{UPD}^{(\text{init},1)}_{\text{c}} \Bigl[ \vec{C}_c^{(\text{init},0)}, \vec{A}_{c} \vec{x} \Bigr] \\
& = \vec{A}_{c} \vec{x} - \vec{C}_c^{(\text{init},0)} \begin{bmatrix} 0 \\ 1\end{bmatrix} \\
& = \vec{A}_{c} \vec{x} - \vec{b}_c.
\end{aligned}
\end{equation*}
We get $\vec{A} \vec{x} - \vec{b}$ by stacking up all the nodes $c \in C(I)$. Still, we would like to store the vector $\boldsymbol{\lambda}$ in the updated node embeddings, and initialize $\vec{p}^{(0)}$ and $\vec{w}^{(0)}$. Hence, we modify the update function as 
\begin{equation*}
\begin{aligned}
\vec{C}_c^{(\text{init},1)} & \coloneqq \textsf{UPD}^{(\text{init},1)}_{\text{c}} \Bigl[ \vec{C}_c^{(\text{init},0)}, \vec{A}_{c} \vec{x} \Bigr] \\
& = \left( \vec{A}_{c} \vec{x} - \vec{C}_c^{(\text{init},0)} \begin{bmatrix} 0 \\ 1\end{bmatrix} \right) \begin{bmatrix} 1 & 1 & 0 & 0 \end{bmatrix}  \\
& + \vec{C}_c^{(\text{init},0)} \begin{bmatrix} 0 & 0 & 0 & 1 \\ 0 & 0 & 0 & 0\end{bmatrix} \\
& = \begin{bmatrix} \vec{A}_{c} \vec{x} - \vec{b}_c &  \vec{A}_{c} \vec{x} - \vec{b}_c & 0 & \boldsymbol{\lambda}_c \end{bmatrix} \in \mathbb{R}^{4}.
\end{aligned}
\end{equation*}

Thus, we accomplish the the initialization of $\vec{r}^{(0)}, \vec{p}^{(0)}, \vec{w}^{(0)}$  with one message-passing step. We notice that we can always split a joint vector, e.g., $\vec{r} \in \mathbb{R}^{n+m}$ (superscript omitted), into $\vec{r} = \begin{bmatrix}
\vec{r}_1 \in \mathbb{R}^{n}, \\
\vec{r}_2 \in \mathbb{R}^{m},
\end{bmatrix}$
with $\vec{r}_1$ and $\vec{r}_2$ stored in variable and constraint nodes. Therefore, we will always refer to their two distributed parts as $\vec{r}_1,\vec{r}_2$ from now on, same with the vectors $\vec{p},\vec{w}$.

After the lines 4-6 in \cref{alg:cg-ipm}, we now have the node embedding matrices
\begin{equation*}
\begin{aligned}
\vec{V}^{(\text{init},1)} & \coloneq \begin{bmatrix} \vec{r}^{(0)}_1 & \vec{p}^{(0)}_1 & \vec{w}^{(0)}_1 & \vec{x} & \vec{s} \end{bmatrix} \in \mathbb{R}^{n \times 5} \\
\vec{C}^{(\text{init},1)} & \coloneq \begin{bmatrix} \vec{r}^{(0)}_2 &  \vec{p}^{(0)}_2 & \vec{w}^{(0)}_2 & \boldsymbol{\lambda} \end{bmatrix} \in \mathbb{R}^{m \times 4},
\end{aligned}
\end{equation*}
\emph{which we now redefine as $\vec{V}^{(0)}$ and $\vec{C}^{(0)}$}, and we do not care about $\vec{G}^{(0)}$ at the moment. Now, we are ready for the iterations in \cref{alg:cg-ipm}.

\paragraph{Line 8}
Let us look at each iteration $t$ in \cref{alg:cg-ipm}. With the information of $\vec{r}_1^{(t-1)}, \vec{r}_2^{(t-1)}$ already in the nodes, the numerator of $\alpha^{(t)}$, ${\vec{r}^{(t-1)}}\trans\vec{r}^{(t-1)} = {\vec{r}_1^{(t-1)}}\trans\vec{r}_1^{(t-1)} + {\vec{r}_2^{(t-1)}}\trans\vec{r}_2^{(t-1)}$ can be computed with one convolution from variable and constraint nodes to the global node. The right part of the denominator, $\vec{P} \vec{r}^{(t-1)} = 
\begin{bmatrix}
\vec{Q} + \vec{X}^{-1} \vec{S} & -\vec{A}\trans \\
- \vec{A} & \bm{0} 
\end{bmatrix}
\begin{bmatrix}
\vec{r}_1^{(t-1)} \\
\vec{r}_2^{(t-1)}
\end{bmatrix} = 
\begin{bmatrix}
\vec{Q} \vec{r}_1^{(t-1)} + \dfrac{\vec{s} \vec{r}_1^{(t-1)}}{\vec{x}} - \vec{A}\trans \vec{r}_2^{(t-1)} \\
- \vec{A} \vec{r}_1^{(t-1)}
\end{bmatrix}$, can be done with one step of message passing between variable and constraint nodes.

Following \cref{eq:proof_mpnn_framework}, we can construct message functions from constraint nodes to variable nodes 
\begin{equation*}
\begin{aligned}
\textsf{MSG}^{(t,1)}_{\text{v} \rightarrow \text{c}} \left( \vec{V}^{(t-1)}_v, \vec{A}_{cv} \right) &\coloneq \vec{A}_{cv} \vec{V}^{(t-1)}_v \begin{bmatrix} 1 \\ 0 \\ 0 \\ 0 \\0 \end{bmatrix} = \vec{A}_{cv} \left(\vec{r}^{(t-1)}_1\right)_v, fv \in N(c) \cap V(I),
\end{aligned}
\end{equation*}
and aggregation function
\begin{equation*}
\begin{aligned}
\textsf{AGG}^{(t,1)}_{\text{v} \rightarrow \text{c}} \left( \{\!\!\{\vec{A}_{cv} \left(\vec{r}^{(t-1)}_1\right)_v \mid v \in N(c) \cap V(I) \}\!\!\} \right) \coloneq \sum_{v \in N(c) \cap V(I)} \vec{A}_{cv} \left(\vec{r}^{(t-1)}_1\right)_v = \vec{A}_{c} \vec{r}^{(t-1)}_1,
\end{aligned}
\end{equation*}
and the update function
\begin{equation}
\begin{aligned}
\label{eq:per_cons_node_line8_1}
\vec{C}_c^{(t,1)} & \coloneqq \textsf{UPD}^{(t,1)}_{\text{c}} \Bigl[ \vec{C}_c^{(t-1)}, \vec{A}_{c} \vec{r}^{(t-1)}_1 \Bigr] \\
& =  \vec{A}_{c} \vec{r}^{(t-1)}_1 \begin{bmatrix} -1 \\ 0 \\ 0 \\ 0 \\ 0 \end{bmatrix} + \vec{C}_c^{(t-1)} \begin{bmatrix} 0 & 1 & 0 & 0 & 0 \\ 0 & 0 & 1 & 0 & 0 \\ 0 & 0 & 0 & 1 & 0 \\ 0 & 0 & 0 & 0 & 1 \\ \end{bmatrix} \\
& = \begin{bmatrix} -\vec{A}_{c} \vec{r}^{(t-1)}_1 & \left(\vec{r}^{(t-1)}_2\right)_c & \left(\vec{p}^{(t-1)}_2\right)_c & \left(\vec{w}^{(t-1)}_2\right)_c & \boldsymbol{\lambda}_c \end{bmatrix}.
\end{aligned}
\end{equation}

For the other direction, i.e., from constraint nodes to variable nodes, we have message functions,
\begin{equation*}
\begin{aligned}
\textsf{MSG}^{(t,1)}_{\text{c} \rightarrow \text{v}} \left( \vec{C}^{(t-1)}_c, \vec{A}_{cv} \right) &\coloneq \vec{A}_{cv} \vec{C}^{(t-1)}_c \begin{bmatrix} 1 \\ 0 \\0 \\ 0 \end{bmatrix} = \vec{A}_{cv} \left(\vec{r}^{(t-1)}_2\right)_c, \forall c \in N(v) \cap C(I) \\
\textsf{MSG}^{(t,1)}_{\text{v} \rightarrow \text{v}} \left( \vec{V}^{(t-1)}_u, \vec{Q}_{vu} \right) &\coloneq \vec{Q}_{vu} \vec{V}^{(t-1)}_u \begin{bmatrix} 1 \\ 0 \\ 0 \\0\\0 \end{bmatrix}  = \vec{Q}_{vu} \left(\vec{r}^{(t-1)}_1\right)_u, \forall u \in N(v) \cap V(I),
\end{aligned}
\end{equation*}
and aggregation function,
\begin{equation*}
\begin{aligned}
\textsf{AGG}^{(t,1)}_{\text{c} \rightarrow \text{v}} \left( \{\!\!\{\vec{A}_{cv} \left(\vec{r}^{(t-1)}_2\right)_c \mid c \in N(v) \cap C(I) \}\!\!\} \right) &\coloneq \sum_{c \in N(v) \cap C(I)} \vec{A}_{cv} \left(\vec{r}^{(t-1)}_2\right)_c = \vec{A}\trans_v \vec{r}^{(t-1)}_2 \\
\textsf{AGG}^{(t,1)}_{\text{v} \rightarrow \text{v}} \left( \{\!\!\{\vec{Q}_{vu} \left(\vec{r}^{(t-1)}_1\right)_u \mid u \in N(v) \cap V(I) \}\!\!\} \right) &\coloneq \sum_{u \in N(v) \cap V(I)} \vec{Q}_{vu} \left(\vec{r}^{(t-1)}_1\right)_u = \vec{Q}_{v} \vec{r}^{(t-1)}_1,
\end{aligned}
\end{equation*}
and the update function,
\begin{equation}
\begin{aligned}
\label{eq:per_var_node_line8_1}
\vec{V}_v^{(t,1)} & \coloneqq \textsf{UPD}^{(t,1)}_{\text{v}} \Bigl[ \vec{V}_v^{(t-1)}, \vec{A}\trans_v \vec{r}^{(t-1)}_2, \vec{Q}_{v} \vec{r}^{(t-1)}_1 \Bigr] \\
& = \vec{V}_v^{(t-1)} \begin{bmatrix} 0 & 1 & 0 & 0 & 0 & 0 \\ 0 & 0 & 1 & 0 & 0 & 0 \\ 0 & 0 & 0 & 1 & 0 & 0 \\ 0 & 0 & 0 & 0 & 1 & 0 \\ 0 & 0 & 0 & 0 & 0 & 1 \\ \end{bmatrix} \\
& + \left(\vec{Q}_{v} \vec{r}^{(t-1)}_1 - \vec{A}\trans_v \vec{r}^{(t-1)}_2 + \dfrac{ \vec{V}_v^{(t-1)} \begin{bmatrix} 0 & 0 & 0 & 0 & 0 \\ 0 & 0 & 0 & 0 & 0 \\ 0 & 0 & 0 & 0 & 0 \\ 0 & 0 & 0 & 0 & 0 \\ 1 & 0 & 0 & 0 & 0 \end{bmatrix} \left(\vec{V}_v^{(t-1)}\right)\trans}{\vec{V}_v^{(t-1)} \begin{bmatrix} 0 \\ 0 \\ 0 \\ 1 \\ 0 \end{bmatrix}} \right) \begin{bmatrix} 1 & 0 & 0 & 0 & 0 & 0 \end{bmatrix} \\
& = \begin{bmatrix} \vec{Q}_{v} \vec{r}^{(t-1)}_1 - \vec{A}\trans_v \vec{r}^{(t-1)}_2 + \dfrac{\vec{s}_v \left(\vec{r}_1^{(t-1)} \right)_v}{\vec{x}_v} & \left(\vec{r}^{(t-1)}_1\right)_v & \left(\vec{p}^{(t-1)}_1\right)_v & \left(\vec{w}^{(t-1)}_1\right)_v & \vec{x}_v  & \vec{s}_v \end{bmatrix}.
\end{aligned}
\end{equation}

If we stack up the per-node embeddings in \cref{eq:per_cons_node_line8_1} and \cref{eq:per_var_node_line8_1}, we will have the matrix representation,
\begin{equation*}
\begin{aligned}
\vec{V}^{(t,1)} & \coloneqq \begin{bmatrix} \vec{Q} \vec{r}^{(t-1)}_1 - \vec{A}\trans \vec{r}^{(t-1)}_2 + \dfrac{\vec{s} \vec{r}_1^{(t-1)}}{\vec{x}} & \vec{r}^{(t-1)}_1 & \vec{p}^{(t-1)}_1 & \vec{w}^{(t-1)}_1 & \vec{x}  & \vec{s} \end{bmatrix}\\
& = \begin{bmatrix} \left(\vec{P} \vec{r}^{(t-1)} \right)_1 &  \vec{r}_1^{(t-1)} & \vec{p}^{(t-1)}_1 & \vec{w}^{(t-1)}_1 & \vec{x} & \vec{s} \end{bmatrix}\\
\vec{C}^{(t,1)} & \coloneqq \begin{bmatrix} -\vec{A} \vec{r}^{(t-1)}_1 & \vec{r}^{(t-1)}_2 & \vec{p}^{(t-1)}_2 & \vec{w}^{(t-1)}_2 & \boldsymbol{\lambda} \end{bmatrix} \\
& = \begin{bmatrix} \left(\vec{P} \vec{r}^{(t-1)} \right)_2 & \vec{r}_2^{(t-1)} &  \vec{p}^{(t-1)}_2 & \vec{w}^{(t-1)}_2 & \boldsymbol{\lambda} \end{bmatrix}.
\end{aligned}
\end{equation*}

Further, ${\vec{r}^{(t-1)}}\trans \vec{P} \vec{r}^{(t-1)} = \begin{bmatrix} \vec{r}_1^{(t-1)} & \vec{r}_2^{(t-1)} \end{bmatrix}
\begin{bmatrix}
\left(\vec{P} \vec{r}^{(t-1)} \right)_1 \\
\left(\vec{P} \vec{r}^{(t-1)} \right)_2
\end{bmatrix}$ can be computed with one additional message passing from variable and constraint nodes to the global node. Together with the numerator ${\vec{r}^{(t-1)}}\trans {\vec{r}^{(t-1)}}$, we can now construct the functions for the global node. Hence, the message function are set to,
\begin{equation*}
\begin{aligned}
\textsf{MSG}^{(t,2)}_{\text{c} \rightarrow \text{g}}\left(\vec{C}_c^{(t,1)} \right) &\coloneq \vec{C}_c^{(t,1)} \begin{bmatrix} 0 & 1 & 0 & 0 & 0  \\ 0 & 0 & 0 & 0 & 0  \\ 0 & 0 & 0 & 0 & 0  \\ 0 & 0 & 0 & 0 & 0  \\ 0 & 0 & 0 & 0 & 0  \\ \end{bmatrix} \left(\vec{C}_c^{(t,1)}\right)\trans \begin{bmatrix} 1 & 0 \end{bmatrix} \\
& + \vec{C}_c^{(t,1)} \begin{bmatrix} 0 & 0 & 0 & 0 & 0  \\ 0 & 1 & 0 & 0 & 0  \\ 0 & 0 & 0 & 0 & 0  \\ 0 & 0 & 0 & 0 & 0  \\ 0 & 0 & 0 & 0 & 0  \\ \end{bmatrix} \left(\vec{C}_c^{(t,1)}\right)\trans \begin{bmatrix} 0 & 1 \end{bmatrix} \\
& =  \begin{bmatrix}\left( \left(\vec{P} \vec{r}^{(t-1)} \right)_2 \right)_c \left(\vec{r}_2^{(t-1)} \right)_c & \left(\vec{r}_2^{(t-1)} \right)_c^2 \end{bmatrix}, \forall c \in C(I), \\
\textsf{MSG}^{(t,2)}_{\text{v} \rightarrow \text{g}}\left(\vec{V}_v^{(t,1)} \right) &\coloneq \vec{V}_v^{(t,1)} \begin{bmatrix} 0 & 1 & 0 & 0 & 0 & 0  \\ 0 & 0 & 0 & 0 & 0 & 0 \\ 0 & 0 & 0 & 0 & 0 & 0 \\ 0 & 0 & 0 & 0 & 0 & 0 \\ 0 & 0 & 0 & 0 & 0 & 0 \\ 0 & 0 & 0 & 0 & 0 & 0 \\ \end{bmatrix} \left(\vec{V}_v^{(t,1)}\right)\trans \begin{bmatrix} 1 & 0 \end{bmatrix} \\
& + \vec{V}_v^{(t,1)} \begin{bmatrix} 0 & 0 & 0 & 0 & 0 & 0  \\ 0 & 1 & 0 & 0 & 0 & 0 \\ 0 & 0 & 0 & 0 & 0 & 0 \\ 0 & 0 & 0 & 0 & 0 & 0 \\ 0 & 0 & 0 & 0 & 0 & 0 \\ 0 & 0 & 0 & 0 & 0 & 0 \\ \end{bmatrix} \left(\vec{V}_v^{(t,1)}\right)\trans \begin{bmatrix} 0 & 1 \end{bmatrix} \\
& =  \begin{bmatrix}\left( \left(\vec{P} \vec{r}^{(t-1)} \right)_1 \right)_v \left(\vec{r}_1^{(t-1)} \right)_v & \left(\vec{r}_1^{(t-1)} \right)_v^2 \end{bmatrix}, \forall v \in V(I),
\end{aligned}
\end{equation*}
and the  aggregation function for the global node is,
\begin{equation*}
\begin{aligned}
\textsf{AGG}^{(t,2)}_{\text{v} \rightarrow \text{g}} &\left(\{\!\!\{ \begin{bmatrix}\left( \left(\vec{P} \vec{r}^{(t-1)} \right)_1 \right)_v \left(\vec{r}_1^{(t-1)} \right)_v & \left(\vec{r}_1^{(t-1)} \right)_v^2 \end{bmatrix} \mid v \in V(I) \}\!\!\}  \right) \\
& \coloneq \sum_{v \in V(I)} \begin{bmatrix}\left( \left(\vec{P} \vec{r}^{(t-1)} \right)_1 \right)_v \left(\vec{r}_1^{(t-1)} \right)_v & \left(\vec{r}_1^{(t-1)} \right)_v^2 \end{bmatrix} \\
& = \begin{bmatrix} \left(\vec{r}_1^{(t-1)} \right)\trans \left(\vec{P} \vec{r}^{(t-1)} \right)_1 &  \left(\vec{r}_1^{(t-1)} \right)\trans \vec{r}_1^{(t-1)}\end{bmatrix}, \\
\textsf{AGG}^{(t,2)}_{\text{c} \rightarrow \text{g}} &\left(\{\!\!\{ \begin{bmatrix}\left( \left(\vec{P} \vec{r}^{(t-1)} \right)_2 \right)_c \left(\vec{r}_2^{(t-1)} \right)_c & \left(\vec{r}_2^{(t-1)} \right)_c^2 \end{bmatrix} \mid c \in C(I) \}\!\!\}  \right) \\
& \coloneq \sum_{c \in C(I)} \begin{bmatrix}\left( \left(\vec{P} \vec{r}^{(t-1)} \right)_2 \right)_c \left(\vec{r}_2^{(t-1)} \right)_c & \left(\vec{r}_2^{(t-1)} \right)_c^2 \end{bmatrix} \\
& = \begin{bmatrix} \left(\vec{r}_2^{(t-1)} \right)\trans \left(\vec{P} \vec{r}^{(t-1)} \right)_2 &  \left(\vec{r}_2^{(t-1)} \right)\trans \vec{r}_2^{(t-1)}\end{bmatrix}. \\
\end{aligned}
\end{equation*}
The update function for the global node, without using the previous state of $\vec{G}^{(t-1)}$,
\begin{equation*}
\begin{aligned}
&\vec{G}^{(t,2)} \coloneqq  \textsf{UPD}^{(t,2)}_{\text{g}} \Bigl[
 \begin{bmatrix} \left(\vec{r}_1^{(t-1)} \right)\trans \left(\vec{P} \vec{r}^{(t-1)} \right)_1 &  \left(\vec{r}_1^{(t-1)} \right)\trans \vec{r}_1^{(t-1)}\end{bmatrix}, \\
& \begin{bmatrix} \left(\vec{r}_2^{(t-1)} \right)\trans \left(\vec{P} \vec{r}^{(t-1)} \right)_2 &  \left(\vec{r}_2^{(t-1)} \right)\trans \vec{r}_2^{(t-1)}\end{bmatrix} \Bigr] \\
& = \dfrac{ \left( \begin{bmatrix} \left(\vec{r}_1^{(t-1)} \right)\trans \left(\vec{P} \vec{r}^{(t-1)} \right)_1 &  \left(\vec{r}_1^{(t-1)} \right)\trans \vec{r}_1^{(t-1)}\end{bmatrix} + \begin{bmatrix} \left(\vec{r}_2^{(t-1)} \right)\trans \left(\vec{P} \vec{r}^{(t-1)} \right)_2 &  \left(\vec{r}_2^{(t-1)} \right)\trans \vec{r}_2^{(t-1)}\end{bmatrix}\right) \begin{bmatrix} 0 \\ 1 \end{bmatrix}}{\left( \begin{bmatrix} \left(\vec{r}_1^{(t-1)} \right)\trans \left(\vec{P} \vec{r}^{(t-1)} \right)_1 &  \left(\vec{r}_1^{(t-1)} \right)\trans \vec{r}_1^{(t-1)}\end{bmatrix} + \begin{bmatrix} \left(\vec{r}_2^{(t-1)} \right)\trans \left(\vec{P} \vec{r}^{(t-1)} \right)_2 &  \left(\vec{r}_2^{(t-1)} \right)\trans \vec{r}_2^{(t-1)}\end{bmatrix} \right) \begin{bmatrix} 1 \\ 0 \end{bmatrix}} \\
& = \dfrac{\left(\vec{r}_1^{(t-1)} \right)\trans \vec{r}_1^{(t-1)} + \left(\vec{r}_2^{(t-1)} \right)\trans \vec{r}_2^{(t-1)}}{\left(\vec{r}_1^{(t-1)} \right)\trans \left(\vec{P} \vec{r}^{(t-1)} \right)_1 + \left(\vec{r}_2^{(t-1)} \right)\trans \left(\vec{P} \vec{r}^{(t-1)} \right)_2} = \dfrac{\left(\vec{r}^{(t-1)} \right)\trans \vec{r}^{(t-1)}}{\left(\vec{r}^{(t-1)} \right)\trans \vec{P} \vec{r}^{(t-1)}} = \alpha^{(t)}.
\end{aligned}
\end{equation*}
Thus, the calculation of $\alpha^{(t)}$ in line 8 in \cref{alg:cg-ipm} consists of two steps of message passing, with the result stored in the global node as $\vec{G}^{(t,2)}$ after the second message passing step. As for $\vec{V}^{(t,2)}$ and $\vec{C}^{(t,2)}$, we can discard the intermediate matrix-vector product $\vec{P} \vec{r}^{(t-1)}$ by designing a simple update function,
\begin{equation*}
\begin{aligned}
\vec{V}_v^{(t,2)} & \coloneqq \textsf{UPD}^{(t,2)}_{\text{v}} \Bigl[ \vec{V}_v^{(t,1)} \Bigr] \\
& = \vec{V}_v^{(t,1)} \begin{bmatrix} 0 & 0 & 0 & 0 & 0 \\ 1 & 0 & 0 & 0 & 0 \\ 0 & 1 & 0 & 0 & 0 \\ 0 & 0 & 1 & 0 & 0 \\ 0 & 0 & 0 & 1 & 0 \\ 0 & 0 & 0 & 0 & 1 \\ \end{bmatrix} \\
& = \begin{bmatrix} \left(\vec{r}_1^{(t-1)}\right)_v & \left(\vec{p}^{(t-1)}_1\right)_v & \left(\vec{w}^{(t-1)}_1\right)_v & \vec{x}_v & \vec{s}_v \end{bmatrix}, \\
\vec{C}_c^{(t,2)} & \coloneqq \textsf{UPD}^{(t,2)}_{\text{c}} \Bigl[ \vec{C}_c^{(t,1)} \Bigr] \\
& = \vec{C}_c^{(t,1)} \begin{bmatrix} 0 & 0 & 0 & 0 \\ 1 & 0 & 0 & 0 \\ 0 & 1 & 0 & 0 \\ 0 & 0 & 1 & 0 \\ 0 & 0 & 0 & 1\\ \end{bmatrix} \\
& = \begin{bmatrix} \left(\vec{r}_2^{(t-1)}\right)_c & \left(\vec{p}^{(t-1)}_2\right)_c & \left(\vec{w}^{(t-1)}_2\right)_c & \boldsymbol{\lambda}_c \end{bmatrix}.
\end{aligned}
\end{equation*}

To quickly recap, after line 8, we have 
\begin{equation*}
\begin{aligned}
\vec{V}^{(t,2)} & \coloneqq \begin{bmatrix} \vec{r}_1^{(t-1)} & \vec{p}^{(t-1)}_1 & \vec{w}^{(t-1)}_1 & \vec{x} & \vec{s} \end{bmatrix}\\
\vec{C}^{(t,2)} & \coloneqq \begin{bmatrix} \vec{r}_2^{(t-1)} &  \vec{p}^{(t-1)}_2 & \vec{w}^{(t-1)}_2 & \boldsymbol{\lambda} \end{bmatrix} \\
\vec{G}^{(t,2)} &\coloneq \alpha^{(t)}.
\end{aligned}
\end{equation*}

\paragraph{Line 9}
The update of $\vec{w}^{(t-1)}$ in line 9 is quite simple by updating the variable and constraint node embeddings with the scalar $\alpha^{(t)}$ in the global node, so one single step of message passing. We can define the message from the global node as,
\begin{equation*}
\begin{aligned}
\textsf{MSG}^{(t,3)}_{\text{g} \rightarrow \text{v}} \left( \vec{G}^{(t,2)} \right) &\coloneq \vec{G}^{(t,2)} = \alpha^{(t)} \\
\textsf{MSG}^{(t,3)}_{\text{g} \rightarrow \text{c}} \left( \vec{G}^{(t,2)} \right) &\coloneq \vec{G}^{(t,2)} = \alpha^{(t)}.
\end{aligned}
\end{equation*}
Then, the variable and constraint nodes update their embeddings locally,
\begin{equation*}
\begin{aligned}
\vec{V}_v^{(t,3)} & \coloneqq \textsf{UPD}^{(t,3)}_{\text{v}} \Bigl[ \vec{V}_v^{(t,2)}, \alpha^{(t)} \Bigr] \\
& = \vec{V}_v^{(t,2)} + \alpha^{(t)} \vec{V}_v^{(t,2)} \begin{bmatrix}  0 & 0 & 0 & 0 & 0 \\ 0 & 0 & 1 & 0 & 0 \\ 0 & 0 & 0 & 0 & 0 \\ 0 & 0 & 0 & 0 & 0 \\ 0 & 0 & 0 & 0 & 0 \\ \end{bmatrix} \\
& = \begin{bmatrix} \left(\vec{r}_1^{(t-1)}\right)_v & \left(\vec{p}^{(t-1)}_1\right)_v & \left(\vec{w}^{(t-1)}_1\right)_v & \vec{x}_v & \vec{s}_v \end{bmatrix} \\
& + \alpha^{(t)} \begin{bmatrix} 0 & 0 & \left(\vec{p}^{(t-1)}_1\right)_v & 0 & 0 \end{bmatrix} \\
& = \begin{bmatrix} \left(\vec{r}_1^{(t-1)}\right)_v & \left(\vec{p}^{(t-1)}_1\right)_v & \left(\vec{w}^{(t)}_1\right)_v & \vec{x}_v & \vec{s}_v \end{bmatrix}. \\
\vec{C}_c^{(t,3)} & \coloneqq \textsf{UPD}^{(t,3)}_{\text{c}} \Bigl[ \vec{C}_c^{(t,2)}, \alpha^{(t)} \Bigr] \\
& = \vec{C}_c^{(t,2)} + \alpha^{(t)} \vec{C}_c^{(t,2)} \begin{bmatrix}  0 & 0 & 0 & 0 \\ 0 & 0 & 1 & 0  \\ 0 & 0 & 0 & 0 \\ 0 & 0 & 0 & 0 \\ \end{bmatrix} \\
& = \begin{bmatrix} \left(\vec{r}_2^{(t-1)}\right)_c & \left(\vec{p}^{(t-1)}_2\right)_c & \left(\vec{w}^{(t-1)}_2\right)_c & \boldsymbol{\lambda}_c \end{bmatrix} + \alpha^{(t)} \begin{bmatrix} 0 & 0 & \left(\vec{p}^{(t-1)}_2\right)_c & 0 \end{bmatrix} \\
& = \begin{bmatrix} \left(\vec{r}_2^{(t-1)}\right)_c & \left(\vec{p}^{(t-1)}_2\right)_c & \left(\vec{w}^{(t)}_2\right)_c & \boldsymbol{\lambda}_c \end{bmatrix}.
\end{aligned}
\end{equation*}

Thereafter, we have the updated node embeddings,
\begin{equation*}
\begin{aligned}
\vec{V}^{(t,3)} & \coloneqq  \begin{bmatrix}  \vec{r}_1^{(t-1)} & \vec{p}^{(t-1)}_1 & \vec{w}^{(t)}_1 & \vec{x} & \vec{s} \end{bmatrix}\\
\vec{C}^{(t,3)} & \coloneqq \begin{bmatrix}  \vec{r}_2^{(t-1)} &  \vec{p}^{(t-1)}_2 & \vec{w}^{(t)}_2 & \boldsymbol{\lambda} \end{bmatrix}. \\
\end{aligned}
\end{equation*}

\paragraph{Line 10}
In line 10, the matrix-vector product $\vec{P} \vec{p}^{(t-1)} = 
\begin{bmatrix}
\vec{Q} + \vec{X}^{-1} \vec{S} & -\vec{A}\trans \\
- \vec{A} & \bm{0} 
\end{bmatrix}
\begin{bmatrix}
\vec{p}_1^{(t-1)} \\
\vec{p}_2^{(t-1)}
\end{bmatrix} = 
\begin{bmatrix}
\vec{Q} \vec{p}_1^{(t-1)} + \dfrac{\vec{s} \vec{p}_1^{(t-1)}}{\vec{x}} - \vec{A}\trans \vec{p}_2^{(t-1)} \\
- \vec{A} \vec{p}_1^{(t-1)}
\end{bmatrix} = \begin{bmatrix} \left(\vec{P} \vec{p}^{(t-1)} \right)_1  \\ \left(\vec{P} \vec{p}^{(t-1)} \right)_2 \end{bmatrix}$ can be done with one step of message passing, similar to the discussion of $\vec{P} \vec{r}^{(t-1)}$ above. 

Explicitly, the message functions from constraint nodes to variable nodes,
\begin{equation*}
\begin{aligned}
\textsf{MSG}^{(t,4)}_{\text{v} \rightarrow \text{c}} \left( \vec{V}^{(t,3)}_v, \vec{A}_{cv} \right) &\coloneq \vec{A}_{cv} \vec{V}^{(t,3)}_v \begin{bmatrix} 0 \\ 1 \\ 0 \\ 0 \\ 0 \end{bmatrix} = \vec{A}_{cv} \left(\vec{p}^{(t-1)}_1\right)_v, v \in N(c) \cap V(I),
\end{aligned}
\end{equation*}
and the aggregation function
\begin{equation*}
\begin{aligned}
\textsf{AGG}^{(t,4)}_{\text{v} \rightarrow \text{c}} \left( \{\!\!\{\vec{A}_{cv} \left(\vec{p}^{(t-1)}_1\right)_v \mid v \in N(c) \cap V(I) \}\!\!\} \right) \coloneq \sum_{v \in N(c) \cap V(I)} \vec{A}_{cv} \left(\vec{p}^{(t-1)}_1\right)_v = \vec{A}_{c} \vec{p}^{(t-1)}_1
\end{aligned},
\end{equation*}
and  the update function,
\begin{equation*}
\begin{aligned}
\vec{C}_c^{(t,4)} & \coloneqq \textsf{UPD}^{(t,4)}_{\text{c}} \Bigl[ \vec{C}_c^{(t,3)}, \vec{A}_{c} \vec{p}^{(t-1)}_1 \Bigr] \\
& =  \vec{A}_{c} \vec{p}^{(t-1)}_1 \begin{bmatrix} -1 \\ 0 \\ 0 \\ 0 \\ 0 \end{bmatrix} + \vec{C}_c^{(t,3)} \begin{bmatrix} 0 & 1 & 0 & 0 & 0 \\ 0 & 0 & 1 & 0 & 0 \\ 0 & 0 & 0 & 1 & 0 \\ 0 & 0 & 0 & 0 & 1 \\ \end{bmatrix} \\
& = \begin{bmatrix} -\vec{A}_{c} \vec{p}^{(t-1)}_1 & \left(\vec{r}^{(t-1)}_2\right)_c & \left(\vec{p}^{(t-1)}_2\right)_c & \left(\vec{w}^{(t)}_2\right)_c & \boldsymbol{\lambda}_c \end{bmatrix} \\
& = \begin{bmatrix} \left( \left(\vec{P} \vec{r}^{(t-1)} \right)_2 \right)_c & \left(\vec{r}^{(t-1)}_2\right)_c & \left(\vec{p}^{(t-1)}_2\right)_c & \left(\vec{w}^{(t)}_2\right)_c & \boldsymbol{\lambda}_c \end{bmatrix}.
\end{aligned}
\end{equation*}

From constraint nodes to variable nodes, we have the message functions,
\begin{equation*}
\begin{aligned}
\textsf{MSG}^{(t,4)}_{\text{c} \rightarrow \text{v}} \left( \vec{C}^{(t,3)}_c, \vec{A}_{cv} \right) &\coloneq \vec{A}_{cv} \vec{C}^{(t,3)}_c \begin{bmatrix} 0 \\ 1 \\0 \\ 0 \end{bmatrix} = \vec{A}_{cv} \left(\vec{p}^{(t-1)}_2\right)_c, c \in N(v) \cap C(I), \\
\textsf{MSG}^{(t,4)}_{\text{v} \rightarrow \text{v}} \left( \vec{V}^{(t,3)}_u, \vec{Q}_{vu} \right) &\coloneq \vec{Q}_{vu} \vec{V}^{(t,3)}_u \begin{bmatrix} 0 \\ 1 \\ 0 \\0\\0 \end{bmatrix}  = \vec{Q}_{vu} \left(\vec{p}^{(t-1)}_1\right)_u,  u \in N(v) \cap V(I),
\end{aligned}
\end{equation*}
and the aggregation function,
\begin{equation*}
\begin{aligned}
\textsf{AGG}^{(t,4)}_{\text{c} \rightarrow \text{v}} \left( \{\!\!\{\vec{A}_{cv} \left(\vec{p}^{(t-1)}_2\right)_c \mid c \in N(v) \cap C(I) \}\!\!\} \right) &\coloneq \sum_{c \in N(v) \cap C(I)} \vec{A}_{cv} \left(\vec{r}^{(t-1)}_2\right)_c = \vec{A}\trans_v \vec{p}^{(t-1)}_2, \\
\textsf{AGG}^{(t,4)}_{\text{v} \rightarrow \text{v}} \left( \{\!\!\{\vec{Q}_{vu} \left(\vec{p}^{(t-1)}_1\right)_u \mid u \in N(v) \cap V(I) \}\!\!\} \right) &\coloneq \sum_{u \in N(v) \cap V(I)} \vec{Q}_{vu} \left(\vec{p}^{(t-1)}_1\right)_u = \vec{Q}_{v} \vec{p}^{(t-1)}_1,
\end{aligned}
\end{equation*}
and the update function,
\begin{equation*}
\begin{aligned}
\vec{V}_v^{(t,4)} & \coloneqq \textsf{UPD}^{(t,4)}_{\text{v}} \Bigl[ \vec{V}_v^{(t,3)}, \vec{A}\trans_v \vec{p}^{(t-1)}_2, \vec{Q}_{v} \vec{p}^{(t-1)}_1 \Bigr] \\
& = \vec{V}_v^{(t,3)} \begin{bmatrix} 0 & 1 & 0 & 0 & 0 & 0 \\ 0 & 0 & 1 & 0 & 0 & 0 \\ 0 & 0 & 0 & 1 & 0 & 0 \\ 0 & 0 & 0 & 0 & 1 & 0 \\ 0 & 0 & 0 & 0 & 0 & 1 \\ \end{bmatrix} \\
& + \left(\vec{Q}_{v} \vec{p}^{(t-1)}_1 - \vec{A}\trans_v \vec{p}^{(t-1)}_2 + \dfrac{ \vec{V}_v^{(t,3)} \begin{bmatrix} 0 & 0 & 0 & 0 & 0 \\ 0 & 0 & 0 & 0 & 1 \\ 0 & 0 & 0 & 0 & 0 \\ 0 & 0 & 0 & 0 & 0 \\ 0 & 0 & 0 & 0 & 0 \\ \end{bmatrix} \left( \vec{V}_v^{(t,3)} \right)\trans}{\vec{V}_v^{(t,3)} \begin{bmatrix} 0 \\ 0 \\ 0 \\ 1 \\ 0 \end{bmatrix}} \right) \begin{bmatrix} 1 & 0 & 0 & 0 & 0 & 0 \end{bmatrix} \\
& = \begin{bmatrix} \vec{Q}_{v} \vec{p}^{(t-1)}_1 - \vec{A}\trans_v \vec{p}^{(t-1)}_2 + \dfrac{\vec{s}_v \left(\vec{p}_1^{(t-1)} \right)_v}{\vec{x}_v} & \left(\vec{r}^{(t-1)}_1\right)_v & \left(\vec{p}^{(t-1)}_1\right)_v & \left(\vec{w}^{(t)}_1\right)_v & \vec{x}_v  & \vec{s}_v \end{bmatrix} \\
& = \begin{bmatrix} \left( \left(\vec{P} \vec{p}^{(t-1)} \right)_1 \right)_v & \left(\vec{r}^{(t-1)}_1\right)_v & \left(\vec{p}^{(t-1)}_1\right)_v & \left(\vec{w}^{(t)}_1\right)_v & \vec{x}_v  & \vec{s}_v \end{bmatrix}.
\end{aligned}
\end{equation*}

Now, we have updated node embeddings,
\begin{equation*}
\begin{aligned}
\vec{V}^{(t,4)} & \coloneqq  \begin{bmatrix} \left(\vec{P} \vec{p}^{(t-1)} \right)_1 &  \vec{r}_1^{(t-1)} & \vec{p}^{(t-1)}_1 & \vec{w}^{(t)}_1 & \vec{x} & \vec{s} \end{bmatrix},\\
\vec{C}^{(t,4)} & \coloneqq \begin{bmatrix} \left(\vec{P} \vec{p}^{(t-1)} \right)_1 & \vec{r}_2^{(t-1)} &  \vec{p}^{(t-1)}_2 & \vec{w}^{(t)}_2 & \boldsymbol{\lambda} \end{bmatrix}, \\
\vec{G}^{(t,4)} &\coloneq \alpha^{(t)}.
\end{aligned}
\end{equation*}
We notice that $\vec{G}^{(t,4)}$ is inherited from $\vec{G}^{(t,2)}$ as we have never updated it. We convey the message from $\vec{G}^{(t,4)}$ with an identity mapping,
\begin{equation*}
\begin{aligned}
\textsf{MSG}^{(t,5)}_{\text{g} \rightarrow \text{v}} \left( \vec{G}^{(t,4)} \right) &\coloneq \vec{G}^{(t,4)} = \alpha^{(t)}, \\
\textsf{MSG}^{(t,5)}_{\text{g} \rightarrow \text{c}} \left( \vec{G}^{(t,4)} \right) &\coloneq \vec{G}^{(t,4)} = \alpha^{(t)}.
\end{aligned}
\end{equation*}
Now, we can update the vector $\vec{r}$ with a simple update function,
\begin{equation*}
\begin{aligned}
\vec{V}_v^{(t,5)} & \coloneqq \textsf{UPD}^{(t,5)}_{\text{v}} \Bigl[ \vec{V}_v^{(t,4)}, \alpha^{(t)} \Bigr] \\
&= \vec{V}_v^{(t,4)} \begin{bmatrix} 0 & 0 & 0 & 0 & 0 & 0 \\ 1 & 1 & 0 & 0 & 0 & 0 \\ 0 & 0 & 1 & 0 & 0 & 0 \\ 0 & 0 & 0 & 1 & 0 & 0 \\ 0 & 0 & 0 & 0 & 1 & 0 \\  0 & 0 & 0 & 0 & 0 & 1 \\ \end{bmatrix} - \alpha^{(t)} \vec{V}_v^{(t,4)} \begin{bmatrix} 1 & 0 & 0 & 0 & 0 & 0 \\ 0 & 0 & 0 & 0 & 0 & 0\\ 0 & 0 & 0 & 0 & 0 & 0\\ 0 & 0 & 0 & 0 & 0 & 0\\ 0 & 0 & 0 & 0 & 0 & 0\\ 0 & 0 & 0 & 0 & 0 & 0\\ \end{bmatrix} \\
& = \begin{bmatrix} \left(\vec{r}^{(t)}_1\right)_v & \left(\vec{r}^{(t-1)}_1\right)_v & \left(\vec{p}^{(t-1)}_1\right)_v & \left(\vec{w}^{(t)}_1\right)_v & \vec{x}_v  & \vec{s}_v \end{bmatrix},
\end{aligned}
\end{equation*}
and 
\begin{equation*}
\begin{aligned}
\vec{C}_c^{(t,5)} & \coloneqq \textsf{UPD}^{(t,5)}_{\text{c}} \Bigl[ \vec{C}_c^{(t,4)}, \alpha^{(t)} \Bigr] \\
&= \vec{C}_c^{(t,4)} \begin{bmatrix} 0 & 0 & 0 & 0 & 0 \\ 1 & 1 & 0 & 0 & 0 \\ 0 & 0 & 1 & 0 & 0  \\ 0 & 0 & 0 & 1 & 0 \\ 0 & 0 & 0 & 0 & 1 \\ \end{bmatrix} - \alpha^{(t)} \vec{C}_c^{(t,4)} \begin{bmatrix} 1 & 0 & 0 & 0 & 0  \\ 0 & 0 & 0 & 0 & 0 \\ 0 & 0 & 0 & 0 & 0 \\ 0 & 0 & 0 & 0 & 0 \\ 0 & 0 & 0 & 0 & 0 \\ \end{bmatrix} \\
& = \begin{bmatrix} \left(\vec{r}^{(t)}_2\right)_c & \left(\vec{r}^{(t-1)}_2\right)_c & \left(\vec{p}^{(t-1)}_2\right)_c & \left(\vec{w}^{(t)}_2\right)_c & \boldsymbol{\lambda}_c \end{bmatrix}.
\end{aligned}
\end{equation*}
We now obtain the node embeddings
\begin{equation*}
\begin{aligned}
\vec{V}^{(t,5)} & \coloneqq  \begin{bmatrix} \vec{r}_1^{(t)} &  \vec{r}_1^{(t-1)} & \vec{p}^{(t-1)}_1 & \vec{w}^{(t)}_1 & \vec{x} & \vec{s} \end{bmatrix}\\
\vec{C}^{(t,5)} & \coloneqq \begin{bmatrix} \vec{r}_2^{(t)} & \vec{r}_2^{(t-1)} &  \vec{p}^{(t-1)}_2 & \vec{w}^{(t)}_2 & \boldsymbol{\lambda} \end{bmatrix}.
\end{aligned}
\end{equation*}

\paragraph{Line 11}
Line 11 can also be done with one step from variable and constraint nodes to the global node. The message functions are,
\begin{equation*}
\begin{aligned}
\textsf{MSG}^{(t,6)}_{\text{v} \rightarrow \text{g}}\left(\vec{V}_v^{(t,5)} \right) &\coloneq \vec{V}_v^{(t,5)} \begin{bmatrix} 1 & 0 & 0 & 0 & 0 & 0  \\ 0 & 0 & 0 & 0 & 0& 0  \\ 0 & 0 & 0 & 0 & 0 & 0 \\ 0 & 0 & 0 & 0 & 0 & 0 \\ 0 & 0 & 0 & 0 & 0 & 0 \\ 0 & 0 & 0 & 0 & 0 & 0 \\ \end{bmatrix} \left(\vec{V}_v^{(t,5)}\right)\trans \begin{bmatrix} 1 & 0 \end{bmatrix} \\
& + \vec{V}_v^{(t,5)} \begin{bmatrix} 0 & 0 & 0 & 0 & 0 & 0  \\ 0 & 1 & 0 & 0 & 0& 0  \\ 0 & 0 & 0 & 0 & 0 & 0 \\ 0 & 0 & 0 & 0 & 0 & 0 \\ 0 & 0 & 0 & 0 & 0 & 0 \\ 0 & 0 & 0 & 0 & 0 & 0 \\ \end{bmatrix} \left(\vec{V}_v^{(t,5)}\right)\trans \begin{bmatrix} 0 & 1 \end{bmatrix} \\
& = \begin{bmatrix} \left(\vec{r}_1^{(t)} \right)_v^2 & \left(\vec{r}_1^{(t-1)} \right)_v^2 \end{bmatrix}, v \in V(I), \\
\textsf{MSG}^{(t,6)}_{\text{c} \rightarrow \text{g}}\left(\vec{C}_c^{(t,5)} \right) &\coloneq \vec{C}_c^{(t,5)} \begin{bmatrix} 1 & 0 & 0 & 0 & 0  \\ 0 & 0 & 0 & 0 & 0  \\ 0 & 0 & 0 & 0 & 0 \\ 0 & 0 & 0 & 0 & 0 \\ 0 & 0 & 0 & 0 & 0 \\ \end{bmatrix} \left(\vec{C}_c^{(t,5)}\right)\trans \begin{bmatrix} 1 & 0 \end{bmatrix} \\
& + \vec{C}_c^{(t,5)} \begin{bmatrix} 0 & 0 & 0 & 0 & 0 \\ 0 & 1 & 0 & 0 & 0 \\ 0 & 0 & 0 & 0 & 0 \\ 0 & 0 & 0 & 0 & 0 \\ 0 & 0 & 0 & 0 & 0 \\ \end{bmatrix} \left(\vec{C}_c^{(t,5)}\right)\trans \begin{bmatrix} 0 & 1 \end{bmatrix} \\
& = \begin{bmatrix} \left(\vec{r}_2^{(t)} \right)_c^2 & \left(\vec{r}_2^{(t-1)} \right)_c^2 \end{bmatrix}, c \in C(I),
\end{aligned}
\end{equation*}
and the aggregation function for the global node is,
\begin{equation*}
\begin{aligned}
\textsf{AGG}^{(t,6)}_{\text{v} \rightarrow \text{g}} &\left(\{\!\!\{ \begin{bmatrix} \left(\vec{r}_1^{(t)} \right)_v^2 & \left(\vec{r}_1^{(t-1)} \right)_v^2 \end{bmatrix} \mid v \in V(I) \}\!\!\}  \right) \\
& \coloneq \sum_{v \in V(I)} \begin{bmatrix} \left(\vec{r}_1^{(t)} \right)_v^2 & \left(\vec{r}_1^{(t-1)} \right)_v^2 \end{bmatrix} \\
& = \begin{bmatrix} \left(\vec{r}_1^{(t)} \right)\trans \vec{r}_1^{(t)} & \left(\vec{r}_1^{(t-1)} \right)\trans \vec{r}_1^{(t-1)}\end{bmatrix}, \\
\textsf{AGG}^{(t,6)}_{\text{c} \rightarrow \text{g}} &\left(\{\!\!\{ \begin{bmatrix} \left(\vec{r}_2^{(t)} \right)_c^2 & \left(\vec{r}_2^{(t-1)} \right)_c^2 \end{bmatrix} \mid c \in C(I) \}\!\!\}  \right) \\
& \coloneq \sum_{c \in C(I)} \begin{bmatrix} \left(\vec{r}_2^{(t)} \right)_c^2 & \left(\vec{r}_2^{(t-1)} \right)_c^2 \end{bmatrix} \\
& = \begin{bmatrix} \left(\vec{r}_2^{(t)} \right)\trans \vec{r}_2^{(t)} & \left(\vec{r}_2^{(t-1)} \right)\trans \vec{r}_2^{(t-1)}\end{bmatrix},
\end{aligned}
\end{equation*}
and the update function for the global node,
\begin{equation*}
\begin{aligned}
&\vec{G}^{(t,6)} \coloneqq  \textsf{UPD}^{(t,6)}_{\text{g}} \Bigl[
 \begin{bmatrix} \left(\vec{r}_1^{(t)} \right)\trans \vec{r}_1^{(t)} & \left(\vec{r}_1^{(t-1)} \right)\trans \vec{r}_1^{(t-1)}\end{bmatrix}, \begin{bmatrix} \left(\vec{r}_2^{(t)} \right)\trans \vec{r}_2^{(t)} & \left(\vec{r}_2^{(t-1)} \right)\trans \vec{r}_2^{(t-1)}\end{bmatrix} \Bigr] \\
& = \dfrac{ \left( \begin{bmatrix} \left(\vec{r}_1^{(t)} \right)\trans \vec{r}_1^{(t)} & \left(\vec{r}_1^{(t-1)} \right)\trans \vec{r}_1^{(t-1)}\end{bmatrix} + \begin{bmatrix} \left(\vec{r}_2^{(t)} \right)\trans \vec{r}_2^{(t)} & \left(\vec{r}_2^{(t-1)} \right)\trans \vec{r}_2^{(t-1)}\end{bmatrix} \right) \begin{bmatrix} 1 \\ 0 \end{bmatrix}}{\left( \begin{bmatrix} \left(\vec{r}_1^{(t)} \right)\trans \vec{r}_1^{(t)} & \left(\vec{r}_1^{(t-1)} \right)\trans \vec{r}_1^{(t-1)}\end{bmatrix} + \begin{bmatrix} \left(\vec{r}_2^{(t)} \right)\trans \vec{r}_2^{(t)} & \left(\vec{r}_2^{(t-1)} \right)\trans \vec{r}_2^{(t-1)}\end{bmatrix} \right) \begin{bmatrix} 0 \\ 1 \end{bmatrix}} \\
& = \dfrac{\left(\vec{r}_1^{(t)} \right)\trans \vec{r}_1^{(t)} + \left(\vec{r}_2^{(t)} \right)\trans \vec{r}_2^{(t)}}{\left(\vec{r}_1^{(t-1)} \right)\trans \vec{r}_1^{(t-1)} + \left(\vec{r}_2^{(t-1)} \right)\trans \vec{r}_2^{(t-1)}} = \dfrac{\left(\vec{r}^{(t)} \right)\trans \vec{r}^{(t)}}{\left(\vec{r}^{(t-1)} \right)\trans \vec{r}^{(t-1)}} = \beta^{(t)}.
\end{aligned}
\end{equation*}
Hence, $\beta^{(t)}$ is stored in the global node. Meanwhile we can discard the unnecessary $\vec{r}^{(t-1)}$ with simple update function,
\begin{equation*}
\begin{aligned}
\vec{C}_c^{(t,6)} & \coloneqq \textsf{UPD}^{(t,6)}_{\text{c}} \Bigl[ \vec{C}_c^{(t,5)} \Bigr] \\
&= \vec{C}_c^{(t,5)} \begin{bmatrix} 1 & 0 & 0 & 0 \\ 0 & 0 & 0 & 0 \\ 0 & 1 & 0 & 0 \\ 0 & 0 & 1 & 0  \\ 0 & 0 & 0 & 1 \end{bmatrix} \\
& = \begin{bmatrix} \left(\vec{r}^{(t)}_2\right)_c & \left(\vec{p}^{(t-1)}_2\right)_c & \left(\vec{w}^{(t)}_2\right)_c & \boldsymbol{\lambda}_c \end{bmatrix},
\end{aligned}
\end{equation*}
and,
\begin{equation*}
\begin{aligned}
\vec{V}_v^{(t,6)} & \coloneqq \textsf{UPD}^{(t,6)}_{\text{v}} \Bigl[ \vec{V}_v^{(t,5)} \Bigr] \\
&= \vec{V}_v^{(t,5)} \begin{bmatrix} 1 & 0 & 0 & 0 & 0 \\ 0 & 0 & 0 & 0 & 0 \\ 0 & 1 & 0 & 0 & 0 \\ 0 & 0 & 1 & 0 & 0  \\ 0 & 0 & 0 & 1 & 0 \\0 & 0 & 0 & 0 & 1 \end{bmatrix} \\
& = \begin{bmatrix} \left(\vec{r}^{(t)}_1\right)_v & \left(\vec{p}^{(t-1)}_1\right)_v & \left(\vec{w}^{(t)}_1\right)_v & \vec{x}_v  & \vec{s}_v \end{bmatrix},
\end{aligned}
\end{equation*}
which leads us to,
\begin{equation*}
\begin{aligned}
\vec{V}^{(t,6)} & \coloneqq  \begin{bmatrix} \vec{r}_1^{(t)} & \vec{p}^{(t-1)}_1 & \vec{w}^{(t)}_1 & \vec{x} & \vec{s} \end{bmatrix}\\
\vec{C}^{(t,6)} & \coloneqq \begin{bmatrix} \vec{r}_2^{(t)} &  \vec{p}^{(t-1)}_2 & \vec{w}^{(t)}_2 & \boldsymbol{\lambda} \end{bmatrix}, \\
\end{aligned}
\end{equation*}
after stacking up the nodes. 

\paragraph{Line 12}
What is left now is updating $\vec{p}^{(t-1)}$ in line 12 with one message passing step. Again, we need to retrieve the information from the global node for all the nodes,
\begin{equation*}
\begin{aligned}
\textsf{MSG}^{(t,7)}_{\text{g} \rightarrow \text{v}} \left( \vec{G}^{(t,6)} \right) &\coloneq \vec{G}^{(t,6)} = \beta^{(t)}, \\
\textsf{MSG}^{(t,7)}_{\text{g} \rightarrow \text{c}} \left( \vec{G}^{(t,6)} \right) &\coloneq \vec{G}^{(t,6)} = \beta^{(t)}.
\end{aligned}
\end{equation*}
Now, we update for the $\vec{p}^{(t-1)}$,
\begin{equation*}
\begin{aligned}
\vec{C}_c^{(t,7)} & \coloneqq \textsf{UPD}^{(t,7)}_{\text{c}} \Bigl[ \vec{C}_c^{(t,6)}, \beta^{(t)} \Bigr] \\
&= \beta^{(t)} \vec{C}_c^{(t,6)} \begin{bmatrix} 0 & 0 & 0 & 0 \\ 0 & 1 & 0 & 0 \\ 0 & 0 & 0 & 0 \\ 0 & 0 & 0 & 0 \\ \end{bmatrix} + \vec{C}_c^{(t,6)} \begin{bmatrix} 1 & 1 & 0 & 0 \\ 0 & 0 & 0 & 0 \\ 0 & 0 & 1 & 0 \\ 0 & 0 & 0 & 1 \\ \end{bmatrix} \\
& = \begin{bmatrix} \left(\vec{r}^{(t)}_2\right)_c & \left(\vec{p}^{(t)}_2\right)_c & \left(\vec{w}^{(t)}_2\right)_c & \boldsymbol{\lambda}_c \end{bmatrix},
\end{aligned}
\end{equation*}
and 
\begin{equation*}
\begin{aligned}
\vec{V}_v^{(t,7)} & \coloneqq \textsf{UPD}^{(t,7)}_{\text{v}} \Bigl[ \vec{V}_v^{(t,6)}, \beta^{(t)} \Bigr] \\
&= \beta^{(t)} \vec{V}_v^{(t,6)} \begin{bmatrix} 0 & 0 & 0 & 0 & 0 \\ 0 & 1 & 0 & 0& 0 \\ 0 & 0 & 0 & 0& 0 \\ 0 & 0 & 0 & 0& 0 \\ 0 & 0 & 0 & 0& 0 \\ \end{bmatrix} + \vec{V}_v^{(t,6)} \begin{bmatrix} 1 & 1 & 0 & 0 & 0 \\ 0 & 0 & 0 & 0 & 0 \\ 0 & 0 & 1 & 0 & 0\\ 0 & 0 & 0 & 1 & 0 \\ 0 & 0 & 0 & 0 & 1 \\ \end{bmatrix} \\
& = \begin{bmatrix} \left(\vec{r}^{(t)}_1\right)_v & \left(\vec{p}^{(t)}_1\right)_v & \left(\vec{w}^{(t)}_1\right)_v & \vec{x}_v & \vec{s}_v \end{bmatrix}.
\end{aligned}
\end{equation*}
Finally, we have
\begin{equation*}
\begin{aligned}
\vec{V}^{(t)} &\coloneq \vec{V}^{(t,7)} = \begin{bmatrix} \vec{r}_1^{(t)} & \vec{p}^{(t)}_1 & \vec{w}^{(t)}_1 & \vec{x} & \vec{s} \end{bmatrix} \\
\vec{C}^{(t)} &\coloneq \vec{C}^{(t,7)} = \begin{bmatrix} \vec{r}_2^{(t)} & \vec{p}^{(t)}_2 & \vec{w}^{(t)}_2 & \boldsymbol{\lambda} \end{bmatrix},\\
\end{aligned}
\end{equation*}
and we can proceed with the next iteration.

In summary, the initialization of \cref{alg:cg-ipm} takes one step of message passing, while each iteration takes seven. Hence, the runtime complexity is $\cO(m + n)$, i.e., linear time in the problem size. However, if we take a closer look at the $\textsf{UPD}, \textsf{MSG}, \textsf{AGG}$ functions, we will notice they are specifically parametrized and shared across different $t$. Therefore, we need only $1 + 7$, a constant number of message-passing layers, with the first executed once and the rest executed in a loop repeatedly. 

At the end of the algorithm, we take the third column of the node embeddings $\vec{V}^{(n+m)}$ and $\vec{C}^{(n+m)}$ as the solution for $\Delta \vec{x}$ and $\Delta \boldsymbol{\lambda}$.

\end{proof}

Given the CG algorithm that can solve for $\Delta \vec{x}, \Delta \boldsymbol{\lambda}$, we can construct the algorithm for solving the QP problem in \cref{alg:ipm-practice}. 

\begin{algorithm}
\caption{Practical IPM for QPs}
\label{alg:ipm-practice}
\begin{algorithmic}[1]
\REQUIRE An instance $(\vec{Q}, \vec{A},\vec{b},\vec{c})$, a barrier reduction hyperparameter $\sigma \in (0,1)$, initial solution $(\vec{x}^{(0)}, \boldsymbol{\lambda}^{(0)}, \vec{s}^{(0)})$ with $\vec{x}^{(0)},\vec{s}^{(0)} > \bm{0}$ and $\mu^{(0)}={\vec{x}^{(0)}}\trans \vec{s}^{(0)}/n$.

\STATE $t \leftarrow 1$
\REPEAT
\STATE Compute $\Delta \vec{x}^{(t)}, \Delta \boldsymbol{\lambda}^{(t)}$ by solving \cref{eq:augmented_eq} using \cref{alg:cg-ipm}
\STATE Compute $\Delta \vec{s}^{(t)}$ with \cref{eq:eliminate_s}
\STATE $\alpha^{(t)} \leftarrow \min \left\{1, \sup \left\{ \alpha \mid \vec{x}^{(t)} + \alpha \Delta \vec{x}^{(t)} \geq \bm{0} \right\}, \sup \left\{ \alpha \mid \vec{s}^{(t)} + \alpha \Delta \vec{s}^{(t)} \geq \bm{0}\right\} \right\}$

\STATE $\vec{x}^{(t)} \leftarrow \vec{x}^{(t-1)} + \text{0.99} \alpha^{(t)} \Delta \vec{x}^{(t)}$
\STATE $\vec{s}^{(t)} \leftarrow \vec{s}^{(t-1)} + \text{0.99} \alpha^{(t)} \Delta \vec{s}^{(t)}$
\STATE $\boldsymbol{\lambda}^{(t)} \leftarrow \boldsymbol{\lambda}^{(t-1)} + \text{0.99} \alpha^{(t)} \Delta \boldsymbol{\lambda}^{(t)}$
\STATE $\mu^{(t)} \leftarrow \sigma \mu^{(t-1)}$
\STATE $t \leftarrow t + 1$
\UNTIL{convergence of $(\vec{x}, \boldsymbol{\lambda}, \vec{s})$}
\RETURN the point $\vec{x}$
\end{algorithmic}
\end{algorithm}

\begin{theorem}\label{thm:ipm-is-MPNN}
There exists an MPNN $f_{\textsf{MPNN,IPM}}$ composed of $\cO(1)$ layers and $\cO(m+n)$ successive message-passing steps that reproduces each iteration pf  \cref{alg:ipm-practice}, in the sense that for any QP instance $I = \left(\vec{Q}, \vec{A},\vec{b},\vec{c} \right)$ and any primal-dual point $\left(\vec{x}^{(t)}, \boldsymbol{\lambda}^{(t)}, \vec{s}^{(t)}\right)$ with $t > 0$, $f_{\textsf{MPNN,IPM}}$ maps the graph $G(I)$ carrying $[\vec{x}^{(t-1)}, \vec{s}^{(t-1)}]$ on the variable nodes, $[\boldsymbol{\lambda}^{(t-1)}]$ on the constraint nodes. and $[\mu, \sigma]$ on the global node to the same graph $G(I)$ carrying the output $[\vec{x}^{(t)}, \vec{s}^{(t)}]$ and $[\boldsymbol{\lambda}^{(t)}]$ of \cref{alg:ipm-practice} on the variable nodes and constraint nodes, respectively.
\end{theorem}

\begin{proof} For the details of the construction $G(I)$, we follow \cref{thm:cg-is-MPNN}. At the beginning of the algorithm, we construct the graph with node feature matrices initialized as following,
\begin{equation*}
\begin{aligned}
\vec{V}^{(0)} & \coloneqq \begin{bmatrix} \vec{x}^{(0)} & \vec{s}^{(0)} & \vec{c} \end{bmatrix} \in \mathbb{R}^{n \times 3}, \\
\vec{C}^{(0)} & \coloneqq \begin{bmatrix} \boldsymbol{\lambda}^{(0)} & \vec{b} \end{bmatrix} \in \mathbb{R}^{m \times 2}, \\
\vec{G}^{(0)} & \coloneqq \begin{bmatrix}\sigma &  \mu \end{bmatrix} \in \mathbb{R}^2.
\end{aligned}
\end{equation*}

\paragraph{Line 3}
For $t > 0$, we have the vectors $\Delta \vec{x}^{(t)}, \Delta \boldsymbol{\lambda}^{(t)}$ given by \cref{alg:cg-ipm}. Following the MPNN framework \cref{eq:proof_mpnn_framework}, we compute the update functions for variable nodes,
\begin{equation*}
\begin{aligned}
\vec{V}_v^{(t,1)} & \coloneqq \textsf{UPD}^{(t,1)}_{\text{v}} \Bigl[ \vec{V}_v^{(t-1)}, \Delta \vec{x}^{(t)}_v \Bigr] \\
&= \Delta \vec{x}^{(t)}_v \begin{bmatrix} 1 & 0 & 0 & 0 \end{bmatrix} + \vec{V}_v^{(t-1)} \begin{bmatrix} 0 & 0 & 0 & 0 \\ 0 & 1 & 0 & 0 \\ 0 & 0 & 1 & 0 \\ 0 & 0 & 0 & 1 \\ \end{bmatrix} \\
& = \begin{bmatrix} \Delta \vec{x}^{(t)}_v & \vec{x}^{(t-1)}_v & \vec{s}^{(t-1)}_v  & \vec{c}_v  \end{bmatrix},
\end{aligned}
\end{equation*}
and constraint nodes,
\begin{equation*}
\begin{aligned}
\vec{C}_c^{(t,1)} & \coloneqq \textsf{UPD}^{(t,1)}_{\text{c}} \Bigl[ \vec{C}_c^{(t-1)}, \Delta \boldsymbol{\lambda}^{(t)}_c \Bigr] \\
&= \Delta \boldsymbol{\lambda}^{(t)}_c \begin{bmatrix} 1 & 0 & 0 \end{bmatrix} + \vec{C}_c^{(t-1)} \begin{bmatrix} 0 & 0 & 0 \\ 0 & 1 & 0 \\ 0 & 0 & 1 \end{bmatrix} \\
& = \begin{bmatrix} \Delta \boldsymbol{\lambda}^{(t)}_c & \boldsymbol{\lambda}^{(t-1)}_c & \vec{b}_c \end{bmatrix}.
\end{aligned}
\end{equation*}

\paragraph{Line 4}
During the derivation of IPM, we eliminated $\Delta \vec{s}$, now we can recover it as $\Delta \vec{s}^{(t)} \coloneqq - \dfrac{\vec{s}^{(t-1)} \Delta \vec{x}^{(t)}}{\vec{x}^{(t-1)}}   - \vec{s}^{(t-1)} + \dfrac{\sigma \mu \bm{1}}{\vec{x}^{(t-1)}}$. To do that, we need a message function from the global to the variable nodes,
\begin{equation*}
\textsf{MSG}^{(t,2)}_{\text{g} \rightarrow \text{v}} \left( \vec{G}^{(t-1)} \right) \coloneq \vec{G}^{(t-1)} \begin{bmatrix} 0 & 0\\ 1 & 0 \end{bmatrix} \left( \vec{G}^{(t-1)} \right)\trans = \mu \sigma,
\end{equation*}
and the update function,
\begin{equation*}
\begin{aligned}
\vec{V}_v^{(t,2)} & \coloneqq \textsf{UPD}^{(t,2)}_{\text{v}} \Bigl[ \vec{V}_v^{(t,1)}, \mu \sigma \Bigr] \\
&= \vec{V}_v^{(t,1)} \begin{bmatrix} 1 & 0 & 0 & 0 & 0 \\ 0 & 0 & 1 & 0 & 0 \\ 0 & 0 & 0 & 1 & 0 \\ 0 & 0 & 0 & 0 & 1 \end{bmatrix} \\
& = \left( \dfrac{\mu \sigma + \vec{V}_v^{(t,1)} \begin{bmatrix} 0 & 0 & 1 & 0 \\ 0 & 0 & 0 & 0 \\ 0 & 0 & 0 & 0 \\ 0 & 0 & 0 & 0  \end{bmatrix} \left(\vec{V}_v^{(t,1)}\right)\trans }{\vec{V}_v^{(t,1)} \begin{bmatrix} 0 \\ 1 \\ 0 \\ 0 \end{bmatrix}} - \vec{V}_v^{(t,1)} \begin{bmatrix} 0 \\ 0 \\ 1 \\ 0 \end{bmatrix} \right) \begin{bmatrix} 0 & 1 & 0 & 0 & 0 \end{bmatrix} \\
& = \begin{bmatrix} \Delta \vec{x}^{(t)}_v & \Delta \vec{s}^{(t)}_v & \vec{x}^{(t-1)}_v & \vec{s}^{(t-1)}_v  & \vec{c}_v  \end{bmatrix}.
\end{aligned}
\end{equation*}

\paragraph{Line 5}
Simulating the line search for the step size $\alpha^{(t)}$ is cumbersome. We need to calculate a vector $\boldsymbol{\alpha}_{\text{x}}^{(t)}$ w.r.t.\ $\vec{x}^{(t-1)}$ and $\Delta \vec{x}^{(t)}$. Each element in the vector is defined as
\begin{equation}
\label{eq:segment_step_size}
\left(\boldsymbol{\alpha}_{\text{x}}^{(t)}\right)_v \coloneq
    \begin{cases}
        - \dfrac{\vec{x}^{(t-1)}_v}{\Delta \vec{x}^{(t)}_v} & \text{if } \Delta \vec{x}^{(t)}_v < 0 \\
        +\infty & \text{else}.
    \end{cases}
\end{equation}
Then we take $\alpha_{\text{x}}^{(t)} \coloneq \inf \left\{ \left(\boldsymbol{\alpha}_{\text{x}}^{(t)}\right)_v \mid v \in V(I) \right\} \in \mathbb{R}_{> 0} \cup \{ +\infty \}$. We repeat the same process for slack variables $\vec{s}$. We then calculate $\alpha^{(t)} \coloneq \min \left\{ 1, \alpha_{\text{x}}^{(t)}, \alpha_{\text{s}}^{(t)}  \right\}$, where 1 is a constant coefficient that can be replaced by any positive number to avoid the step size being unlimited.

However, the definition in \cref{eq:segment_step_size} is not continuous as it contains $+\infty$ term. We introduce a small number $\epsilon > 0$ and modify it to be 
\begin{equation*}
\label{eq:segment_step_size_eps}
\left(\boldsymbol{\alpha}_{\text{x}}^{(t)}\right)_v \coloneq
    \begin{cases}
        - \dfrac{\vec{x}^{(t-1)}_v}{\Delta \vec{x}^{(t)}_v} & \text{if } \Delta \vec{x}^{(t)}_v <= -\epsilon \\
        \dfrac{\vec{x}^{(t-1)}_v}{\epsilon} & \text{else},
    \end{cases}
\end{equation*}
so that $\left(\boldsymbol{\alpha}_{\text{x}}^{(t)}\right)_v$ is continuous w.r.t.\ $\vec{x}^{(t-1)}$ and $\Delta \vec{x}^{(t)}$, thus compatible with neural networks. 

Now, we can design functions for the message passing. We construct the message function from variable nodes to the global node as, 
\begin{equation*}
\begin{aligned}
\textsf{MSG}^{(t,3)}_{\text{v} \rightarrow \text{g}} \left(\vec{V}_v^{(t,2)} \right) &\coloneq \dfrac{\vec{V}_v^{(t,2)} \begin{bmatrix} 0 \\ 0 \\ 1 \\ 0 \\ 0 \end{bmatrix}}{- \min\left\{-\epsilon, \vec{V}_v^{(t,2)} \begin{bmatrix} 1 \\ 0 \\ 0 \\ 0 \\ 0 \end{bmatrix}\right\}} \begin{bmatrix} 1 & 0 \end{bmatrix} + \dfrac{\vec{V}_v^{(t,2)} \begin{bmatrix} 0 \\ 0 \\ 0 \\ 1 \\ 0 \end{bmatrix}}{- \min\left\{-\epsilon, \vec{V}_v^{(t,2)} \begin{bmatrix} 0 \\ 1 \\ 0 \\ 0 \\ 0 \end{bmatrix}\right\}} \begin{bmatrix} 0 & 1 \end{bmatrix} \\
& = \begin{bmatrix} \dfrac{\vec{x}_v^{(t-1)}}{-\min\left\{ -\epsilon, \Delta \vec{x}_v^{(t)} \right\}} & \dfrac{\vec{s}_v^{(t-1)}}{-\min\left\{ -\epsilon, \Delta \vec{s}_v^{(t)} \right\}} \end{bmatrix} \\
& = \begin{bmatrix} \left(\boldsymbol{\alpha}_{\text{x}}^{(t)}\right)_v & \left(\boldsymbol{\alpha}_{\text{s}}^{(t)}\right)_v \end{bmatrix}.
\end{aligned}
\end{equation*}
Unlike other sum aggregation functions that we have used so far, here we need a min aggregation function,
\begin{equation*}
\begin{aligned}
\textsf{AGG}^{(t,3)}_{\text{v} \rightarrow \text{g}} &\left(\{\!\!\{ \begin{bmatrix} \left(\boldsymbol{\alpha}_{\text{x}}^{(t)}\right)_v & \left(\boldsymbol{\alpha}_{\text{s}}^{(t)}\right)_v \end{bmatrix} \mid v \in V(I) \}\!\!\}  \right) \\
& \coloneq \min_{v \in V(I)} \begin{bmatrix} \left(\boldsymbol{\alpha}_{\text{x}}^{(t)}\right)_v & \left(\boldsymbol{\alpha}_{\text{s}}^{(t)}\right)_v \end{bmatrix} \\
& = \begin{bmatrix} \alpha_{\text{x}}^{(t)} & \alpha_{\text{s}}^{(t)} \end{bmatrix},
\end{aligned}
\end{equation*}
and the update function of the global node,
\begin{equation*}
\begin{aligned}
\vec{G}^{(t,3)} &\coloneqq  \textsf{UPD}^{(t,3)}_{\text{g}} \Bigl[\vec{G}^{(t,2)}, \begin{bmatrix} \alpha_{\text{x}}^{(t)} & \alpha_{\text{s}}^{(t)} \end{bmatrix} \Bigr] \\
& = \vec{G}^{(t,2)} \begin{bmatrix} 1 & 0 & 0 \end{bmatrix} + \left(\min\left\{ 1, \begin{bmatrix} \alpha_{\text{x}}^{(t)} & \alpha_{\text{s}}^{(t)} \end{bmatrix} \begin{bmatrix} 1 \\ 0 \end{bmatrix}, \begin{bmatrix} \alpha_{\text{x}}^{(t)} & \alpha_{\text{s}}^{(t)} \end{bmatrix} \begin{bmatrix} 0 \\ 1 \end{bmatrix} \right\} \right) \begin{bmatrix} 0 & 0 & 1 \end{bmatrix} \\
&= \begin{bmatrix} \sigma & \mu & \alpha^{(t)} \end{bmatrix}.
\end{aligned}
\end{equation*}

\paragraph{Line 6-8}
The update of the variables is straightforward. We notice that $\vec{V}_v^{(t,3)}$ is inherited from $\vec{V}_v^{(t,2)}$ and $\vec{C}_c^{(t,3)}$ from $\vec{C}_c^{(t,1)}$ as we have never updated them. We need the step size information from the global node,
\begin{equation*}
\begin{aligned}
\textsf{MSG}^{(t,4)}_{\text{g} \rightarrow \text{v}} \left( \vec{G}^{(t,3)} \right) &\coloneq \vec{G}^{(t,3)} \begin{bmatrix} 0 & 0 & 1 \end{bmatrix} = \alpha^{(t)} \\
\textsf{MSG}^{(t,4)}_{\text{g} \rightarrow \text{c}} \left( \vec{G}^{(t,3)} \right) &\coloneq \vec{G}^{(t,3)} \begin{bmatrix} 0 & 0 & 1 \end{bmatrix} = \alpha^{(t)},
\end{aligned}
\end{equation*}
and the update function,
\begin{equation*}
\begin{aligned}
\vec{V}_v^{(t,4)} &\coloneq \textsf{UPD}^{(t,4)}_{\text{v}} \Bigl[ \vec{V}_v^{(t,3)}, \alpha^{(t)} \Bigr] \\
&= \vec{V}_v^{(t,3)} \begin{bmatrix} 0 & 0 & 0 \\ 0 & 0 & 0 \\ 1 & 0 & 0 \\ 0 & 1 & 0 \\ 0 & 0 & 1 \end{bmatrix} + 0.99 \alpha^{(t)} \vec{V}_v^{(t,3)} \begin{bmatrix} 1 & 0 & 0 \\ 0 & 1 & 0 \\ 0 & 0 & 0 \\ 0 & 0 & 0 \\ 0 & 0 & 0 \end{bmatrix} \\
& = \begin{bmatrix} \vec{x}^{(t-1)}_v & \vec{s}^{(t-1)}_v  & \vec{c}_v  \end{bmatrix} + 0.99 \alpha^{(t)} \begin{bmatrix} \Delta \vec{x}^{(t)}_v & \Delta \vec{s}^{(t)}_v & 0 \end{bmatrix}\\
& = \begin{bmatrix} \vec{x}^{(t)}_v & \vec{s}^{(t)}_v  & \vec{c}_v  \end{bmatrix}, \\
\vec{C}_c^{(t,4)} &\coloneq \textsf{UPD}^{(t,4)}_{\text{c}} \Bigl[ \vec{C}_c^{(t,3)}, \alpha^{(t)} \Bigr] \\
&= \vec{C}_c^{(t,3)} \begin{bmatrix} 0 & 0 \\ 1 & 0 \\ 0 & 1 \end{bmatrix} + 0.99 \alpha^{(t)} \vec{C}_c^{(t,3)} \begin{bmatrix} 1 & 0 \\ 0 & 0 \\ 0 & 0 \end{bmatrix} \\
& = \begin{bmatrix} \boldsymbol{\lambda}^{(t-1)}_c & \vec{b}_c \end{bmatrix} + 0.99 \alpha^{(t)} \begin{bmatrix} \Delta \boldsymbol{\lambda}^{(t)}_c & 0 \end{bmatrix} \\
& = \begin{bmatrix} \boldsymbol{\lambda}^{(t)}_c & \vec{b}_c  \end{bmatrix}. \\
\end{aligned}
\end{equation*}

We can update the global node by dropping the used $\alpha^{(t)}$ and scale it,
\begin{equation*}
\begin{aligned}
\vec{G}^{(t,4)} &\coloneqq  \textsf{UPD}^{(t,4)}_{\text{g}} \Bigl[\vec{G}^{(t,3)} \Bigr] \\
& = \vec{G}^{(t,3)} \begin{bmatrix} 1 & 0 \\ 0 & 0 \\ 0 & 0 \end{bmatrix} + \vec{G}^{(t,3)} \begin{bmatrix} 0 & 1 & 0 \\ 0 & 0 & 0 \\ 0 & 0 & 0 \end{bmatrix} \left( \vec{G}^{(t,3)} \right)\trans \begin{bmatrix} 0 & 1 \end{bmatrix} = \begin{bmatrix} \sigma & \sigma\mu \end{bmatrix},
\end{aligned}
\end{equation*}
where we overwrite $\mu \coloneq \sigma \mu$.

To wrap up, at the end of each iteration, we now have,
\begin{equation*}
\begin{aligned}
\vec{G}^{(t)} &= \vec{G}^{(t,4)} = \begin{bmatrix} \sigma & \mu \end{bmatrix}, \\
\vec{C}^{(t)} &= \vec{C}^{(t,4)} = \begin{bmatrix} \boldsymbol{\lambda}^{(t)} & \vec{b}  \end{bmatrix}, \\
\vec{V}^{(t)} &= \vec{V}^{(t,4)} = \begin{bmatrix} \vec{x}^{(t)} & \vec{s}^{(t)}  & \vec{c}  \end{bmatrix},
\end{aligned}
\end{equation*}
and we are ready for the next iteration. 

Observing the functions we have defined, we see that they are shared across different iterations $t$. Combining the results from \cref{thm:cg-is-MPNN}, we have an MPNN with $\cO(1)$ layers and $\cO(m+n)$ message passing steps that can reproduce an iteration in \cref{thm:ipm-is-MPNN}.
\end{proof}

\subsubsection{MPNNs can predict displacement}
Let us quickly recap the QP graph representation from \citet{chen2024qp} with our notations. They consider LCQPs of the form
\begin{equation}
\begin{aligned}
\label{eq:qp_ineq}
\min_{\vec{x}} \quad & \frac{1}{2} \vec{x}\trans \vec{Q} \vec{x} + \vec{c}\trans \vec{x} \\
\text{s.t.} \quad & \vec{A} \vec{x} \circ \vec{b} \\
& \vec{x} \in \mathbb{Q}^n, \vec{l} \leq \vec{x} \leq \vec{u},
\end{aligned}
\end{equation}
where $\circ \in \{\leq, =, \geq \}^m$, $\vec{l} \in \left( \mathbb{Q} \cup -\infty \right)^n$, and $\vec{u} \in \left( \mathbb{Q} \cup \infty \right)^n$. They encode an LCQP into a graph with variable and constraint node types, where $\vec{Q}$ and $\vec{A}$ are represented within the intra- and inter-node edges, respectively. Particularly, the node embeddings are initialized as
\begin{equation}
\label{eq:chen_qp_init}
\begin{aligned}
    \vec{h}_c^{(0)} \coloneqq & \textsf{INIT}_{\text{c}} \left( \vec{b}_c, \circ_c \right), \forall c \in C(I),\\
    \vec{h}_v^{(0)} \coloneqq & \textsf{INIT}_{\text{v}} \left( \vec{c}_v, \vec{l}_v, \vec{u}_v \right), \forall v \in V(I).
\end{aligned}
\end{equation}
The message-passing steps in their work would be,
\begin{equation}
\label{eq:chen_qp_mp}
\begin{aligned}
    \vec{h}_c^{(l)} \coloneqq \textsf{UPD}^{(l)}_{\text{c}}\Bigl[ \vec{h}_c^{(l-1)}, & \sum_{v \in {N}\left(c \right) \cap V(I)} \vec{A}_{cv} \vec{h}_v^{(l-1)} \Bigr] \in \mathbb{Q}^d, \\
    \vec{h}_v^{(l)} \coloneqq \textsf{UPD}^{(l)}_{\text{v}}\Bigl[ \vec{h}_v^{(l-1)}, & \sum_{u \in {N}\left(v \right) \cap V(I)} \vec{Q}_{vu} \vec{h}_u^{(l-1)},\\
    & \sum_{c \in {N}\left(v \right) \cap C(I)} \vec{A}_{cv} \vec{h}_c^{(l-1)} \Bigr] \in \mathbb{Q}^d.
\end{aligned}
\end{equation}
Finally, the prediction head
\begin{equation}
\label{eq:chen_qp_pred}
\begin{aligned}
    \vec{x}_v \coloneqq \textsf{READOUT}\left( \sum_{v \in V(I)}\vec{h}_v^{(L)}, \sum_{c \in C(I)} \vec{h}_c^{(L)}, \vec{h}_v^{(L)}\right) \in \mathbb{Q}.
\end{aligned}
\end{equation}
They assume all the functions $\textsf{INIT}, \textsf{UPD}, \textsf{READOUT}$ are parametrized by MLPs with ReLU activation functions. 

\begin{lemma}
(Reformulation of Theorem 3.4 in \citet{chen2024qp}) Given a LCQP instance $I$, assume $I$ is feasible with solution $\vec{x}^*$, for any $\epsilon, \delta > 0$, there exists an MPNN $f_{\textsf{MPNN,1}}$ of the form \cref{eq:chen_qp_init,eq:chen_qp_mp,eq:chen_qp_pred}, such that
\begin{equation}
    P\left[ \lVert f_{\textsf{MPNN,1}}(I) - \vec{x}^* \rVert_2 > \delta \right] < \epsilon
\end{equation}
\end{lemma}

\begin{proof}
    See~\citet{chen2024qp} for a proof.
\end{proof}

We now make some slight modifications to \cref{eq:chen_qp_init}. We remove $\circ_c, c \in C(I)$, and $\vec{l}_v, \vec{u}_v, v \in V(I)$ from the initializations. Instead, we add an initial solution $\vec{x}^{(0)}$ as an extra input for $\textsf{INIT}_{\text{v}}$. Specifically, now we have the initialization,
\begin{equation*}
\label{eq:qian_qp_init}
\begin{aligned}
    \vec{h}_c^{(0)} \coloneqq & \textsf{INIT}_{\text{c}} \left( \vec{b}_c \right), \forall c \in C(I),\\
    \vec{h}_v^{(0)} \coloneqq & \textsf{INIT}_{\text{v}} \left( \vec{c}_v, \vec{x}_v^{(0)} \right),  v \in V(I),
\end{aligned}
\end{equation*}
and we get the following result.
\begin{theorem}
Given an LCQP instance $I$, assume $I$ is feasible with solution $\vec{x}^*$, $\vec{x}^{(0)} \geq \bm{0}$ is an initial point, for any $\epsilon, \delta > 0$, there exists an MPNN $f_{\textsf{MPNN,2}}$ of the form \cref{eq:qian_qp_init,eq:chen_qp_mp,eq:chen_qp_pred}, such that
\begin{equation*}
    P\left[ \lVert f_{\textsf{MPNN,2}}\left(I, \vec{x}^{(0)}\right) - \left(\vec{x}^* - \vec{x}^{(0)} \right) \rVert_2 > \delta \right] < \epsilon.
\end{equation*}
\end{theorem}

\begin{proof}
For now, we assume the MLPs in $f_{\textsf{MPNN,1}}$ all have the form 
\begin{equation*}
\textsf{MLP}(\vec{x}) \coloneq \vec{W}^{(L)} \sigma \left( \cdots \sigma \left( \vec{W}^{(1)} \vec{x} + \vec{w}^{(1)} \right) \right) + \vec{w}^{(L)},
\end{equation*}
where $\sigma$ is a ReLU activation function not applied to the last layer.

Let us remove the dependency on $\circ, \vec{l}, \vec{u}$. In our case, we restrict the problem instances to have equality constraints, and the bound of variables is always $[0, +\infty)$. Therefore, the $\circ_c, c \in C(I)$ and $\vec{l}_v, \vec{u}_v, v \in V(I)$ are constant for all the variables and constraints, for all the instances, not contributing to the distinguish of variables or QP instances. 
We can formulate the initialization in \cref{eq:chen_qp_init} as
\begin{equation*}
\vec{h}_c^{(0)} \coloneqq \textsf{INIT}_{\text{c},1} \left( \vec{b}_c, \circ_c \right) = \vec{W}_{(1)}^{(L)} \sigma \left( \cdots \sigma \left( \vec{W}_{(1)}^{(1)} \vec{b}_c + \text{const} + \vec{w}^{(1)}_{(1)} \right) \right) + \vec{w}^{(L)}_{(1)}.
\end{equation*}
Therefore, it is clear that we can find an initialization $\textsf{INIT}_{\text{c},1'}$ that has a bias $\vec{w}^{(1)}_{(1')} \coloneq \vec{w}^{(1)}_{(1)} + \text{const}$, while other parameters remaining the same, that can recover $\textsf{INIT}_{\text{c},1}$ without the input $\circ_c$. Similarly, we can construct $\textsf{INIT}_{\text{v},1'}$ that takes $\vec{c}_v$ only as the input. 

Now, we would like to have further $\vec{x}^{(0)}$ as the input and add it to the output of $f_{\textsf{MPNN,1}}$. We concatenate $\vec{x}_c^{(0)}$ to $\vec{c}_v$, and we construct the initialization as following,
\begin{equation*}
\begin{aligned}
& \textsf{INIT}_{\text{v},2} \left( \vec{c}_v, \vec{x}^{(0)}_c \right) \\
& = \begin{bmatrix} \vec{W}_{(L)}^{(1)} & 0 \\ 0 & 1 \end{bmatrix} \sigma \left( \cdots \sigma \left(  \begin{bmatrix} \vec{W}_{(1)}^{(1)} & 0 \\ 0 & 1 \end{bmatrix} \begin{bmatrix} \vec{c}_v \\ \vec{x}^{(0)}_c \end{bmatrix} + \begin{bmatrix} \vec{w}^{(1)}_{(1')} \\ 0 \end{bmatrix} \right) \right) + \begin{bmatrix} \vec{w}^{(L)}_{(1)} \\ 0 \end{bmatrix} \\
& = \begin{bmatrix} \vec{h}_v^{(0)} \\ \vec{x}^{(0)}_c \end{bmatrix}.
\end{aligned}
\end{equation*}
The final equation holds, for $\vec{x}^{(0)}$ is assumed to be non-negative $\vec{x}^{(0)} \geq \bm{0}$, therefore the ReLU activation function does not have any effect. For the initialization of $\vec{h}_c^{(0)}$, we can concatenate a zero vector to the input $\vec{b}_c$ and modify the weight matrices and biases in the same way and obtain $\begin{bmatrix} \vec{h}_c^{(0)} \\ 0 \end{bmatrix}$.
For the following message passing steps, we also pad the weight matrices $\vec{W}$ and biases $\vec{w}$ so that the information $\vec{x}^{(0)}$ is carried to the end of embedding $\begin{bmatrix} \vec{h}_v^{(L)} \\ \vec{x}^{(0)}_c \end{bmatrix}$. 
Finally, we need to modify the last layer of the $\textsf{READOUT}$ function by modifying the weights and biases the same way and padding a ones vector column at the last weight matrix $\begin{bmatrix} \vec{W}^{(L)} & \bm{1} \end{bmatrix}$. Therefore, the input $\vec{x}^{(0)}$ will be finally added to the output, i.e.,
\begin{equation*}
f_{\textsf{MPNN,2}}\left(I, \vec{x}^{(0)}\right) = f_{\textsf{MPNN,1}}(I) + \vec{x}^{(0)}.
\end{equation*}
Hence, we immediately have the conclusion that $f_{\textsf{MPNN,2}}$ can approximate $\vec{x}^* - \vec{x}^{(0)}$ arbitrarily well.
\end{proof}

\end{document}